\documentclass[letterpaper]{article} % DO NOT CHANGE THIS

\newif\ifincludemain
\includemaintrue  % or \includemaintrue 

\newif\ifincludeappendix
\includeappendixtrue  % or \includeappendixfalse 

\usepackage[]{aaai2026}  % DO NOT CHANGE THIS
\usepackage{times}  % DO NOT CHANGE THIS
\usepackage{helvet}  % DO NOT CHANGE THIS
\usepackage{courier}  % DO NOT CHANGE THIS
\usepackage[hyphens]{url}  % DO NOT CHANGE THIS
\usepackage{graphicx} % DO NOT CHANGE THIS
\usepackage{natbib}  % DO NOT CHANGE THIS AND DO NOT ADD ANY OPTIONS TO IT
\usepackage{caption} % DO NOT CHANGE THIS AND DO NOT ADD ANY OPTIONS TO IT
\usepackage{subcaption}
\usepackage{algorithm}
\usepackage[noend]{algorithmic}

\usepackage{amsmath}
\usepackage{amssymb}
\usepackage{amsthm}
\usepackage{dsfont}
\usepackage{bm}
\usepackage{mathtools}
\usepackage{units} %% nicefrac

\newtheorem{theorem}{Theorem}
\newtheorem{lemma}{Lemma}

\newtheorem{assumption}{Assumption}

\newcommand{\norm}[1]{\left\lVert#1\right\rVert}

\usepackage[shortlabels]{enumitem}
\usepackage{tikz}          
\usetikzlibrary{calc}
\usepackage{booktabs}
\usepackage{xr}

\ifincludemain
\else
\fi

\usepackage{newfloat}
\usepackage{listings}
\DeclareCaptionStyle{ruled}{labelfont=normalfont,labelsep=colon,strut=off} % DO NOT CHANGE THIS
\floatstyle{ruled}
\newfloat{listing}{tb}{lst}{}
\floatname{listing}{Listing}
\title{A Unified Convergence Analysis for Semi-Decentralized Learning: Sampled-to-Sampled vs. Sampled-to-All Communication}
\author{
    Angelo Rodio\textsuperscript{\rm 1},
    Giovanni Neglia\textsuperscript{\rm 2},
    Zheng Chen\textsuperscript{\rm 1},
    Erik G. Larsson\textsuperscript{\rm 1}
}
\affiliations{
    \textsuperscript{\rm 1}Department of Electrical Engineering, Link\"oping University, Sweden\\
    Centre Inria d'Université Côte d'Azur, France\\
    \{angelo.rodio, zheng.chen, erik.g.larsson\}@liu.se, giovanni.neglia@inria.fr
}

\usepackage{bibentry}
\begin{document}

\ifincludemain

\maketitle

% --- add page numbers for arxiv ---
\pagenumbering{arabic}  
\thispagestyle{plain}   
\pagestyle{plain}       
% -------------------------------

\begin{abstract}

In semi-decentralized federated learning, devices primarily rely on device-to-device communication but occasionally interact with a central server. Periodically, a sampled subset of devices uploads their local models to the server, which computes an aggregate model. The server can then either (i)~share this aggregate model only with the sampled clients (sampled-to-sampled, S2S) or (ii) broadcast it to all clients (sampled-to-all, S2A). Despite their practical significance, a rigorous theoretical and empirical comparison of these two strategies remains absent.
We address this gap by analyzing S2S and S2A within a unified convergence framework that accounts for key system parameters: sampling rate, server aggregation frequency, and network connectivity. Our results---both analytical and experimental---reveal distinct regimes where one strategy outperforms the other, depending primarily on the degree of data heterogeneity across devices. These insights lead to concrete design guidelines for practical semi-decentralized FL deployments.

\end{abstract}

\begin{links}
    \link{Code}{https://github.com/arodio/SemiDec}
    % \link{Datasets}{https://aaai.org/example/datasets}
    % \link{Extended version}{https://arxiv.org/abs/2511.11560}
\end{links}

\section{Introduction}

The performance of large-scale machine learning models depends critically on the volume and diversity of data; however, in many practical scenarios, training data are decentralized, generated by edge devices such as smartphones or sensors~\cite{mcmahanCommunicationEfficientLearningDeep2017, kairouzAdvancesOpenProblems2021}. Centralizing these data is often prohibitively expensive---or even infeasible---due to network limitations and privacy constraints~\cite{bonawitzFederatedLearningScale2019, liFederatedLearningChallenges2020}. 

Federated learning (FL) is a distributed machine learning paradigm in which multiple devices cooperate to learn a global model under the orchestration of a central server without sharing their data~\cite{mcmahanCommunicationEfficientLearningDeep2017}.
Device-to-server (D2S) communication is typically expensive in FL, especially when the central server is located in a wide-area network, where limited uplink bandwidth dominates both communication latency and energy consumption. The de facto optimization method, local stochastic gradient descent (SGD)~\cite{konecnyFederatedLearningStrategies2017, mcmahanCommunicationEfficientLearningDeep2017}, addresses this communication bottleneck by enabling devices to perform multiple local updates before server aggregation.
This simple trick reduces the number of D2S communications but has a well-known drawback: multiple local SGD steps on non-identically distributed (non-IID) data lead to local over-fitting (known as model drift) and hinder convergence~\cite{karimireddySCAFFOLDStochasticControlled2020, liConvergenceFedAvgNonIID2023}.

Fully-decentralized learning eliminates the central server and relies solely on device-to-device (D2D) communications, where devices average their local models with those of their neighbors after each SGD update~\cite{lianCanDecentralizedAlgorithms2017, koloskovaUnifiedTheoryDecentralized2020}. These exchanges are typically inexpensive, leveraging high-bandwidth local-area networks or direct short-range wireless links.
The convergence of such algorithms depends on the connectivity of the underlying communication graph. Intuitively, sparse connectivity slows convergence---a phenomenon analyzed in prior work~\cite{yuanConvergenceDecentralizedGradient2016, negliaDecentralizedGradientMethods2020, barsRefinedConvergenceTopology2023, larssonUnifiedAnalysisDecentralized2025}. More critically, convergence is impossible when the graph is disconnected, as information cannot propagate between different graph components.

Semi-decentralized learning interleaves D2D consensus rounds within components with periodic communication between a sampled subset of devices and a central server, which aggregates their models~\cite{chenAcceleratingGossipSGD2021, linSemiDecentralizedFederatedLearning2021}. This hybrid design leverages the hierarchical structure of modern networks: frequent, low-cost D2D exchanges foster local consensus within  components, while periodic D2S rounds ensure information sharing across components
and enable global convergence.
Once the server aggregates the models of the sampled devices, two  communication primitives have been proposed in the literature:
\begin{enumerate}[label=(\roman*), leftmargin=20pt, topsep=1pt, itemsep=1pt, parsep=0pt, partopsep=0pt]
\item \emph{Sampled-to-Sampled} (S2S): the server sends the aggregate model \emph{only} to the sampled devices, while the remaining devices retain their current local models~\cite{chenTamingSubnetDriftD2DEnabled2024};
\item \emph{Sampled-to-All} (S2A): the server broadcasts the aggregate model to \emph{all} devices, which then replace their current models~\cite{chenAcceleratingGossipSGD2021, linSemiDecentralizedFederatedLearning2021, guoHybridLocalSGD2021}.
\end{enumerate}
While both variants appear in prior work, their relative merits have not been thoroughly investigated. Intuitively, S2A may spread information faster because the aggregated model is immediately disseminated to all clients. 
However, this comes at the cost of introducing a bias: the sampled clients exert a disproportionate influence, as their models overwrite information from unsampled ones. In this work, we address this gap through a unified theoretical analysis and extensive experimental comparison of the two strategies.
\paragraph{Our contributions.}
\begin{itemize}
    \item We develop a unified theoretical framework that captures \emph{(i)} intra- and inter-component statistical heterogeneity, \emph{(ii)} the sampling rate, \emph{(iii)}~the server aggregation period, and \emph{(iv)}~the D2D network connectivity. 
    Our analysis reveals a fundamental trade-off. S2A introduces a broadcast-induced bias due to the shift in the global average model after each D2S aggregation but reduces disagreement error by periodically realigning all local models. Conversely, S2S avoids this bias but suffers from greater disagreement, as non-sampled models remain misaligned after aggregation.
    \item By comparing convergence bounds, we identify regimes in which one communication primitive outperforms the other. Specifically, S2A converges faster when both intra- and inter-component heterogeneity are low, while S2S outperforms as inter-component heterogeneity increases---particularly at low sampling rates, short server periods, or sparse network connectivity.
    \item Simulations on benchmark FL datasets across varying sampling rates, aggregation periods, and network topologies confirm these regimes and highlight the importance of selecting the appropriate communication primitive. These insights translate into practical guidelines for configuring semi-decentralized FL deployments.
\end{itemize}

\section{Related Work}

% The cost of device-to-server communication in FL has been widely studied~\cite{shamirCommunicationEfficientDistributedOptimization2014, alistarhQSGDCommunicationEfficientSGD2017, horvothNaturalCompressionDistributed2022}. Both classical~\cite{stichLocalSGDConverges2018, reddiAdaptiveFederatedOptimization2023} and  refined~\cite{mishchenkoProxSkipYesLocal2022} analyses of local SGD establish a fundamental trade-off: a moderate number of local steps reduces wall-clock time, whereas too many local updates on non-IID data induce \emph{model drift} and hinder convergence. Advanced algorithms, e.g., relying on control variates~\cite{karimireddySCAFFOLDStochasticControlled2020} and proximal corrections~\cite{mishchenkoProxSkipYesLocal2022}, mitigate this drift at additional computational or communication cost. However, the qualitative conclusion remains: an excessive number of local SGD steps under high statistical heterogeneity slows global progress.

The cost of device-to-server communication in FL has been widely studied~\cite{shamirCommunicationEfficientDistributedOptimization2014, alistarhQSGDCommunicationEfficientSGD2017, horvothNaturalCompressionDistributed2022}. Both classical~\cite{stichLocalSGDConverges2018, reddiAdaptiveFederatedOptimization2023} and refined~\cite{mishchenkoProxSkipYesLocal2022} analyses of local SGD establish a fundamental trade-off: a moderate number of local steps reduces wall-clock time, whereas many local updates on non-IID data induce \emph{model drift} and hinder convergence. More advanced methods, e.g., using control variates~\cite{karimireddySCAFFOLDStochasticControlled2020} or proximal corrections~\cite{mishchenkoProxSkipYesLocal2022}, mitigate this drift at additional computational or communication cost, but the main conclusion remains: too many local SGD steps under high statistical heterogeneity slow convergence.

In fully-decentralized SGD (D-SGD), which relies solely on D2D communications, the convergence rate is governed by the spectral gap of the doubly stochastic mixing matrix~$W$. Specifically, the iteration complexity scales inversely with $\gamma \coloneqq 1-\lambda_2(W^\top W)$, where $\lambda_2$ denotes the second-largest eigenvalue of $W^\top W$~\cite{nedicStochasticGradientPushStrongly2016, yuanConvergenceDecentralizedGradient2016, koloskovaUnifiedTheoryDecentralized2020, barsRefinedConvergenceTopology2023}. Convergence becomes slower as~$\gamma$ approaches zero, and for $\gamma=0$ (disconnected graph), D-SGD fails to reach the global optimum, as each connected component converges to its \emph{local} minimizer.

Hierarchical FL assumes a multi-tier tree topology (cloud-edge-device) and aggregates along the hierarchy~\cite{wangResourceEfficientFederatedLearning2021}; semi-decentralized FL supports \emph{arbitrary} D2D topologies~\cite{chenAcceleratingGossipSGD2021, linSemiDecentralizedFederatedLearning2021}. Prior work has analyzed the S2S and S2A primitives under convex objectives~\cite{linSemiDecentralizedFederatedLearning2021, chenTamingSubnetDriftD2DEnabled2024} and, for S2A, also under non-convex objectives~\cite{guoHybridLocalSGD2021}, but assuming that at least one device per connected component is sampled in every server round. 
This assumption implicitly requires the server to know the component membership of each device---a requirement that is difficult to satisfy in practice due to the large number of devices, their mobility (resulting in time-varying communication graphs), and privacy constraints (as it may indirectly reveals user locations). 
To the best of our knowledge, a convergence analysis of the S2S primitive is still lacking for non-convex objectives, and there is no systematic theoretical or empirical comparison of S2S and S2A within a unified framework.

Our analysis tackles the following technical challenges:
\begin{enumerate}[label=(\roman*),leftmargin=20pt]
\item The broadcast-induced bias error in S2A, defined as the change in the average model before and after a D2S communication, and the disagreement error in S2S, measuring the divergence of the local models from the global average, scale \emph{differently} with stepsize, sampling rate, server period, and network connectivity, making their comparison non-trivial.
\item The S2A update rule can be modeled as a rank-one, column-stochastic but not row-stochastic averaging operator; thus, standard spectral-gap arguments for doubly stochastic $W$ matrices in D-SGD analyses (e.g.,~\citet{koloskovaUnifiedTheoryDecentralized2020}) are not applicable.
\item Although the S2S update rule involves a symmetric, stochastic weight matrix---formally compatible with the assumptions in~\citet{koloskovaUnifiedTheoryDecentralized2020}---their analysis fails to capture the fundamental asymmetry between D2D and D2S rounds, where inter-component statistical heterogeneity is reduced \emph{only} through server aggregation. This distinction is crucial for the comparison of S2S and S2A and motivates our analysis.
\end{enumerate}
We address these challenges by \emph{(a)}~characterizing bias and disagreement errors through the properties of the S2S and S2A operators, \emph{(b)}~introducing an orthogonal decomposition of the total disagreement into intra- and inter-component terms, and \emph{(c)}~capturing the distinct effects of D2D and D2S communication on intra- versus inter-component heterogeneity.

\section{Problem Setting}

\subsubsection{Network model.}
We consider a network consisting of a central server and $n$ devices, organized in $C$ disjoint components (or clusters). 
Each component $c\!\in\!\{ 1, \dots, C \}$ is modeled as an undirected, connected, and time-varying graph $\mathcal{G}_c^{(t)}\!=\!(\mathcal{V}_c, \mathcal{E}_c^{(t)})$, where $\mathcal{V}_c$ denotes the set of $n_c\!\coloneqq\!|\mathcal{V}_c|$ devices in component~$c$, and $(i,j)\!\in\!\mathcal{E}_c^{(t)}$ indicates that devices $i,j\!\in\!\mathcal{V}_c$ communicate via D2D links at round $t$.
In addition, each device can communicate with the central server through D2S links. The overall network at round $t$ is modeled as $\mathcal{G}^{(t)}\!=\!(\mathcal{V},\mathcal{E}^{(t)})$, where $\mathcal{V}\!\coloneqq\!\bigcup_{c=1}^{C} \mathcal{V}_c$ and $\mathcal{E}^{(t)}\!\coloneqq\!\bigcup_{c=1}^{C} \mathcal{E}_c^{(t)}$. 

\subsubsection{Learning task.} We consider an FL system where the central server and the devices collaborate to learn the parameters $\bm{x} \in \mathbb{R}^{d}$ of a machine learning model, where $d \in \mathbb{N}$ is the model dimension. Each device $i \in \mathcal{V}$ has a local dataset $\mathcal{D}_i$ of data samples $\xi \in \mathcal{D}_i$. We denote by $F_i(\bm{x};\xi)$ the loss incurred by the model with parameters $\bm{x}$ on data sample $\xi$.
The goal is to solve an optimization problem of the form:
\begin{align}
    \min_{\bm{x} \in \mathbb{R}^d} f(\bm{x}) \coloneqq \frac{1}{n} \sum_{i=1}^{n} f_i(\bm{x}) ,
    \label{eq:g_obj}
\end{align}
where $f_i(\bm{x}) \coloneqq \frac{1}{|\mathcal{D}_i|} \sum_{\xi \in \mathcal{D}_i}F_i(\bm{x},\xi)$ is the local objective of device $i \in \mathcal{V}$. 

\subsubsection{Notation.}
All vectors are by default column vectors. 
$\bm{0}$ and $\bm{1}$ denote the all-zeros and all-ones vectors of appropriate dimension. $I$ is the identity matrix.
The global averaging projector is $\Pi \coloneqq \frac{1}{n} \bm{1} \bm{1}^\top$. 
Given $n$ vectors $\bm{x}_1,\dots,\bm{x}_n \in \mathbb{R}^{d}$, we write their average as $\bar{\bm{x}} \coloneqq \frac{1}{n} \sum_{i=1}^{n} \bm{x}_i \in \mathbb{R}^{d}$.
We stack the $n$ vectors as columns in the matrix $X \coloneqq [\bm{x}_1,\dots,\bm{x}_n] \in \mathbb{R}^{d\times n}$,
such that right-multiplication by $\Pi$ performs column averaging: $\bar{X} \coloneqq X\Pi = [\bar{\bm{x}},\dots,\bar{\bm{x}}] \in \mathbb{R}^{d\times n}$.
We use $\|\cdot\|_2$ for both the Euclidean norm of a vector and the spectral norm of a matrix, and $\|\cdot\|_F$ for the Frobenius norm. 
% The orthogonal complement to the global averaging projector is $I-\Pi$; the global disagreement error is $\|X - \bar{X}\|_F = \|X(I-\Pi)\|_F$.

\section{Two Communication Primitives for Semi-Decentralized FL}

We study two semi-decentralized learning primitives, summarized in Algorithm~\ref{alg:mat}, for solving Problem~\eqref{eq:g_obj}. The training proceeds over $T$ communication rounds, where each round $t \in \{ 0, \dots, T-1 \}$ consists of two or three steps:
\begin{enumerate}[label=(\roman*),leftmargin=20pt]
\item \emph{Local stochastic descent.}
        Each device $i \in \mathcal{V}$ updates its local model $\bm{x}_i^{(t)}$ by one local SGD step:
        \begin{align}
            \textstyle
            \bm{x}_i^{(t+\nicefrac{1}{3})} = \bm{x}_i^{(t)} - \eta_t \nabla F_i(\bm{x}_i^{(t)}, \mathcal{B}_i^{(t)}),
        \end{align}
        where $\eta_t$ is the stepsize, $\mathcal{B}_i^{(t)}$ is a mini-batch sampled from the local dataset $\mathcal{D}_i$, and $\nabla F_i(\bm{x}_i^{(t)}, \mathcal{B}_i^{(t)})$ is an unbiased estimate of $\nabla F_i(\bm{x}_i^{(t)})$.

\item \emph{Device-to-device (D2D) mixing.}
        Each device $i \in \mathcal{V}$ averages its local model $\bm{x}_i^{(t+\nicefrac{1}{3})}$ with neighbors via mixing weight $w_{ji}^{(t)}$, where $w_{ji}^{(t)}>0$ iff $(j,i) \in \mathcal{E}^{(t)}$:
        \begin{align}
            \textstyle
            \bm{x}_i^{(t+\nicefrac{2}{3})} = \sum_{j=1}^{n} w_{ji}^{(t)} \bm{x}_j^{(t+\nicefrac{1}{3})}.
        \end{align}
        In fully-decentralized rounds,
        $\bm{x}_i^{(t+1)}=\bm{x}_{i}^{(t+\nicefrac{2}{3})}$.

\item \emph{Device-to-server (D2S) aggregation.}
        Every $H$ rounds, the server
        samples a subset $\mathcal S^{(t)} \subseteq \mathcal{V}$
        of $K$ devices uniformly at random without replacement and averages their local models:
        \begin{align}
            \textstyle
            \hat{\bm{x}}^{(t+1)} = \frac{1}{|\mathcal{S}^{(t)}|} \sum_{i\in\mathcal{S}^{(t)}} \bm{x}_{i}^{(t+\nicefrac{2}{3})}.
        \end{align}
\end{enumerate}
The dissemination of this aggregate from the server to the devices can follow two distinct communication primitives: 
Sampled-to-Sampled (S2S) and Sampled-to-All~(S2A). 

\paragraph{Sampled-to-Sampled (S2S).}
The server transmits the aggregate model  \emph{only} to the sampled devices: $\bm{x}_i^{(t+1)}\!=\!\hat{\bm{x}}^{(t+1)}$, $i\!\in\!\mathcal{S}^{(t)}$; the other devices retain their local model. 
The evolution of the local models can be represented as the matrix multiplication $X^{(t+1)}\!=\!X^{(t+2/3)} W_{\text{S2S}}^{(t)}$, where:
\begin{align}
    (W_{\text{S2S}}^{(t)})_{ij} = \begin{cases}
    \frac{1}{K}, & i,j \in \mathcal{S}^{(t)}; \\
    1, & i = j \notin \mathcal{S}^{(t)}; \\
    0, & \text{otherwise.}
    \end{cases}
    \label{eq:S2S}
\end{align}

\paragraph{Sampled-to-All (S2A).}
The server broadcasts $\bm{x}^{(t+1)}$ to \emph{all} devices: $\bm{x}_i^{(t+1)} \!=\!\hat{\bm{x}}^{(t+1)}$  for all $i\!\in\!\mathcal{V}$. As above, this can be represented as $X^{(t+1)}\!=\!X^{(t+2/3)} W_{\text{S2A}}^{(t)}$, where:
\begin{align}
    (W_{\text{S2A}}^{(t)})_{ij} = \begin{cases}
    \frac{1}{K}, & i \in \mathcal{S}^{(t)}; \\
    0, & \text{otherwise.}
    \end{cases}
    \label{eq:S2A}
\end{align}

%---------------------------------------------------------------------
\begin{algorithm}[t]
\caption{Semi-Decentralized Federated Learning}
\label{alg:mat}
\textbf{Input:} $X^{(0)}\!\in\!\mathbb{R}^{d\times n}$, rounds $T$, period $H$,
stepsizes $\{\eta_t\}$, mixing matrices $W^{(t)} \sim \mathcal{W}$
\begin{algorithmic}[1]
\FOR{$t=0,\dots,T-1$}
\STATE $X^{(t+\nicefrac13)} \gets X^{(t)}-\eta_t\nabla F(X^{(t)},\mathcal{B}^{(t)})$
\STATE $X^{(t+\nicefrac23)} \gets X^{(t+\nicefrac13)}W^{(t)}$
\IF{$t\equiv0~(\mathrm{mod}\,H)$}
      \STATE sample $\mathcal S^{(t)} \subseteq \mathcal{V}$, $|\mathcal S^{(t)}| = K$
      \STATE build $W_{\text{S2(S/A)}}^{(t)}$ by Eq.~\eqref{eq:S2S} (S2S) or Eq.~\eqref{eq:S2A} (S2A) 
      \STATE $X^{(t+1)} \gets X^{(t+\nicefrac23)}W_{\text{S2(S/A)}}^{(t)}$
\ELSE
      \STATE $X^{(t+1)} \gets X^{(t+\nicefrac23)}$
\ENDIF
\ENDFOR
\STATE\textbf{return} $X^{(T)}$
\end{algorithmic}
\end{algorithm}

\subsection{High-level Comparison of S2S and S2A}
\label{subsec:comparison}

We identify the following two errors after the D2S round:
\begin{enumerate}[label=(\roman*),leftmargin=20pt,topsep=0.2em,itemsep=0.1em, parsep=0.2em]
\item the \emph{bias error}, which quantifies the change in the global average model induced by the D2S step, defined as
$\mathbb{E}[\|\bar X^{(t+1)} - \bar X^{(t+\nicefrac23)}\|_F^2]$;
\item the \emph{disagreement error}, which quantifies the divergence of the local models from the global average model, defined as 
$\mathbb{E}[\|X^{(t+1)} - \bar X^{(t+1)}\|_F^2]$.
\end{enumerate}

The two primitives, S2S and S2A, exhibit opposite error behaviors: S2S preserves the global average (zero bias) but leaves residual disagreement, whereas
S2A enforces perfect consensus (zero disagreement) at the cost of a non-zero bias.

For S2S, the matrix $W_{\text{S2S}}$ is symmetric and doubly stochastic, satisfying $W_{\text{S2S}} \Pi = \Pi W_{\text{S2S}} = \Pi$. 

Therefore, the bias error vanishes since:
\begin{align}
    \bar{X}^{(t+1)} = X^{(t+\nicefrac{2}{3})} W_{\text{S2S}} \Pi = X^{(t+\nicefrac{2}{3})} \Pi = \bar{X}^{(t+\nicefrac{2}{3})}.
\end{align}
However, non-sampled devices are not updated with the server aggregate, resulting in residual disagreement:
\begin{align}
    X^{(t+1)} = X^{(t+\nicefrac{2}{3})} W_{\text{S2S}} 
    \neq X^{(t+\nicefrac{2}{3})} \Pi = \bar{X}^{(t+1)},
\end{align}
with magnitude (bounded in Lemma~\ref{lem:sampling_app}, Appendix~\ref{app:subsec:S2S}):
\begin{align}
    \textstyle
    \mathbb{E} [ \| X^{(t+1)} - \bar{X}^{(t+1)} \|_F^2 ]
    =
    \frac{n-K}{n-1} \mathbb{E} \| X^{(t+\frac{2}{3})} - \bar{X}^{(t+\frac{2}{3})} \|_F^2,
    \label{eq:deviation_sts}
\end{align}
where $\mathbb{E} \| X^{(t+\nicefrac{2}{3})} - \bar{X}^{(t+\nicefrac{2}{3})} \|_F^2$ denotes the disagreement inherited from the D2D step at time $t+{2}/{3}$.

Conversely, $W_{\text{S2A}}$ is column-stochastic but \emph{not} row-stochastic, with $\Pi W_{\text{S2A}}=\Pi$ and $W_{\text{S2A}}\Pi = W_{\text{S2A}} \neq \Pi$. 
This property eliminates disagreement since:
\begin{align}
    X^{(t+1)} - \bar{X}^{(t+1)} = X^{(t+\nicefrac{2}{3})} (W_{\text{S2A}} - W_{\text{S2A}} \Pi) = 0,
\end{align}
but introduces the broadcast-induced bias:
\begin{align}
    \bar{X}^{(t+1)} = X^{(t+\nicefrac{2}{3})} W_{\text{S2A}} \Pi 
    \neq X^{(t+\nicefrac{2}{3})} \Pi = \bar{X}^{(t+\nicefrac{2}{3})},
\end{align}
with magnitude (bounded in Lemma~\ref{lem:broadcast_app}, Appendix~\ref{app:subsec:S2A}):
\begin{align}
    \textstyle
    \mathbb{E} [ \| \bar{X}^{(t+1)} - \bar{X}^{(t+\frac{2}{3})} \|_F^2 ]
    {=}
    \frac{n-K}{K(n-1)} \mathbb{E} \| X^{(t+\frac{2}{3})} - \bar{X}^{(t+\frac{2}{3})} \|_F^2.
    \label{eq:deviation_sta}
\end{align}

Although the bias factor in Eq.~\eqref{eq:deviation_sta} might appear smaller than the disagreement factor in Eq.~\eqref{eq:deviation_sts}, the two equations describe different error sources, which propagate under different scalings with respect to stepsize, sampling rate, server period, and network connectivity. This interplay makes the comparison between S2S and S2A non-trivial and motivates our subsequent unified convergence analysis.

\section{Unified Convergence Analysis}
\label{sec:analysis}

Our framework extends the convergence theory of decentralized optimization~\citep{koloskovaUnifiedTheoryDecentralized2020, barsRefinedConvergenceTopology2023} to semi-decentralized federated learning, and provides the first theoretical comparison of S2S and S2A.

All theoretical results assume Lipschitz continuity of the stochastic gradients~\cite{nguyenNewConvergenceAspects2019}. 
\begin{assumption}[L-smoothness]
    \label{asm:smoothness}
    For every $i \in \mathcal{V}$ and every $\xi \sim \mathcal{D}_i$, the stochastic loss $F_i(\cdot,\xi)$ is $L$-smooth; i.e., there exists $L > 0$ such that, for all $\bm{x}, \bm{y} \in \mathbb{R}^d$, 
    \begin{align}
        \|\nabla F_i(\bm{x},\xi) - \nabla F_i(\bm{y},\xi)\|_2 \leq L \|\bm{x} - \bm{y}\|_2.
    \end{align}
\end{assumption}
For convex results, we additionally invoke convexity of
the local objectives~\cite{bubeckConvexOptimizationAlgorithms2015}.
\begin{assumption}[Convexity]
    \label{asm:convexity}
    Each $f_i: \mathbb{R}^d \rightarrow \mathbb{R}$ is convex:
    \begin{align}
        f_i(\bm{y}) \geq f_i(\bm{x}) + \langle \nabla f_i(\bm{x}),\, \bm{y} - \bm{x} \rangle, \quad \forall \bm{x}, \bm{y} \in \mathbb{R}^d.
    \end{align}
\end{assumption}
To keep the analysis unified across convex and non-convex settings, we assume that the stochastic variance is uniformly bounded in $\bm{x}$~\cite{barsRefinedConvergenceTopology2023}, although in the convex case it suffices to bound it only at the optimum.
\begin{assumption}[Bounded stochastic variance]
\label{asm:variance}
    For every $i \in \mathcal{V}$, there exists a constant $\bar{\sigma}^2 > 0$ such that, for all $\bm{x} \in \mathbb{R}^{d}$,
    \begin{align}
        \mathbb{E}_{\xi \sim \mathcal{D}_i} \left[ \|\nabla F_i(\bm{x}, \xi) - \nabla f_i(\bm{x})\|_2^2 \right] \leq \bar{\sigma}^2.
    \end{align}
\end{assumption}
 
For clarity of exposition, our analysis assumes a fixed deterministic mixing matrix $W$. However, all results extend to dynamic D2D communication graphs~\cite{koloskovaUnifiedTheoryDecentralized2020}, which are represented by time-varying mixing matrices (as detailed in Appendix~\ref{app:sec_mixing_random}).

\begin{assumption}[Mixing matrix~\cite{koloskovaUnifiedTheoryDecentralized2020,barsRefinedConvergenceTopology2023}]
\label{asm:mixing}
The mixing matrix $W$ is doubly stochastic, i.e., $W \in [0,1]^{n \times n}$, $W \bm{1} = \bm{1}$, and $\bm{1}^\top W  = \bm{1}^\top$.
\end{assumption}

The matrix $W$ is block diagonal, reflecting the $C$ disconnected components of the communication graph $\mathcal{G}$. Each diagonal block $W_c \coloneqq W[\mathcal{V}_c, \mathcal{V}_c] \in \mathbb{R}^{n_c \times n_c}$ corresponds to the D2D mixing matrix of component $c \in \{ 1, \dots, C \}$.
To decompose disagreement within and across components, we define the component projector $\Pi_C\in\mathbb R^{n\times n}$ as:
\begin{align}
        \left( \Pi_C \right)_{ij} = \begin{cases}
            \frac{1}{n_c}, & i,j \in \mathcal{V}_c; \\
            0, & \text{otherwise}.
        \end{cases}
\end{align}
The operators $I-\Pi_C$ and $\Pi_C-\Pi$ 
enable an orthogonal decomposition of the global disagreement at any time 
$t$ into intra-component and inter-component terms.

\begin{lemma}[Orthogonal decomposition]
\label{lem:orthogonal}
For any $X{\in}\mathbb R^{d\times n}$,
\begin{align}
    \hspace{-0.66em}\|X(I-\Pi)\|_F^2
    = 
    \|X(I-\Pi_C)\|_F^2 
    + 
    \|X(\Pi_C-\Pi)\|_F^2.
\end{align}
\end{lemma}
Only the intra-component disagreement is reduced by D2D consensus steps, while the inter-component term requires periodic D2S aggregation.

\begin{lemma}[Intra-component mixing parameter]
\label{lem:spectral-mixing}
There exists a constant $p \in (0,1]$ such that, for all $X\in\mathbb R^{d\times n}$,
\begin{align}
        \| X (W - \Pi_C) \|_F^2 \leq (1 - p) \| X (I - \Pi_C) \|_F^2.
\end{align}
\end{lemma}
For a fixed $W$, Lemma~\ref{lem:spectral-mixing} holds with
$p \!=\! \textstyle \frac{\sum_{c=1}^C p_c (n_c - 1)}{\sum_{c=1}^C (n_c - 1)}$,
where $p_c \!=\! 1 \!-\! \lambda_2(W_c^\top W_c)$ (see Appendix~\ref{app:sec_mixing}).
For Metropolis-Hastings weights $w_{ij}\!=\!w_{ji} \!=\! \min\{1/(\deg(i)\!+\!1),\,1/(\deg(j)\!+\!1)\}$, we have $p_c \!=\! 1$ for complete graphs, $p_c \!=\!\Theta(n_c^{-1})$ for 2D grid topologies, and $p_c \!=\! \Theta(n_c^{-2})$ for ring graphs~\cite{XIAO200465, boydRandomizedGossipAlgorithms2006}.

A key step in our analysis is to disentangle heterogeneity within components from heterogeneity across components: this distinction is crucial for comparing S2S and S2A.

\begin{assumption}[Intra- and inter-component heterogeneity]
\label{asm:heterogeneity}
There exist $\bar{\zeta}_{\text{intra}}^2$, $\bar{\zeta}_{\text{inter}}^2 > 0$ such that, for all $\bm{x} \in \mathbb{R}^d$:
\begin{align}
\textstyle \hspace{-1.5em}\frac{1}{n} \sum_{i=1}^{n} \mathbb{E}_{\xi} \| \sum_{j=1}^{n} (W - \Pi_C)_{ij} \nabla F_j(\bm{x}, \xi) \|_2^2 &\leq \bar{\zeta}_{\text{intra}}^2, \label{asm:het_intra} \\
\textstyle \frac{1}{n} \sum_{i=1}^{n} \mathbb{E}_{\xi} \| \sum_{j=1}^{n} ( \Pi_C - \Pi)_{ij} \nabla F_j(\bm{x}, \xi) \|_2^2 &\leq \bar{\zeta}_{\text{inter}}^2. \label{asm:het_inter}
\end{align}
\end{assumption}

The constants $\bar{\zeta}_{\text{intra}}$ and $\bar{\zeta}_{\text{inter}}$ quantify intra- and inter-component noise arising from both stochastic variance and statistical heterogeneity.
We treat these two sources of noise jointly:
our intra-component bound (Eq.~\ref{asm:het_intra}) generalizes the neighborhood heterogeneity of~\cite{barsRefinedConvergenceTopology2023}---defined as the deviation between the $W$-weighted neighborhood gradients and their intra-component average---and is weaker than Assumption~4 in~\cite{guoHybridLocalSGD2021}.

\subsection{Main Results}

We are now ready to present our main convergence results; all proofs are deferred to Appendices~\ref{app:sec:unified_framework}--\ref{app:sec:convergenceS2A}.

\begin{theorem}[Sampled-to-Sampled]
\label{thm:sampling_main}
Under Assumptions 1--5, there exists a constant stepsize $\eta \leq \frac{p}{8L}$ such that, for any target accuracy $\epsilon > 0$, Algorithm~\ref{alg:mat} (S2S) achieves \\
\textbf{Convex:} $\frac{1}{T+1} \sum_{t=0}^{T} \mathbb{E} \left( f(\bar{\bm{x}}^{(t)}) - f^\star \right) \leq \epsilon$ after
\begin{align} 
T 
\geq
\mathcal{O}\biggl(
&\frac{\bar{\sigma}^2}{n \epsilon^2} 
+ \sqrt{\frac{n-1}{K-1}} \frac{\sqrt{L} \bar{\zeta}_{\text{intra}}}{p \epsilon^{3/2}} \notag \\
&+ \frac{n-1}{K-1} \frac{\sqrt{L} H \bar{\zeta}_{\text{inter}}}{\epsilon^{3/2}}
+ \frac{L}{p \epsilon}
\biggr) R_0^2,
\label{eq:bound_sts_1}
\end{align}
\textbf{Non-Convex:} $\frac{1}{T+1} \sum_{t=0}^T \mathbb{E} \| \nabla f(\bar{\bm{x}}^{(t)}) \|_2^2 \leq \epsilon$ after
\begin{align} 
T 
\geq
\mathcal{O}\biggl(
&\frac{\bar{\sigma}^2}{n \epsilon^2} 
+ \sqrt{\frac{n-1}{K-1}} \frac{\bar{\zeta}_{\text{intra}}}{p \epsilon^{3/2}} \notag \\
&+ \frac{n-1}{K-1} \frac{H \bar{\zeta}_{\text{inter}}}{\epsilon^{3/2}}
+ \frac{1}{p \epsilon}
\biggr) L f_0,
\label{eq:bound_sts_2}
\end{align}
where $R_0 \coloneqq \| \bm{x}^{(0)} - \bm{x}^\star \|_2$ and $f_0 \coloneqq f(\bm{x}^{(0)}) - f^\star$ denote the initial errors, and $\mathcal{O}(\cdot)$ hides the numerical constants explicitly provided in Appendix~\ref{app:subsec:theorem_S2S}.
\end{theorem}

\begin{theorem}[Sampled-to-All]
\label{thm:broadcast_main}
Under Assumptions 1--5, there exists a constant stepsize $\eta \leq \frac{p}{8L}$ such that, for any target accuracy $\epsilon > 0$, Algorithm~\ref{alg:mat} (S2A) achieves \\
\textbf{Convex:} $\frac{1}{T+1} \sum_{t=0}^{T} \mathbb{E} \left( f(\bar{\bm{x}}^{(t)}) - f^\star \right) \leq \epsilon$ after
\begin{align} 
T \geq
\mathcal{O} \biggl(
&\frac{\bar{\sigma}^2}{n \epsilon^2}
+ \frac{n-K}{K(n-1)} \frac{\bar{\zeta}_{\text{intra}}^2}{H p^2 \epsilon^2} 
+ \frac{n-K}{K(n-1)} \frac{H \bar{\zeta}_{\text{inter}}^2}{\epsilon^2} \notag \\
&+ \frac{\sqrt{L} \bar{\zeta}_{\text{intra}}}{p\epsilon^{3/2}}
+ \frac{\sqrt{L} H \bar{\zeta}_{\text{inter}}}{\epsilon^{3/2}}
+\frac{L}{p\epsilon}
\biggr) R_0^2,
\label{eq:bound_sta_1}
\end{align}
\textbf{Non-Convex:} $\frac{1}{T+1} \sum_{t=0}^{T} \mathbb{E} \| \nabla f(\bar{\bm{x}}^{(t)}) \|_2^2 \leq \epsilon$ after
\begin{align} 
T \geq
\mathcal{O} \biggl(
&\frac{\bar{\sigma}^2}{n \epsilon^2}
+ \frac{n-K}{K(n-1)} \frac{\bar{\zeta}_{\text{intra}}^2}{H p^2 \epsilon^2} 
+ \frac{n-K}{K(n-1)} \frac{H \bar{\zeta}_{\text{inter}}^2}{\epsilon^2} \notag \\
&+ \frac{\bar{\zeta}_{\text{intra}}}{p\epsilon^{3/2}}
+ \frac{H \bar{\zeta}_{\text{inter}}}{\epsilon^{3/2}}
+\frac{1}{p\epsilon}
\biggr) L f_0, 
\label{eq:bound_sta_2}
\end{align}
where $R_0 \coloneqq \| \bm{x}^{(0)} - \bm{x}^* \|_2$ and $f_0 \coloneqq f(\bm{x}^{(0)}) - f^\star$ denote the initial errors, and $\mathcal{O}(\cdot)$ hides the numerical constants explicitly provided in Appendix~\ref{app:subsec:theorem_S2A}.
\end{theorem}

\subsection{Discussion}

\begin{figure}[t]               
  \centering
  % ---------- shared legend ----------
  \includegraphics[width=0.8\columnwidth]{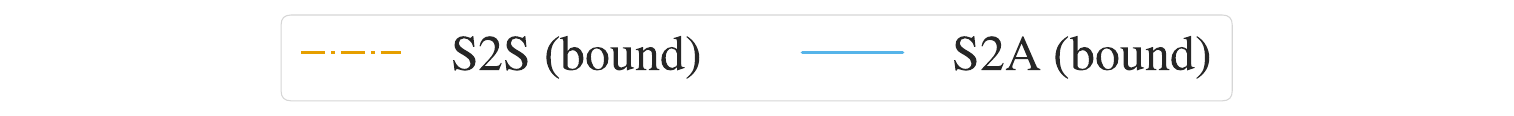}

  % =============== GRID =============== %
  \begin{minipage}{\linewidth}
  \centering

  % ---------- row 1 ----------
  \begin{subfigure}[t]{.32\linewidth}
    \includegraphics[width=\linewidth]{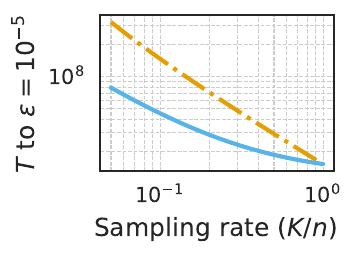}
    \caption{\tiny $\bar{\zeta}_{\text{intra}} {=} \bar{\zeta}_{\text{inter}} {=} \frac{1}{10}$}
  \end{subfigure}\hfill
  \begin{subfigure}[t]{.32\linewidth}
    \includegraphics[width=\linewidth]{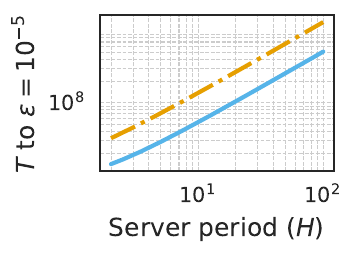}
    \caption{\tiny $\bar{\zeta}_{\text{intra}} {=} \bar{\zeta}_{\text{inter}} {=} \frac{1}{10}$}
  \end{subfigure}\hfill
  \begin{subfigure}[t]{.32\linewidth}
    \includegraphics[width=\linewidth]{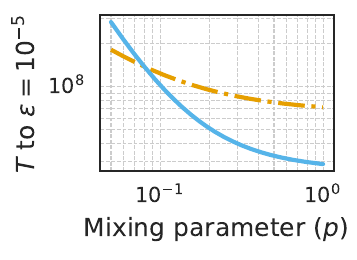}
    \caption{\tiny $\bar{\zeta}_{\text{intra}} {=} \bar{\zeta}_{\text{inter}} {=} \frac{1}{10}$}
  \end{subfigure}\\[3pt]

  % ---------- row 2 ----------
  \begin{subfigure}[t]{.32\linewidth}
    \includegraphics[width=\linewidth]{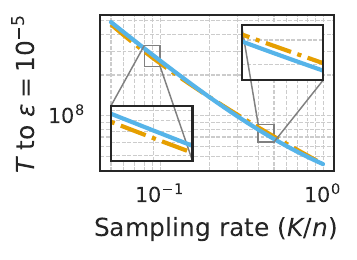}
    \caption{\tiny $\bar{\zeta}_{\text{intra}} {=} 1, \bar{\zeta}_{\text{inter}} {=} \frac{1}{10}$}
  \end{subfigure}\hfill
  \begin{subfigure}[t]{.32\linewidth}
    \includegraphics[width=\linewidth]{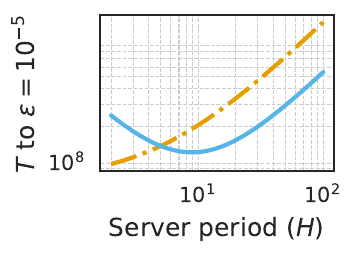}
    \caption{\tiny $\bar{\zeta}_{\text{intra}} {=} 1, \bar{\zeta}_{\text{inter}} {=} \frac{1}{10}$}
  \end{subfigure}\hfill
  \begin{subfigure}[t]{.32\linewidth}
    \includegraphics[width=\linewidth]{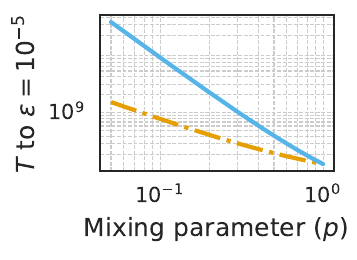}
    \caption{\tiny $\bar{\zeta}_{\text{intra}} {=} 1, \bar{\zeta}_{\text{inter}} {=} \frac{1}{10}$}
  \end{subfigure}\\[3pt]

  % ---------- row 3 ----------
  \begin{subfigure}[t]{.32\linewidth}
    \includegraphics[width=\linewidth]{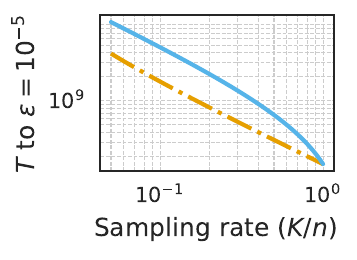}
    \caption{\tiny $\bar{\zeta}_{\text{intra}} {=} \frac{1}{10}, \bar{\zeta}_{\text{inter}} {=} 1$}
  \end{subfigure}\hfill
  \begin{subfigure}[t]{.32\linewidth}
    \includegraphics[width=\linewidth]{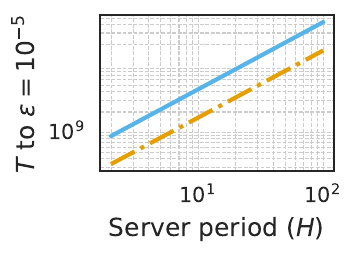}
    \caption{\tiny $\bar{\zeta}_{\text{intra}} {=} \frac{1}{10}, \bar{\zeta}_{\text{inter}} {=} 1$}
  \end{subfigure}\hfill
  \begin{subfigure}[t]{.32\linewidth}
    \includegraphics[width=\linewidth]{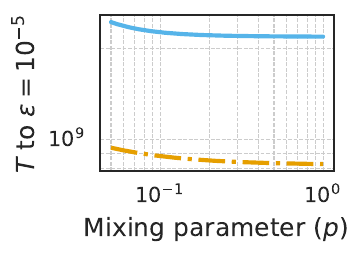}
    \caption{\tiny $\bar{\zeta}_{\text{intra}} {=} \frac{1}{10}, \bar{\zeta}_{\text{inter}} {=} 1$}
  \end{subfigure}\\[3pt]

  % ---------- row 4 ----------
  \begin{subfigure}[t]{.32\linewidth}
    \includegraphics[width=\linewidth]{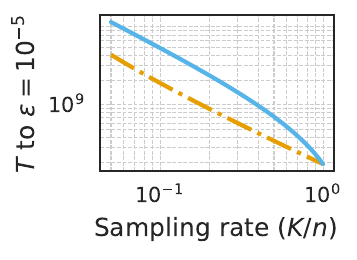}
    \caption{\tiny $\bar{\zeta}_{\text{intra}} {=} \bar{\zeta}_{\text{inter}} {=} 1$}
  \end{subfigure}\hfill
  \begin{subfigure}[t]{.32\linewidth}
    \includegraphics[width=\linewidth]{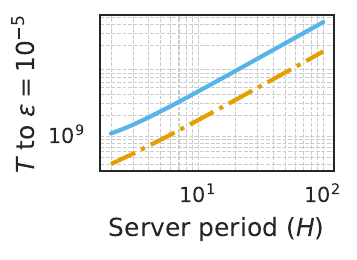}
    \caption{\tiny $\bar{\zeta}_{\text{intra}} {=} \bar{\zeta}_{\text{inter}} {=} 1$}
  \end{subfigure}\hfill
  \begin{subfigure}[t]{.32\linewidth}
    \includegraphics[width=\linewidth]{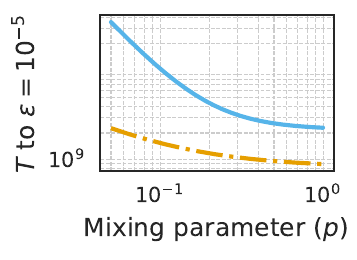}
    \caption{\tiny $\bar{\zeta}_{\text{intra}} {=} \bar{\zeta}_{\text{inter}} {=} 1$}
  \end{subfigure}

  \end{minipage}

  \caption{Convergence rates for S2S and S2A,
           comparing Eqs.~\eqref{eq:bound_sts_2}--\eqref{eq:bound_sta_2} with $n\!=\!100$, $L\!=\!f_0\!=\!1$, $\bar{\sigma}\!=\!0$. 
           Left column: Sampling rate ($K/n$) with $H\!=\!5$, $p\!=\!1$. Center column: Server period ($H$) with $K/n\!=\!0.2$, $p\!=\!1$.
           Right column: Mixing parameter~($p$) with $K/n\!=\!0.2$, $H\!=\!5$.}
  \label{fig:bound}
\end{figure}

\begin{figure*}[t]
  \centering
  
  % ---------- Shared legend ----------
  \includegraphics[width=0.95\textwidth]{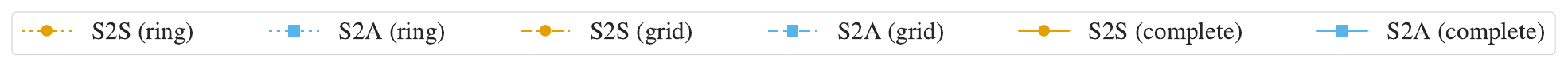}
  
  % ---------- Left figure (MNIST) ----------
  \begin{minipage}[t]{0.46\textwidth}
    \centering
    % ---------- Row 1 ----------
    \begin{subfigure}[t]{0.44\textwidth}
      \includegraphics[width=\linewidth]{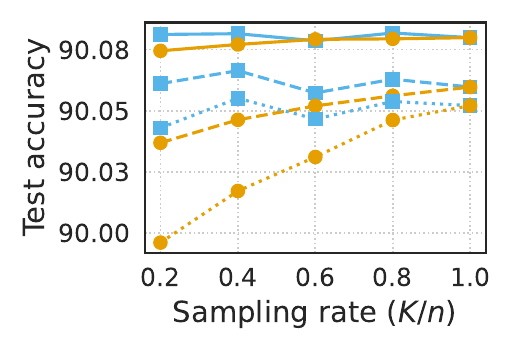}
      \caption{\scriptsize Intra IID, Inter IID}
    \end{subfigure}
    \begin{subfigure}[t]{0.44\textwidth}
      \includegraphics[width=\linewidth]{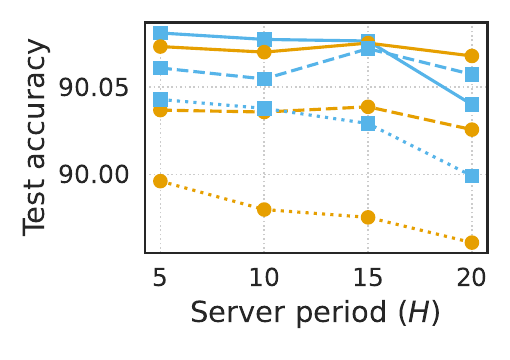}
      \caption{\scriptsize Intra IID, Inter IID}
    \end{subfigure}
    % ---------- Row 2 ----------
    \begin{subfigure}[t]{0.44\textwidth}
      \includegraphics[width=\linewidth]{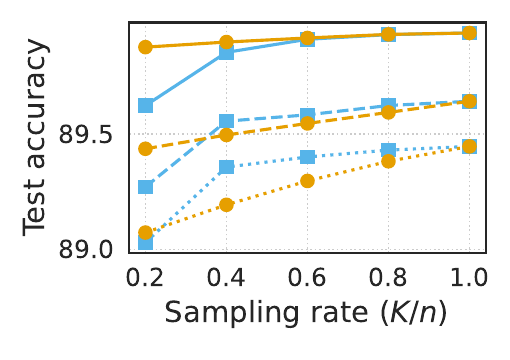}
      \caption{\scriptsize Intra non-IID, Inter IID}
    \end{subfigure}
    \begin{subfigure}[t]{0.44\textwidth}
      \includegraphics[width=\linewidth]{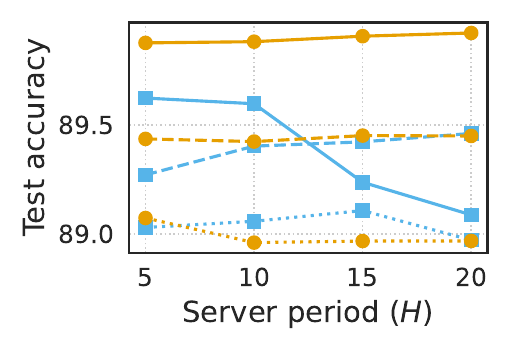}
      \caption{\scriptsize Intra non-IID, Inter IID}
    \end{subfigure}
    % ---------- Row 3 ----------
    \begin{subfigure}[t]{0.44\textwidth}
      \includegraphics[width=\linewidth]{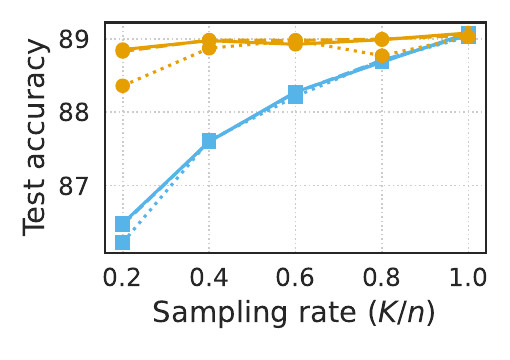}
      \caption{\scriptsize Intra IID, Inter non-IID}
    \end{subfigure}
    \begin{subfigure}[t]{0.44\textwidth}
      \includegraphics[width=\linewidth]{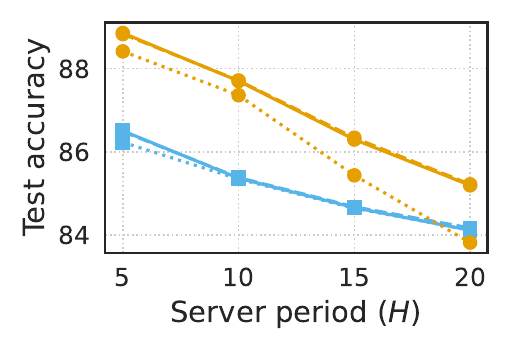}
      \caption{\scriptsize Intra IID, Inter non-IID}
    \end{subfigure}
    % ---------- Row 4 ----------
    \begin{subfigure}[t]{0.44\textwidth}
      \includegraphics[width=\linewidth]{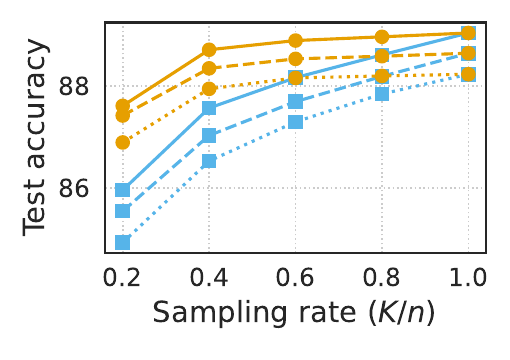}
      \caption{\scriptsize Intra non-IID, Inter non-IID}
    \end{subfigure}
    \begin{subfigure}[t]{0.44\textwidth}
      \includegraphics[width=\linewidth]{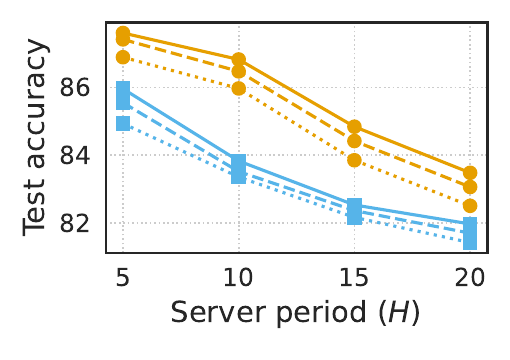}
      \caption{\scriptsize Intra non-IID, Inter non-IID}
    \end{subfigure}

    \caption{Test accuracy on MNIST dataset. 
      Left column: Sampling rate ($K/n$) with $H=5$. 
      Right column: Server period ($H$) with $K/n=0.2$.}
    \label{fig:MNIST}
  \end{minipage}\hfill
  %
  % ---------- Right figure (CIFAR-10) ----------
  \begin{minipage}[t]{0.46\textwidth}
    \centering
    % ---------- Row 1 ----------
    \begin{subfigure}[t]{0.44\textwidth}
      \includegraphics[width=\linewidth]{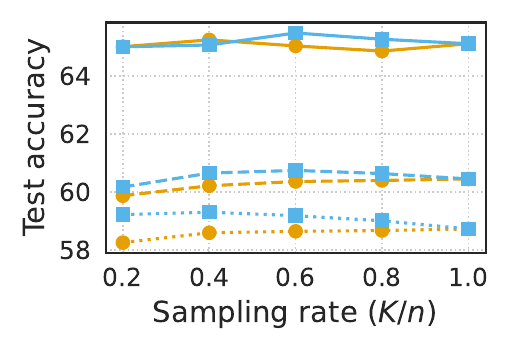}
      \caption{\scriptsize Intra IID, Inter IID}
    \end{subfigure}
    \begin{subfigure}[t]{0.44\textwidth}
      \includegraphics[width=\linewidth]{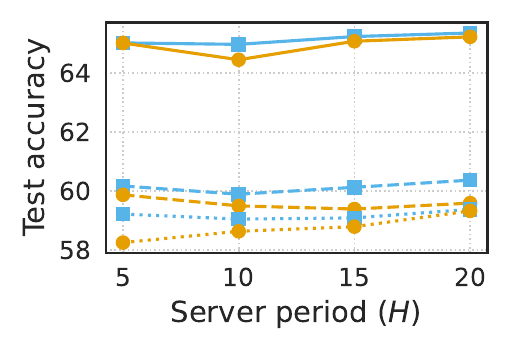}
      \caption{\scriptsize Intra IID, Inter IID}
    \end{subfigure}
    % ---------- Row 2 ----------
    \begin{subfigure}[t]{0.44\textwidth}
      \includegraphics[width=\linewidth]{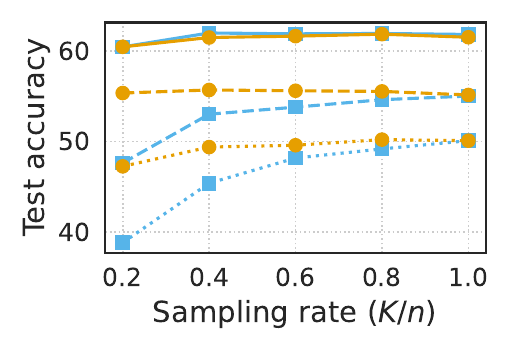}
      \caption{\scriptsize Intra non-IID, Inter IID}
    \end{subfigure}
    \begin{subfigure}[t]{0.44\textwidth}
      \includegraphics[width=\linewidth]{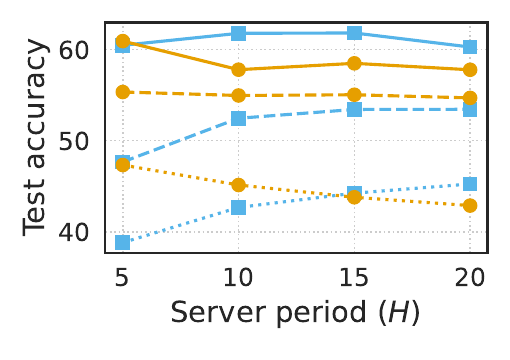}
      \caption{\scriptsize Intra non-IID, Inter IID}
    \end{subfigure}
    % ---------- Row 3 ----------
    \begin{subfigure}[t]{0.44\textwidth}
      \includegraphics[width=\linewidth]{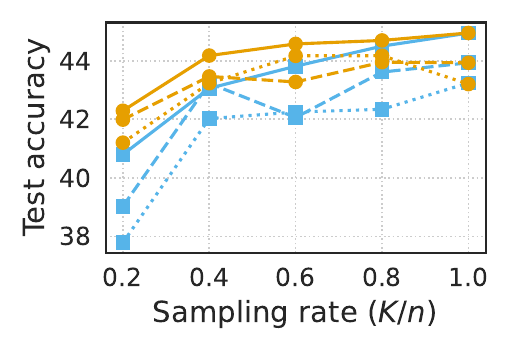}
      \caption{\scriptsize Intra IID, Inter non-IID}
    \end{subfigure}
    \begin{subfigure}[t]{0.44\textwidth}
      \includegraphics[width=\linewidth]{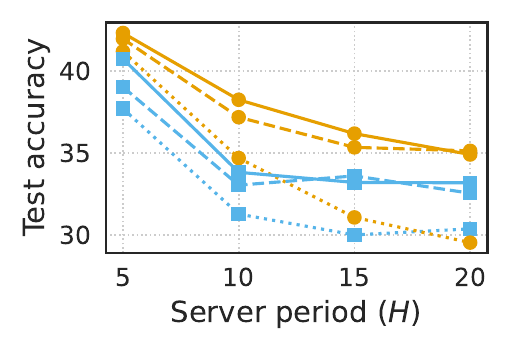}
      \caption{\scriptsize Intra IID, Inter non-IID}
    \end{subfigure}
    % ---------- Row 4 ----------
    \begin{subfigure}[t]{0.44\textwidth}
      \includegraphics[width=\linewidth]{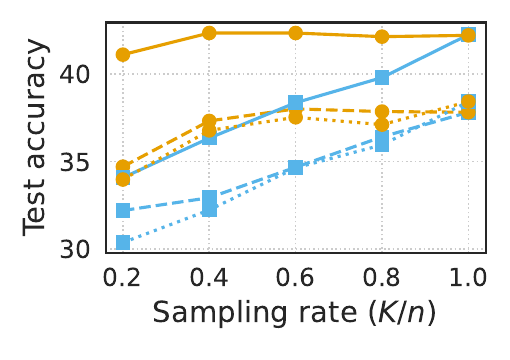}
      \caption{\scriptsize Intra non-IID, Inter non-IID}
    \end{subfigure}
    \begin{subfigure}[t]{0.44\textwidth}
      \includegraphics[width=\linewidth]{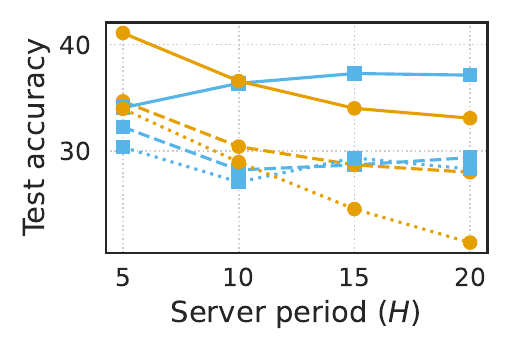}
      \caption{\scriptsize Intra non-IID, Inter non-IID}
    \end{subfigure}

    \caption{Test accuracy on CIFAR-10 dataset. 
      Left column: Sampling rate ($K/n$) with $H=5$. 
      Right column: Server period ($H$) with $K/n=0.2$.}
    \label{fig:CIFAR}
  \end{minipage}
\end{figure*}

We compare S2S and S2A under the convergence bounds of
Theorems~\ref{thm:sampling_main}--\ref{thm:broadcast_main}.  
Overall, S2S achieves a faster convergence than S2A: neglecting common factors, the dominant error terms scale as $\mathcal{O}(\epsilon^{-3/2})$ in Eqs.~\eqref{eq:bound_sts_1}--\eqref{eq:bound_sts_2}, as compared to $\mathcal{O}(\epsilon^{-2})$ in Eqs.~\eqref{eq:bound_sta_1}--\eqref{eq:bound_sta_2}.
The slower convergence of S2A is primarily due to the broadcast-induced bias error discussed in Section~\ref{subsec:comparison}.  
Moreover, S2A incurs an \emph{extra quadratic} dependence on the intra- and inter-component heterogeneity terms,
$\bar\zeta_{\text{intra}}$ and $\bar\zeta_{\text{inter}}$, which can dominate the bounds in Eqs.~\eqref{eq:bound_sta_1}--\eqref{eq:bound_sta_2} under statistically diverse data distributions.

\paragraph{Effect of sampling rate ($K/n$).}
The number of sampled devices $K$ affects both heterogeneity terms, $\bar\zeta_{\text{intra}}$ and $\bar\zeta_{\text{inter}}$, with different multiplicative factors for S2S and S2A.
Two limiting cases are noteworthy:
\begin{itemize}[topsep=0pt, itemsep=0pt, parsep=0pt, partopsep=0pt]
\item When \emph{all} devices are sampled ($K=n$), the two update rules coincide
($W_{\text{S2A}}=W_{\text{S2S}}=\Pi$), and
the two algorithms share the same convergence rate.

\item When only \emph{one} device is sampled ($K=1$), $W_{\text{S2S}} = I$,
S2S is unable to mix the sampled model across components, and the bounds in Eqs.~\eqref{eq:bound_sts_1}--\eqref{eq:bound_sts_2} diverge.
In contrast, S2A still broadcasts the (single) sampled model to all devices and thus converges, albeit at a slower rate.
\end{itemize}

\paragraph{Effect of server period ($H$).}
All $\bar{\zeta}_{\text{inter}}$ terms are penalized by a factor~$H$ in both bounds, reflecting the fact that only D2S rounds mitigate inter-component heterogeneity. For $H \to \infty$, both bounds diverge, as each components may reach consensus to their local optima, but no convergence to the global optimum can be guaranteed. 
Nonetheless,
S2A grows \emph{quadratically} in
$\bar\zeta_{\text{inter}}$, whereas S2S grows linearly.

\paragraph{Effect of mixing parameter ($p$).}  
All $\bar{\zeta}_{\text{intra}}$ terms are multiped by the inverse of the mixing parameter $p$, as D2D rounds can only mitigate intra-component heterogeneity. 

\subsection{Theoretical Heterogeneity Regimes}
\label{sec:regimes_theory}

To better interpret Theorems~\ref{thm:sampling_main}--\ref{thm:broadcast_main}, Figure~\ref{fig:bound} shows the right-hand sides of Eqs.~\eqref{eq:bound_sts_2} and~\eqref{eq:bound_sta_2}, comparing the number of rounds $T$ required to achieve the target accuracy $\epsilon = 10^{-5}$ as a function of the sampling rate (left column), server period (center column), and mixing parameter (right column). 
We consider $n=100$ devices, and set the parameters $L = f_0 = 1$ and $\bar{\sigma} = 0$ (as they are common to both S2S and S2A, their choice does not influence the comparison).  

We identify three main qualitative regimes:
\begin{enumerate}[label=\emph{\textbf{R\arabic*.}},leftmargin=20pt, topsep=0pt, itemsep=0pt, parsep=0pt, partopsep=0pt]
\item \emph{$\bm{\bar{\zeta}}_{\text{\textbf{intra}}},\, \bm{\bar{\zeta}}_{\text{\textbf{inter}}}$ \textbf{are low:}} 
S2A converges faster than S2S for most sampling rates (Fig.~\ref{fig:bound}(a)), server periods (Fig.~\ref{fig:bound}(b)), and mixing parameters (Fig.~\ref{fig:bound}(c)).

\item \emph{$\bm{\bar\zeta}_{\text{\textbf{inter}}}\!\bm{\ll}\!\bm{\bar\zeta}_{\text{\textbf{intra}}}\textbf{:}$}
S2S converges slightly faster
for low sampling rates, low server periods, and for most mixing parameters ($K/n\!<\!0.2$, $H\!<\!5$, and $p<1$);
S2A converges slightly faster
otherwise (Figs.~\ref{fig:bound}(d,e,f)).

\item \emph{$\bm{\bar\zeta}_{\text{\textbf{inter}}}$ \textbf{is high:}}
S2S converges faster for most values of $K/n$, $H$, and $p$, irrespective of
$\bar\zeta_{\text{intra}}$ (Figs.~\ref{fig:bound}(g--l)).
\end{enumerate}

\section{Experimental Results}

We simulate a semi-decentralized FL system consisting of a central
server and $n\!=\!100$ devices partitioned into $C\!=\!2$ equal-sized
components ($n_1\!=\!n_2\!=\!50$). 
For the D2S communication network, we vary the sampling rate
$K/n\!\in\!\{0.2,0.4,0.6,0.8,1\}$ and the aggregation period $H\!\in\!\{5,10,15,20\}$. 
For the D2D communication graph, we consider three representative topologies: ring, grid, and complete graph, with Metropolis-Hastings mixing weights.

We benchmark our comparison on two image-classification tasks widely adopted in prior work on semi-decentralized FL for evaluating S2S and S2A separately: the MNIST dataset \cite{lecun-mnisthandwrittendigit-2010} trained with a single-hidden-layer logistic classifier ($d\!=\!7{,}850$ parameters), and the CIFAR-10 dataset \cite{krizhevsky2009learning} trained with a reference convolutional neural network ($d \approx 1.1$ million parameters)~\cite{linSemiDecentralizedFederatedLearning2021, guoHybridLocalSGD2021, chenTamingSubnetDriftD2DEnabled2024}.

We introduce intra- and inter-component heterogeneity mimicking the constants
$\bar\zeta_{\text{intra}}$ and $\bar\zeta_{\text{inter}}$ of
Assumption~\ref{asm:heterogeneity}:
\begin{itemize}[topsep=0pt, itemsep=0pt, parsep=0pt, partopsep=0pt]
    \item \emph{Inter-component heterogeneity}.
    We partition the dataset across components either through an IID split (each component receives samples from all classes), or through a pathological non-IID split (each component receives samples from only half of the classes, with disjoint class sets)~\cite{mcmahanCommunicationEfficientLearningDeep2017}.

    \item \emph{Intra-component heterogeneity}. Within each component, we partition the dataset across devices either IID or non-IID, the latter through a Dirichlet distribution with concentration parameter $0.1$~\cite{wangFederatedLearningMatched2019}.
\end{itemize}

All models are trained with mini-batch
SGD (batch size $128$) for $T\!=\!100$ rounds.
For each algorithm, we tune the stepsize
$\eta\!\in\!\{10^{-2.5},10^{-2},10^{-1.5},10^{-1}\}$.
Results are averaged over five independent runs.
Additional experimental details are given in Appendix~\ref{app:sec:experiments}.

\subsection{Experimental Heterogeneity Regimes}

\begin{figure}
    \centering
    \begin{subfigure}[b]{0.44\linewidth}
        \centering
        \includegraphics[width=\linewidth]{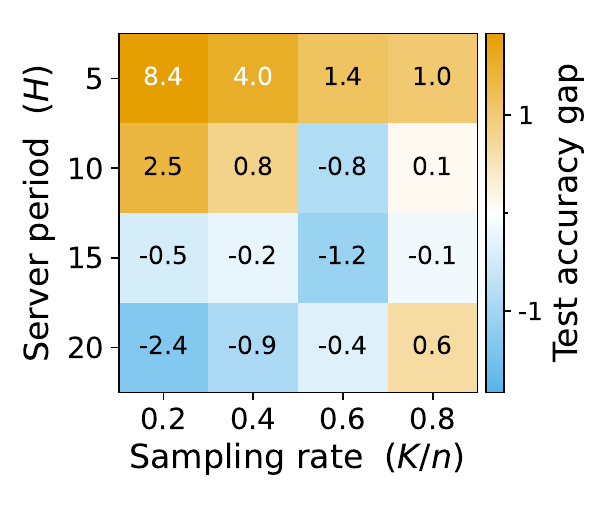}
        \caption{\scriptsize Intra non-IID, Inter IID}
        \label{fig:heatmap_non_iid_iid}
    \end{subfigure}
    \hspace{0.5em}
    \begin{subfigure}[b]{0.44\linewidth}
        \centering
        \includegraphics[width=\linewidth]{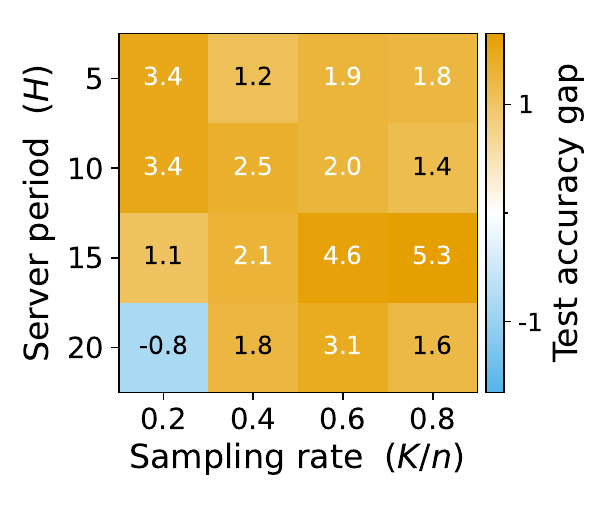}
        \caption{\scriptsize Intra IID, Inter non-IID}
        \label{fig:heatmap_iid_non_iid}
    \end{subfigure}
    \caption{Accuracy gap on CIFAR-10 with ring topology.}
    \label{fig:cifar_heatmaps}
\end{figure}

\begin{figure}
    \centering
    \begin{subfigure}[b]{0.44\linewidth}
    \centering
        \includegraphics[width=\linewidth]{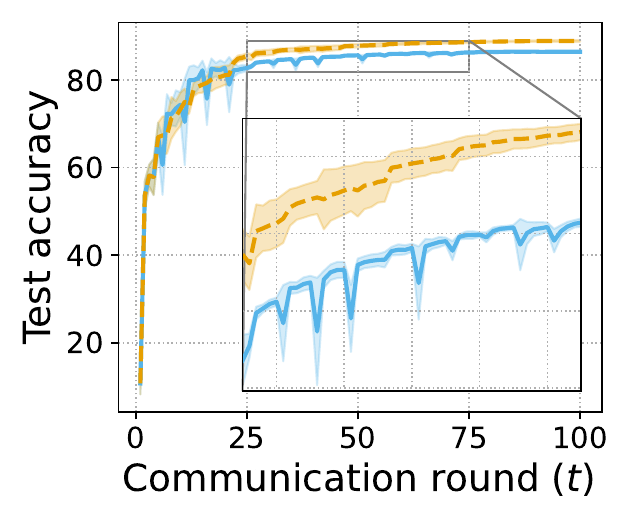}
        \caption{\scriptsize MNIST, $H\!=\!5$}
        \label{fig:mnist_rounds}
    \end{subfigure}
    \centering
        \begin{subfigure}[b]{0.44\linewidth}
        \includegraphics[width=\linewidth]{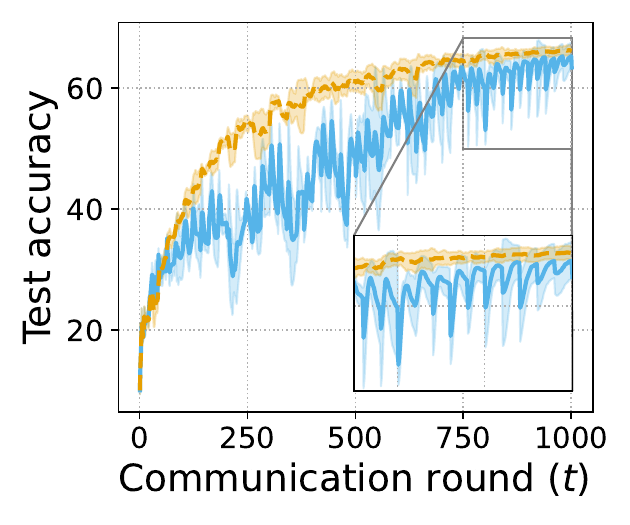}
        \caption{\scriptsize CIFAR-10, $H\!=\!20$}
        \label{fig:cifar_rounds}
    \end{subfigure}
    \caption{Test accuracy over communication rounds for intra IID, inter non-IID heterogeneity, $K/n\!=\!0.2$, ring topology.}
    \label{fig:learning_curves}
\end{figure}

Figures 2 and 3 report the test
accuracy achieved by S2S and S2A on
MNIST and CIFAR-10 datasets, respectively.

\paragraph{Effect of sampling rate (Figs.~\ref{fig:MNIST}--\ref{fig:CIFAR}, left column).} 
For both S2S and S2A, accuracy improves as the sampling rate increases, with an average gain of +2 percentage points (p.p.) between $K/n\!=\!0.2$ and $K/n\!=\!1$.
Interestingly, our experiments confirm the same qualitative heterogeneity regimes identified by our theoretical analysis:
\begin{enumerate}[label=\emph{\textbf{R\arabic*.}},leftmargin=20pt, topsep=0pt, itemsep=0pt, parsep=0pt, partopsep=0pt]
\item \textbf{\emph{Intra IID, Inter IID (Figs.~\ref{fig:MNIST}--\ref{fig:CIFAR}(a)):}}  
S2A outperforms S2S in over 80\% of configurations, although the gain is modest (up to 1 p.p. on the ring for $K/n\!=\!0.2$).
\item \textbf{\emph{Intra non-IID, Inter IID (Figs.~\ref{fig:MNIST}--\ref{fig:CIFAR}(c)):}} S2A outperforms in 40\% of cases (up to +0.5 p.p. on the complete graph with high $K/n$), while S2S prevails in the remaining 60\% (up to +8.4 p.p. on the ring at $K/n\!=\!0.2$).
\item \textbf{\emph{Inter non-IID (Figs.~\ref{fig:MNIST}--\ref{fig:CIFAR}(e,g))}:} S2S outperforms S2A in over 90\% of settings, with the largest gain at $K/n\!=\!0.2$ (+2.4 p.p. on MNIST, +7 p.p. on CIFAR-10).
\end{enumerate}
Across the 96 evaluated configurations, S2S outperforms S2A in about 60\% of cases, S2A in 30\%, and the remaining 10\% are not statistically significant (gap below standard error).
Topology also plays a role: ring accounts for 45\% of the largest gaps, grid for 30\%, and complete graph for 25\%.

\paragraph{Effect of server period (Figs.~\ref{fig:MNIST}--\ref{fig:CIFAR}, right column).} Accuracy decreases as the server period $H$ increases (by an average of -2.4 p.p. from $H\!=\!5$ to $H\!=\!20$), highlighting the importance of frequent D2S communication. Again, our experiments confirm the theoretical regimes from Section~\ref{sec:regimes_theory}:
\begin{enumerate}[label=\emph{\textbf{R\arabic*.}},leftmargin=20pt, topsep=0pt, itemsep=0pt, parsep=0pt, partopsep=0pt]
\item \textbf{\emph{Intra IID, Inter IID (Figs.~\ref{fig:MNIST}--\ref{fig:CIFAR}(b)):}}
S2A consistently outperforms S2S in over 95\% of cases, although the gap remains modest (below 1 p.p. at $H\!=\!5$, ring).
\item \textbf{\emph{Intra non-IID, Inter IID (Figs.~\ref{fig:MNIST}--\ref{fig:CIFAR}(d))}:} S2S outperforms in 70\% of configurations (up to +8.5 p.p. at $H\!=\!5$, ring), whereas S2A prevails in the remaining~30\% (up to +4 p.p. at $H\!=\!10$, complete).
\item \textbf{\emph{Inter non-IID (Figs.~\ref{fig:MNIST}--\ref{fig:CIFAR}(g,h))}:} S2S prevails in over 90\% of configurations, with the largest gap at $H\!=\!5$ (+2 p.p. on MNIST, +7 p.p. on CIFAR-10).
\end{enumerate}
Across 96 comparisons, S2S outperforms in 60\% of them, while S2A in 40\%. Interestingly, in 80\% of heterogeneity regimes, S2S shows a steeper accuracy drop with increasing~$H$, yet it still outperforms S2A in 60\% of these cases. 

\paragraph{Intra vs. Inter Heterogeneity (Fig.~\ref{fig:cifar_heatmaps}).}
Figure~\ref{fig:cifar_heatmaps} compares the accuracy gap (S2S minus S2A) on CIFAR‑10 with ring topology under two opposite heterogeneity regimes. With non-IID intra and IID inter-component heterogeneity (Fig.~\ref{fig:cifar_heatmaps}(a)), S2S prevails at low sampling rates or low server periods (+8.4 p.p. at $K/n\!=\!0.2$, $H\!=\!5$), while S2A prevails for higher $K/n$ or $H$ (+1.2 p.p. at $K/n\!=\!0.6$, $H\!=\!15$). In the opposite regime, with IID intra and non-IID inter heterogeneity (Fig.~\ref{fig:cifar_heatmaps}(b)), S2S consistently outperforms S2A.

\paragraph{Learning curves (Fig.~\ref{fig:learning_curves}).} To better understand why S2S outperforms S2A in the inter non-IID regime, Figure~\ref{fig:learning_curves} reports representative test accuracy over communication rounds. While S2A’s broadcast step initially accelerates inter-component information exchange and achieves higher early-round accuracy, it becomes detrimental in later stages, with periodic drops in test accuracy at every D2S round.

\section{Conclusion}

This paper provides the first theoretical and empirical comparison of two fundamental server-to-device communication primitives for semi-decentralized federated learning: sampled-to-all (S2A) and sampled-to-sampled (S2S). Our results yield practical configuration guidelines: S2S is the better choice when \emph{(i)}~inter-component heterogeneity is high; or \emph{(ii)}~intra-component heterogeneity is high, and the server can sample only a small subset of devices while D2S communication is more frequent. Conversely, when data are nearly IID across components or when a high sampling rate and a well-connected topology mitigate intra-component noise, S2A offers the potential to accelerate convergence.

\section*{Acknowledgments}

This research was supported by the Knut and Alice Wallenberg Foundation; by ELLIIT and the Swedish Research Council (VR); by the French government through the ``Plan de relance'' and the 3IA C\^ote d'Azur Investments in the Future project, managed by the National Research Agency (ANR) under reference ANR-19-P3IA-0002; by the European Network of Excellence dAIEDGE (Grant Agreement No.~101120726) and the EU HORIZON MSCA 2023 DN project FINALITY (Grant Agreement No.~101168816); and by Groupe La Poste, sponsor of the Inria Foundation, within the framework of the FedMalin Inria Challenge.
Experiments presented in this paper were carried out using the Grid'5000 testbed, supported by a scientific interest group hosted by Inria and including CNRS, RENATER and several Universities as well as other organizations (see https://www.grid5000.fr).

\fi

\ifincludemain
\bibliography{aaai2026}
\else
\nobibliography{aaai2026}
\fi

\ifincludeappendix
\appendix
\onecolumn
\normalsize
\setlength{\parskip}{0.6em}
\setlength{\parindent}{0pt}
\setcounter{secnumdepth}{2}  % after \appendix
\setcounter{tocdepth}{2}
\setcounter{lemma}{2}
\setcounter{figure}{5}

\section*{\LARGE APPENDIX\\A Unified Convergence Analysis for Semi-Decentralized Learning:\\Sampled-to-All vs. Sampled-to-Sampled Communication}

\bigskip

The appendix is organized as follows:
\begin{itemize}[label=, leftmargin=10pt, topsep=0pt, itemsep=0pt, parsep=0pt, partopsep=0pt]
    \item \textbf{Appendix~\ref{app:sec:unified_framework} \hspace{1em} Unified Framework for the Convergence Analysis}
    \item \textbf{Appendix~\ref{app:sec:S2S} \hspace{1em} Convergence Analysis of S2S}
    \item \textbf{Appendix~\ref{app:sec:convergenceS2A} \hspace{1em} Convergence Analysis of S2A}
    \item \textbf{Appendix~\ref{sec:additional_theory} \hspace{1em} Additional Theoretical Results}
    \item \textbf{Appendix~\ref{app:sec:experiments} \hspace{1em} Additional Experimental Results}
\end{itemize}

\begin{figure}[h]
\centering
\renewcommand{\thealgorithm}{2.A}
\begin{minipage}[t]{0.48\textwidth}
\begin{algorithm}[H]
\caption{\textsc{S2S --- Vector Notation}}
\label{alg:sts_vec}
\textbf{Input:} initial parameters $\bm{x}_i^{(0)} = \bm{x}^{(0)} \in\mathbb{R}^{d}$ for all $i \in \mathcal{V}$, communication rounds $T$, server aggregation period $H$,
stepsizes $\{\eta_t\}$, mixing distribution $\mathcal{W}$
\begin{algorithmic}[1]
\FOR{$t=0,\dots,T-1$}
\STATE sample mixing matrix $W^{(t)} \sim \mathcal{W}$
\FOR{each device $i \in \mathcal{V}$, \emph{in parallel}}
\STATE sample batch $\mathcal{B}_i^{(t)}$ and compute $\nabla F_i(\bm{x}_i^{(t)}, \mathcal{B}_i^{(t)})$
\STATE $\bm{x}_i^{\left(t+\nicefrac13\right)} = \bm{x}_i^{(t)} - \eta_{t} \nabla F_i(\bm{x}_i^{(t)}, \mathcal{B}_i^{(t)})$ 
\STATE $\bm{x}_i^{(t+\nicefrac23)} = \sum_{j=1}^{n} W_{ij}^{(t)} \bm{x}_j^{\left(t+\nicefrac13\right)}$
\ENDFOR
\IF{$t \in \mathcal{H}$}
      \STATE sample devices $\mathcal S^{(t)} \subseteq \mathcal{V}$, $|\mathcal{S}^{(t)}|=K$
      \STATE compute $\hat{\bm{x}}^{(t+1)} = \frac{1}{K} \sum_{i\in\mathcal{S}^{(t)}} \bm{x}_{i}^{(t+\nicefrac{2}{3})}$
      \STATE send  $\hat{\bm{x}}^{(t+1)}$ to the \emph{sampled devices} only: \\
      $ \bm{x}_i^{(t+1)} = \begin{cases}
          \hat{\bm{x}}^{(t+1)}, & i \in \mathcal{S}^{(t)} \\
          \bm{x}_i^{(t+\nicefrac{2}{3})}, & \text{otherwise}
      \end{cases} $
\ELSE
      \STATE $\bm{x}_i^{(t+1)} = \bm{x}_{i}^{(t+\nicefrac{2}{3})}$
\ENDIF
\ENDFOR
\STATE\textbf{return} $\{\bm{x}_i^{(T)}\}_{i \in \mathcal{V}}$
\end{algorithmic}
\end{algorithm}
\end{minipage}
\hfill
\renewcommand{\thealgorithm}{3.A}
\begin{minipage}[t]{0.48\textwidth}
\begin{algorithm}[H]
\caption{\textsc{S2A --- Vector Notation}}
\label{alg:sta_vec}
\textbf{Input:} initial parameters $\bm{x}_i^{(0)} = \bm{x}^{(0)} \in\mathbb{R}^{d}$ for all $i \in \mathcal{V}$, communication rounds $T$, server aggregation period $H$,
stepsizes $\{\eta_t\}$, mixing distribution $\mathcal{W}$
\begin{algorithmic}[1]
\FOR{$t=0,\dots,T-1$}
\STATE sample mixing matrix $W^{(t)} \sim \mathcal{W}$
\FOR{each device $i \in \mathcal{V}$, \emph{in parallel}}
\STATE sample batch $\mathcal{B}_i^{(t)}$ and compute $\nabla F_i(\bm{x}_i^{(t)}, \mathcal{B}_i^{(t)})$
\STATE $\bm{x}_i^{\left(t+\nicefrac13\right)} = \bm{x}_i^{(t)} - \eta_{t} \nabla F_i(\bm{x}_i^{(t)}, \mathcal{B}_i^{(t)})$ 
\STATE $\bm{x}_i^{(t+\nicefrac23)} = \sum_{j=1}^{n} W_{ij}^{(t)} \bm{x}_j^{\left(t+\nicefrac13\right)}$
\ENDFOR
\IF{$t \in \mathcal{H}$}
      \STATE sample devices $\mathcal S^{(t)} \subseteq \mathcal{V}$, $|\mathcal{S}^{(t)}|=K$
      \STATE compute $\hat{\bm{x}}^{(t+1)} = \frac{1}{K} \sum_{i\in\mathcal{S}^{(t)}} \bm{x}_{i}^{(t+\nicefrac{2}{3})}$
      \STATE broadcast $\hat{\bm{x}}^{(t+1)}$ to \emph{all devices}: \\
      \vspace{1em}
      $\bm{x}_i^{(t+1)} = \hat{\bm{x}}^{(t+1)}$ for all $i \in \mathcal{V}$
      \vspace{0.9em}
\ELSE
      \STATE $\bm{x}_i^{(t+1)} = \bm{x}_{i}^{(t+\nicefrac{2}{3})}$
\ENDIF
\ENDFOR
\STATE\textbf{return} $\{\bm{x}_i^{(T)}\}_{i \in \mathcal{V}}$
\end{algorithmic}
\end{algorithm}
\end{minipage}
\renewcommand{\thealgorithm}{2.B}
\begin{minipage}[t]{0.48\textwidth}
\begin{algorithm}[H]
\caption{\textsc{S2S --- Matrix Notation}}
\label{alg:sts_mat_app}
\textbf{Input:} initial parameters $X^{(0)}\!\in\!\mathbb{R}^{d\times n}$, communication rounds $T$, server aggregation period $H$,
stepsizes $\{\eta_t\}$, mini-batches $\mathcal{B}^{(t)}$, mixing distribution $\mathcal{W}$
\begin{algorithmic}[1]
\FOR{$t=0,\dots,T-1$}
\STATE sample mixing matrix $W^{(t)} \sim \mathcal{W}$
\STATE $X^{(t+\nicefrac13)} \gets X^{(t)}-\eta_t\nabla F(X^{(t)},\mathcal{B}^{(t)})$
\STATE $X^{(t+\nicefrac23)} \gets X^{(t+\nicefrac13)}W^{(t)}$
\IF{$t \in \mathcal{H}$}
      \STATE sample devices $\mathcal S^{(t)} \subseteq \mathcal{V}$, $|\mathcal{S}^{(t)}|=K$
      \STATE build 
      $(W_{\text{S2S}}^{(t)})_{ij} = \begin{cases}
            \frac{1}{K}, & i,j \in \mathcal{S}^{(t)} \\
            1, & i = j \notin \mathcal{S}^{(t)} \\
            0, & \text{otherwise}
        \end{cases}$
      \STATE $X^{(t+1)} \gets X^{(t+\nicefrac23)}W_{\text{S2S}}^{(t)}$
\ELSE
      \STATE $X^{(t+1)} \gets X^{(t+\nicefrac23)}$
\ENDIF
\ENDFOR
\STATE\textbf{return} $X^{(T)}$
\end{algorithmic}
\end{algorithm}
\end{minipage}
\hfill
\renewcommand{\thealgorithm}{3.B}
\begin{minipage}[t]{0.48\textwidth}
\begin{algorithm}[H]
\caption{\textsc{S2A --- Matrix Notation}}
\label{alg:sta_mat_app}
\textbf{Input:} initial parameters $X^{(0)}\!\in\!\mathbb{R}^{d\times n}$, communication rounds $T$, server aggregation period $H$,
stepsizes $\{\eta_t\}$, mini-batches $\mathcal{B}^{(t)}$, mixing distribution $\mathcal{W}$
\begin{algorithmic}[1]
\FOR{$t=0,\dots,T-1$}
\STATE sample mixing matrix $W^{(t)} \sim \mathcal{W}$
\STATE $X^{(t+\nicefrac13)} \gets X^{(t)}-\eta_t\nabla F(X^{(t)},\mathcal{B}^{(t)})$
\STATE $X^{(t+\nicefrac23)} \gets X^{(t+\nicefrac13)}W^{(t)}$
\IF{$t \in \mathcal{H}$}
      \STATE sample devices $\mathcal S^{(t)} \subseteq \mathcal{V}$, $|\mathcal{S}^{(t)}|=K$
      \vspace{0.6em}
      \STATE build 
      $
        (W_{\text{S2A}}^{(t)})_{ij} = \begin{cases}
        \frac{1}{K}, & i \in \mathcal{S}^{(t)}; \\
        0, & \text{otherwise.}
        \end{cases}
    $
      \vspace{0.7em}
      \STATE $X^{(t+1)} \gets X^{(t+\nicefrac23)}W_{\text{S2A}}^{(t)}$
\ELSE
      \STATE $X^{(t+1)} \gets X^{(t+\nicefrac23)}$
\ENDIF
\ENDFOR
\STATE\textbf{return} $X^{(T)}$
\end{algorithmic}
\end{algorithm}
\end{minipage}
\end{figure}

\section{Unified Framework for the Convergence Analysis}
\label{app:sec:unified_framework}

We can now rewrite Algorithm~\ref{alg:mat}, separately for S2S and S2A, in both vector and matrix form: Algorithms~\ref{alg:sts_vec} and~\ref{alg:sts_mat_app} for S2S, and Algorithms~\ref{alg:sta_vec} and~\ref{alg:sta_mat_app} for S2A. The following notation extends the one in the main text:
\begin{align*}
    \mathcal{H} &\coloneqq \Bigl\{ t \leq T \mid t \equiv 0 \bmod H \Bigr\}, \\
    X^{(t)} &\coloneqq \left[ \bm{x}_1^{(t)}, \dots, \bm{x}_n^{(t)} \right] \in \mathbb{R}^{d \times n}, \\
    \bar{X}^{(t)} &\coloneqq \left[ \bar{\bm{x}}^{(t)}, \dots, \bar{\bm{x}}^{(t)} \right] \in \mathbb{R}^{d \times n}, \\
    \nabla F(X^{(t)},\xi^{(t)}) &\coloneqq \left[ \nabla F_1(\bm{x}_1^{(t)},\xi_1^{(t)}), \dots, \nabla F_n(\bm{x}_n^{(t)},\xi_n^{(t)}) \right] \in \mathbb{R}^{d \times n}.
\end{align*}

\subsection{Descent Lemmas}

The following lemmas bound the per-iteration descent from iterate $\bar{\bm{x}}^{(t+1)}$ to iterate $\bar{\bm{x}}^{(t)}$ in terms of disagreement error ($\frac{1}{n} \mathbb{E} \| X^{(t)} - \bar{X}^{(t)} \|_F^2$) and bias error ($\mathbb{E} \| \bar{\bm{x}}^{(t+1)} - \bar{\bm{x}}^{(t+\frac{2}{3})} \|_2^2$). 

Lemma~\ref{lem:descent} assumes that the local objectives are convex, while Lemma~\ref{lem:descent_nc} is for non-convex objectives.

\begin{lemma}[Descent Lemma (Convex Objectives)] 
\label{lem:descent}
Under Assumptions~1--6, for all $t \geq 0$, the average $\bar{\bm{x}}^{(t)} \coloneqq \frac{1}{n} \sum_{i=1}^{n} \bm{x}_i^{(t)}$ of the
iterates of Algorithms~1--2 with the stepsize \( \eta_{t} \leq \frac{1}{4L} \) satisfies:
\[
\mathbb{E} \left\| \bar{\bm{x}}^{(t+1)} - \bm{x}^* \right\|_2^2 
\leq \mathbb{E} \left\| \bar{\bm{x}}^{(t)} - \bm{x}^* \right\|_2^2 - \eta_{t} \mathbb{E} \left( f(\bar{\bm{x}}^{(t)}) - f^* \right) 
+ \frac{\eta_{t}^{2} \bar{\sigma}^2}{n} + \frac{3\eta_{t}L}{2} \frac{1}{n} \mathbb{E} \left\| X^{(t)} - \bar{X}^{(t)} \right\|_F^2 
+ \mathbb{E} \left\| \bar{\bm{x}}^{(t+1)} - \bar{\bm{x}}^{(t+\frac{2}{3})} \right\|_2^2,
\]
where the term $\Xi^{(t)} \coloneqq \frac{1}{n} \mathbb{E} \| X^{(t)} - \bar{X}^{(t)} \|_F^2 = \frac{1}{n} \sum_{i=1}^{n} \mathbb{E} \| \bm{x}_i^{(t)} - \bar{\bm{x}}^{(t)} \|_2^2$ is the disagreement error, \\
and the term $\mathbb{E} \| \bar{\bm{x}}^{(t+1)} - \bar{\bm{x}}^{(t+\frac{2}{3})} \|_2^2$ is the bias error.
\end{lemma}
\begin{proof}[Proof of Lemma~\ref{lem:descent}]  
    We decompose the optimality gap into:
    \begin{align*}
        \mathbb{E} \left\| \bar{\bm{x}}^{(t+1)} - \bm{x}^* \right\|_2^2
        =
        \mathbb{E} \left\| \bar{\bm{x}}^{(t+1)} - \bar{\bm{x}}^{(t+\frac{2}{3})} \right\|_2^2 + \mathbb{E} \left\| \bar{\bm{x}}^{(t+\frac{2}{3})} - \bm{x}^* \right\|_2^2,
    \end{align*}
    where orthogonality follows from $\mathbb{E}_{\mathcal{S}^{(t)}}[\bar{\bm{x}}^{(t+1)}] = \bar{\bm{x}}^{(t+\frac{2}{3})}$ for all $t \geq 0$, which holds immediately for $t \notin \mathcal{H}$, and is proved in the subsequent Lemmas~8~\emph{(ii)} and~11~\emph{(iii)} for $t \in \mathcal{H}$.

    The bias error $\mathbb{E} \| \bar{\bm{x}}^{(t+1)} - \bar{\bm{x}}^{(t+\frac{2}{3})} \|_2^2$ is a key difference for S2S and S2A. We bound it separately in Lemmas~\ref{prop:sampling:averaging} and~\ref{lem:sampling_error}.
    
    The term $\mathbb{E} \| \bar{\bm{x}}^{(t+\nicefrac{2}{3})} - \bm{x}^*\|_2^2$ follows the D-SGD descent lemma for convex objectives. For this reason, we refer the reader to~\cite[Lemma~8]{koloskovaUnifiedTheoryDecentralized2020} and \cite[Lemma~1]{barsRefinedConvergenceTopology2023}.
\end{proof}

\begin{lemma}[Descent Lemma (Non-Convex Objectives)] 
\label{lem:descent_nc} 
Under Assumptions~1 and 3--6, for all $t\!\geq\!0$, the average $\bar{\bm{x}}^{(t)}\!\coloneqq\!\frac{1}{n} \sum_{i=1}^{n} \bm{x}_i^{(t)}$ of the
iterates of Algorithms~1--2 with the stepsize \( \eta_{t} \leq \frac{1}{4L} \) satisfies:
\begin{align*}
    \mathbb{E} [f(\bar{\bm{x}}^{(t+1)})]
    \leq
    \mathbb{E} [f(\bar{\bm{x}}^{(t)})] - \frac{\eta_t}{4} \mathbb{E}\left\| \nabla f(\bar{\bm{x}}^{(t)}) \right\|_2^2 + \frac{\eta_t^2 L \bar{\sigma}^2}{2n} + \eta_t L^2 \frac{1}{n} \mathbb{E} \left\| X^{(t)} -  \bar{X}^{(t)} \right\|_F^2 + \frac{L}{2} \mathbb{E} \left\| \bar{\bm{x}}^{(t+1)} - \bar{\bm{x}}^{(t+\frac{2}{3})} \right\|_2^2.
\end{align*}
\end{lemma}
\begin{proof}[Proof of Lemma~\ref{lem:descent_nc}] 
By $L$-smoothness of $f(\cdot)$ (Assumption~1, see~\cite{nesterovIntroductoryLecturesConvex2004}):
\begin{align*}
    \mathbb{E} [f(\bar{\bm{x}}^{(t+1)})]
    &\leq
    f(\bar{\bm{x}}^{(t)}) + \mathbb{E} \left\langle \nabla f(\bar{\bm{x}}^{(t)}), \bar{\bm{x}}^{(t+1)} - \bar{\bm{x}}^{(t)} \right\rangle + \frac{L}{2} \mathbb{E} \left\| \bar{\bm{x}}^{(t+1)} - \bar{\bm{x}}^{(t)} \right\|_2^2.
\end{align*}

For the term $\mathbb{E} \| \bar{\bm{x}}^{(t+1)} - \bar{\bm{x}}^{(t)} \|_2^2$, we again invoke the error decomposition:
\begin{align*}
    \mathbb{E} \left\| \bar{\bm{x}}^{(t+1)} - \bar{\bm{x}}^{(t)} \right\|_2^2 
    =
    \left\| \bar{\bm{x}}^{(t+1)} - \bar{\bm{x}}^{(t+\frac{2}{3})} \right\|_2^2
    +
    \left\| \bar{\bm{x}}^{(t+\frac{2}{3})} - \bar{\bm{x}}^{(t+1)} \right\|_2^2,
\end{align*}
where the cross product is zero because $\mathbb{E}_{\mathcal{S}^{(t)}}[\bar{\bm{x}}^{(t+1)}] = \bar{\bm{x}}^{(t+\frac{2}{3})}$.

Therefore:
\begin{align*}
    \mathbb{E} [f(\bar{\bm{x}}^{(t+1)})]
    \leq
    \mathbb{E} [f(\bar{\bm{x}}^{(t)})] + \mathbb{E} \left\langle \nabla f(\bar{\bm{x}}^{(t)}), \bar{\bm{x}}^{(t+\frac{2}{3})} - \bar{\bm{x}}^{(t)} \right\rangle + \frac{L}{2} \mathbb{E} \left\| \bar{\bm{x}}^{(t+\frac{2}{3})} - \bar{\bm{x}}^{(t)} \right\|_2^2 + \frac{L}{2} \mathbb{E} \left\| \bar{\bm{x}}^{(t+1)} - \bar{\bm{x}}^{(t+\frac{2}{3})} \right\|_2^2.
\end{align*}

Once we isolated the bias error, the remaining terms follow the standard descent lemma for D-SGD on non-convex objectives; we therefore refer the reader to~\citep[Lemma~11]{koloskovaUnifiedTheoryDecentralized2020} and \citep[Lemma~2]{barsRefinedConvergenceTopology2023}.
\end{proof}

\subsection{D2D Disagreement Errors}

The next lemma bounds the disagreement error after the D2D round, $\mathbb{E}\| X^{(t+\frac{2}{3})} - \bar{X}^{(t+\frac{2}{3})} \|_F^2$, common to both S2S and S2A.
We will bound the D2S disagreement error, $\mathbb{E} \| X^{(t+1)} - \bar{X}^{(t+1)} \|_F^2$, specific to each primitive, separately in Lemmas~\ref{lem:sampling:consensus_intra_inter} and~\ref{consensus_sta}.

\begin{lemma}[Disagreement Errors (D2D)] 
\label{consensus_intra_inter}
For all $t \geq 0$, 

\textbf{Error Decomposition.} We decompose the D2D disagreement into intra- and inter-component terms:
\begin{align*}
    \mathbb{E} \left\| X^{(t+\frac{2}{3})} - \bar{X}^{(t+\frac{2}{3})} \right\|_F^2 = 
    \mathbb{E} \left\| X^{(t+\frac{2}{3})} - X^{(t+\frac{2}{3})} \Pi_C \right\|_F^2 + \mathbb{E} \left\| X^{(t+\frac{2}{3})} \Pi_C - \bar{X}^{(t+\frac{2}{3})} \right\|_F^2.
\end{align*}

\textbf{Intra-Component Disagreement.} Under Assumptions 1 and~3--6, for \( \eta_{t} \leq \frac{p}{8L} \):
\begin{align*}
    \mathbb{E} \left\| X^{(t+\frac{2}{3})} - X^{(t+\frac{2}{3})} \Pi_C \right\|_F^2 
    &\leq  \left( 1 - \frac{p}{4} \right) \mathbb{E} \left\| X^{(t)} - {X}^{(t)} \Pi_C \right\|_F^2 
    +  \frac{6n \eta^{2} \bar{\zeta}_{\text{intra}}^2}{p}.  
\end{align*}
\textbf{Inter-Component Disagreement.} Under Assumptions 1 and~3--6, for \( \eta_{t} \leq \frac{p}{8L} \), there exists $\rho>0$ such that:
\begin{align*}
    \mathbb{E} \left\| X^{(t+\frac{2}{3})} \Pi_C - \bar{X}^{(t+\frac{2}{3})} \right\|_F^2 
    &\leq (1 + \rho) \mathbb{E} \left\| X^{(t)} \Pi_C - \bar{X}^{(t)} \right\|_F^2 + (1 + \rho^{-1}) n \eta^2 \bar{\zeta}_{\text{inter}}^2.
\end{align*}
The D2D round alone does not contract the inter-component disagreement $(\rho>0)$.
\end{lemma}
\begin{proof}[Proof of Lemma~\ref{consensus_intra_inter}] 
The error decomposition follows from Lemma~1 in the main paper.

\emph{For the intra-component disagreement},
\begin{align*}
    \mathbb{E} \left\| X^{(t+\frac{2}{3})} - X^{(t+\frac{2}{3})} \Pi_C \right\|_F^2 
    &\stackrel{\text{(a)}}{=} \mathbb{E} \left\| X^{(t+\frac{2}{3})} \left( I - \Pi_C \right) \right\|_F^2 \\
    &\stackrel{\text{(b)}}{=} \mathbb{E} \left\| X^{(t+\frac{1}{3})}  W^{(t)} \left( I - \Pi_C \right) \right\|_F^2 \\
    &\stackrel{\text{(c)}}{=} \mathbb{E} \left\| \left( X^{(t)} - \eta \nabla F(X^{(t)}, \xi^{(t)}) \right) \left( W^{(t)} - \Pi_C \right) \right\|_F^2 \\
    &\stackrel{\text{(d)}}{\leq} (1 + \alpha) \mathbb{E} \left\| X^{(t)} \left( W^{(t)} - \Pi_C \right) \right\|_F^2 + (1 + \alpha^{-1}) \eta^2 \mathbb{E} \left\| \nabla F(X^{(t)}, \xi^{(t)}) \left( W^{(t)} - \Pi_C \right) \right\|_F^2 \\
    &\stackrel{\text{(e)}}{\leq} (1 + \alpha)(1 - p) \mathbb{E} \left\| X^{(t)} \left(I  - \Pi_C \right) \right\|_F^2 + (1 + \alpha^{-1}) (1 - p) n \eta^2 \bar{\zeta}_{\text{intra}}^2 \\
    &\stackrel{\text{(f)}}{\leq}  \left( 1 - \frac{p}{4} \right) \mathbb{E} \left\| X^{(t)} - X^{(t)} \Pi_C \right\|_F^2
    + \frac{6n \eta^{2} \bar{\zeta}_{\text{intra}}^2}{p},
\end{align*}
where equalities (a)--(c) follow from the D2D update rule $X^{(t+\frac{2}{3})} = (X^{(t)} - \eta \nabla F(X^{(t)}, \xi^{(t)})) W^{(t)}$; inequality (d) applies $\| \bm{a} + \bm{b} \|_2^2 = (1+\alpha) \| \bm{a} \|_2^2 + (1+\alpha^{-1}) \| \bm{b} \|_2^2$ for any $\alpha > 0$; inequality (e) uses intermediate steps detailed in~\cite[Lemma 3]{barsRefinedConvergenceTopology2023}; and inequality (f) sets $\alpha=\frac{p}{2}$ and uses $p \in (0,1]$, therefore $1-p < 1$.

\emph{For the inter-component disagreement},
\begin{align*}
    \mathbb{E} \left\| X^{(t+\frac{2}{3})} \Pi_C - \bar{X}^{(t+\frac{2}{3})} \right\|_F^2 
    &\stackrel{\text{(g)}}{=} \left\| X^{(t+\frac{2}{3})} \left( \Pi_C - \Pi \right) \right\|_F^2 \\
    &\stackrel{\text{(h)}}{=} \mathbb{E} \left\| X^{(t+\frac{1}{3})} W^{(t)} \left( \Pi_C - \Pi \right) \right\|_F^2 \\
    &\stackrel{\text{(i)}}{=} \mathbb{E} \left\| \left( X^{(t)} - \eta \nabla F(X^{(t)}, \xi^{(t)}) \right) \left( \Pi_C - \Pi \right) \right\|_F^2 \\
    &\stackrel{\text{(j)}}{\leq} (1+\rho) \mathbb{E} \left\| X^{(t)} \left( \Pi_C - \Pi \right) \right\|_F^2 + (1+\rho^{-1}) \eta^2 \mathbb{E} \left\| \nabla F(X^{(t)}, \xi^{(t)}) \left( \Pi_C - \Pi \right) \right\|_F^2 \\
    &\stackrel{\text{(k)}}{\leq} (1+\rho) \mathbb{E} \left\| X^{(t)} \Pi_C - \bar{X}^{(t)} \right\|_F^2 + (1+\rho^{-1}) n \eta^2 \bar{\zeta}_{\text{inter}}^2,
\end{align*} 
where steps (g)--(j) replicate the arguments in (a)--(d), with $\rho > 0$, while inequality (k) is borrowed from~\cite[Lemma 3]{barsRefinedConvergenceTopology2023}. The choice of $\rho$ will be addressed separately for S2S and S2A in Theorems~1 and~2.
\end{proof}

\subsection{Alternating Disagreement Recursion}

The following lemma is a central building block of our unified analysis.
It applies to both S2S and S2A, and is used in the proofs of Theorems~1 and~2.
The recursion parameters, however, will differ for the two primitives, producing different results.

\begin{lemma}[Disagreement Recursion]
\label{prop:recursion}
Let $\{\Xi^{(t)}\}_{t \ge 0}$ be a nonnegative sequence satisfying the recursion
\[
    \Xi^{(t)} \le
    \begin{dcases}
        a_1 \, \Xi^{(t-1)} + b_1, & t \equiv 0 \pmod{H}, \\
        a_2 \, \Xi^{(t-1)} + b_2, & t \not\equiv 0 \pmod{H},
    \end{dcases}
\]
with constants $a_1, a_2, b_1, b_2 \ge 0$, and $H \ge 1$.
Define $0 \leq C \coloneqq a_1 a_2^{H-1} < 1$, and $D \coloneqq a_1 b_2 \frac{1 - a_2^{H-1}}{1 - a_2} + b_1$.

Then, for any horizon $T \ge 0$,
\[
    \bar{\Xi}^{(T+1)} \coloneqq \frac{1}{T+1} \sum_{t=0}^{T} \Xi^{(t)}
    \le
    \begin{dcases}
        \left[ \frac{D}{(1-C)(1-a_2)} + \frac{b_2}{1 - a_2} (H - 1) \right] \left( \frac{1}{H} + \frac{1}{T+1} \right), & 0 \le a_2 < 1, \\
        \left[ \frac{D}{1-C} \frac{a_2^{H} - 1 }{a_2 - 1} + \frac{b_2 (a_2^{H} - a_2 H + H - 1)}{(a_2 - 1)^2} \right] \left( \frac{1}{H} + \frac{1}{T+1} \right), & a_2 > 1.
    \end{dcases}
\]
\end{lemma}

\begin{proof}[Proof of Lemma~\ref{prop:recursion}]
    For any $t = mH + s$, where $m \coloneqq \lfloor \frac{t}{H} \rfloor$ and $s \in \{0, \dots, H-1\}$,
    \begin{align*}
        \Xi^{(mH + s)} &\leq a_2^{s} \Xi^{(mH)} + b_2 \sum_{k=0}^{s - 1} a_2^{k}
        = a_2^{s} \Xi^{(mH)} + b_2 \frac{1 - a_2^{s}}{1 - a_2}.
    \end{align*}
    First, consider the behavior at $t = mH$:
    \begin{align*}
        \Xi^{(mH)} 
        & \leq C \Xi^{((m - 1)H)} + D \\
        &\leq C^{m} \underbrace{\Xi^{(0)}}_{=0} + D \sum_{k=0}^{m - 1} C^{k} = \frac{D}{1 - C} (1 - C^{m}).
    \end{align*}
    Combining:
    \[
    \Xi^{(t)} \leq
    a_2^{s} \left[ \frac{D}{1 - C} (1 - C^{m}) \right] + b_2 \frac{1 - a_2^{s}}{1 - a_2}.
    \]
    Second, sum over one full period ($s = 0, \dots, H-1$):
    \begin{align*}
        \sum_{s=0}^{H-1} \Xi^{(mH + s)}
        &\leq
        \frac{D}{1-C} \underbrace{ \sum_{s=0}^{H-1} a_2^{s}}_{\coloneqq S_1} - \underbrace{\frac{C^{m}}{1 - C} \sum_{s=0}^{H-1} a_2^{s}}_{\text{decays as } C^m,~C<1} + \frac{b_2}{1 - a_2} \underbrace{\ \sum_{s=1}^{H-1} (1 - a_2^{s})}_{\coloneqq S_2} \\
        &\leq  \underbrace{ \frac{D}{1-C} S_1 + \frac{b_2}{1 - a_2} S_2}_{\coloneqq Q}.
    \end{align*}
    Third, sum up to $T = M H + S$, with $M \coloneqq \lfloor \frac{T}{H} \rfloor \leq (T+H-1)/H$, $S \in \{ 0, \dots, H-1 \}$.

    The contribution of the $M$ complete periods is:
    \begin{align*}
        \sum_{m=0}^{M-1} \sum_{s=0}^{H-1} \Xi^{(mH + s)} \leq M Q.
    \end{align*}

    The contribution of the partial period $S$ is:
    \begin{align*}
        \sum_{s=0}^{S-1} \Xi^{(MH + s)} \leq Q.
    \end{align*}

    Combining:
    \begin{align*}
        \sum_{t=1}^{T} \Xi^{(t)} \leq (M+1) Q \leq \frac{T+H}{H}  Q.
    \end{align*}
    Next, divide by $T+1$:
    \begin{align*}
        \bar{\Xi}^{(T+1)} \coloneqq \frac{1}{T+1} \sum_{t=0}^{T} \Xi^{(t)} 
        &\leq \frac{T+H}{H(T+1)}  Q \leq \left( \frac{1}{H} + \frac{1}{T+1} \right) \left[ \frac{D}{1-C} S_1 + \frac{b_2}{1 - a_2} S_2 \right].
    \end{align*}

    Finally, consider the two signs of $a_2$ separately. \\
    \emph{Contractive step ($0 \leq a_2 < 1$).} $S_1 = \sum_{s=0}^{H-1} a_2^{s} = \frac{1 - a_2^H}{1 - a_2} \leq \frac{1}{1 - a_2}$ and $S_2 \leq H - 1 - \frac{a_2}{1 - a_2} \leq H-1$:
    \begin{align*}
        \bar{\Xi}^{(T+1)} \leq \left( \frac{1}{H} + \frac{1}{T+1} \right) \left[ \frac{D}{(1-C)(1-a_2)} + \frac{b_2}{1 - a_2} (H - 1) \right].
    \end{align*}

    \emph{Expansive step ($a_2 > 1$).} $S_1 = \sum_{s=0}^{H-1} a_2^{s} = \frac{a_2^H - 1 }{a_2 - 1}$ and $S_2 = \sum_{s=1}^{H-1} (1 - a_2^{s}) = - \frac{a_2^H - a_2H + H -1}{a_2 - 1}$:
    \begin{align*}
        \bar{\Xi}^{(T+1)} \leq \left( \frac{1}{H} + \frac{1}{T+1} \right) \left[ \frac{D}{1-C} \frac{a_2^H - 1 }{a_2 - 1} + \frac{b_2(a_2^H - a_2H + H -1)}{(a_2 - 1)^2} \right].
    \end{align*}
\end{proof}  

\subsection{Alternating Convergence Recursion}

The next lemma telescopes the descent recursion over a time horizon $T$, alternating D2D and D2S rounds with period $H$. 

\begin{lemma}[Convergence Recursion]
\label{lem:convergence_recursion}
Let $\mathcal{H} := \{ t \leq T: t \equiv 0 \pmod H \}$. \\
Consider a nonnegative sequence $\{ r^{(t)} \}_{t \ge 0}$ satisfying the descent recursion
\[
  r^{(t+1)} \;\le\;
  \begin{cases}
      r^{(t)} - \eta\Delta^{(t)}
      + \dfrac{\eta^2\bar\sigma^{2}}{n}
      + a\,\Xi_{\text{intra}}^{(t)}
      + b\,\Xi_{\text{inter}}^{(t)}
      + c+d, & t\in\mathcal H,\\[6pt]
      r^{(t)} - \eta\Delta^{(t)}
      + \dfrac{\eta^2\bar\sigma^{2}}{n}
      + e\!\left(\Xi_{\text{intra}}^{(t)}+\Xi_{\text{inter}}^{(t)}\right),
      & t\notin\mathcal H .
  \end{cases}
\]
Then, for any horizon $T \ge 0$,
\[
  \frac{1}{T+1}\sum_{t=0}^{T}\Delta^{(t)}
  \;\le\;
    \frac{r^{(0)}}{\eta\,(T+1)}
    +\frac{\eta\,\bar\sigma^{2}}{n}
    +\frac{e+\frac{a-e}{H}}{\eta}\,\bar\Xi_{\text{intra}}^{(T+1)}
    +\frac{e+\frac{b-e}{H}}{\eta}\,\bar\Xi_{\text{inter}}^{(T+1)}
    +\frac{c+d}{\eta\,H},
\]
where $\bar\Xi_{\text{intra}}^{(T+1)}
  \coloneqq \frac1{T+1}\sum_{t=0}^{T}\Xi_{\text{intra}}^{(t)}$, and $
  \bar\Xi_{\text{inter}}^{(T+1)}
  \coloneqq \frac1{T+1}\sum_{t=0}^{T}\Xi_{\text{inter}}^{(t)}$.
\end{lemma}
\begin{proof}[Proof of Lemma~\ref{lem:convergence_recursion}]
First, isolate $\Delta^{(t)}$ by rearranging the descent recursion and dividing by $\eta$:
\[
  \Delta^{(t)}
  \;\le\;
  \frac{r^{(t)}-r^{(t+1)}}{\eta}
  +\frac{\eta\,\bar\sigma^{2}}{n}
  +\begin{cases}
      \dfrac{a}{\eta}\,\Xi_{\text{intra}}^{(t)}
      +\dfrac{b}{\eta}\,\Xi_{\text{inter}}^{(t)}
      +\dfrac{c+d}{\eta}, & t\in\mathcal H,\\[8pt]
      \dfrac{e}{\eta}\bigl(\Xi_{\text{intra}}^{(t)}+\Xi_{\text{inter}}^{(t)}\bigr),
      & t\notin\mathcal H .
    \end{cases}
\]
Next, sum and average over $t=0, \dots, T$:
\begin{align*}
  \frac{1}{T+1}\sum_{t=0}^{T}\Delta^{(t)}
  &\le \frac{r^{(0)}}{\eta\,(T+1)}
  +\frac{\eta\,\bar\sigma^{2}}{n}\\
  &\quad
  +\frac1{T+1}\sum_{t=0}^{T}
     \frac{\mathds1_{t\in\mathcal H}\,a+\mathds1_{t\notin\mathcal H}\,e}{\eta}\,
     \Xi_{\text{intra}}^{(t)}
  +\frac1{T+1}\sum_{t=0}^{T}
     \frac{\mathds1_{t\in\mathcal H}\,b+\mathds1_{t\notin\mathcal H}\,e}{\eta}\,
     \Xi_{\text{inter}}^{(t)}
  +\frac1{T+1}\sum_{t=0}^{T}
     \frac{\mathds1_{t\in\mathcal H}(c+d)}{\eta}.
\end{align*}
Observe that $\frac{|\mathcal H|}{T+1}\le\frac1H$, because exactly one index in every block of length $H$ lies in $\mathcal H$.

Consequently, 
\begin{align*}
    \frac{1}{T+1}\sum_{t=0}^{T}
    (\mathds1_{t\in\mathcal H}\,a+\mathds1_{t\notin\mathcal H}\,e)\,
    \Xi_{\text{intra}}^{(t)}
    \le
    \Bigl(e+\frac{a-e}{H}\Bigr)
    \bar\Xi_{\text{intra}}^{(T+1)},
\end{align*}
and an analogous bound holds for the inter-component term. 
Moreover,
$\tfrac{1}{T+1}\sum_{t=0}^{T}\mathds1_{t\in\mathcal H}
  \le\frac1H.$
\end{proof}

\section{Convergence Analysis of S2S}
\label{app:sec:S2S}

\subsection{Properties of S2S}
\label{app:subsec:S2S}

\begin{lemma}[Sampled-to-Sampled]
\label{lem:sampling_app}
Define:
\begin{align*}
    (W_{\text{S2S}}^{(t)})_{ij} = \begin{dcases}
            \frac{1}{K}, & i,j \in \mathcal{S}^{(t)}, \\
            1, & i = j \notin \mathcal{S}^{(t)}, \\
            0, & \text{otherwise}.
        \end{dcases}
\end{align*}
The matrix $W_{\text{S2S}}^{(t)}$ satisfies the following properties:
\begin{enumerate}[label=(\roman*), leftmargin=20pt]
    \item \textbf{Symmetric and stochastic.} $(W_{\text{S2S}}^{(t)})^\top = W_{\text{S2S}}^{(t)}$, and $W_{\text{S2S}}^{(t)} \bm{1} = \bm{1}$.

    Consequently, $W_{\text{S2S}}^{(t)} \Pi = \Pi W_{\text{S2S}}^{(t)} = \Pi$.
    
    \item \textbf{The bias error is zero.} $W_{\text{S2S}}^{(t)}$ preserves the average of the iterates between D2D and D2S rounds:
    \begin{align*}
        \bar{X}^{(t+1)} = X^{(t+\nicefrac{2}{3})} W_{\text{S2S}}^{(t)} \Pi = X^{(t+\nicefrac{2}{3})} \Pi \eqqcolon \bar{X}^{(t+\nicefrac{2}{3})}.
    \end{align*}

    \item \textbf{Disagreement error.} The matrix $W_{\text{S2S}}^{(t)}$ leaves residual disagreement:
    \begin{align*}
        X^{(t+1)} = X^{(t+\nicefrac{2}{3})} W_{\text{S2S}}^{(t)} 
        \qquad \neq \qquad
        \bar{X}^{(t+1)} \coloneqq X^{(t+\nicefrac{2}{3})} W_{\text{S2S}}^{(t)} \Pi = X^{(t+\nicefrac{2}{3})} \Pi.
    \end{align*}
    
    We define the disagreement error at time $t+1$ as:
    \begin{align*}
        n \Xi^{(t+1)} \coloneqq \mathbb{E} \left\| X^{(t+1)} -  \bar{X}^{(t+1)} \right\|_F^2 = \mathbb{E} \left\| X^{(t+\frac{2}{3})} (W_{\text{S2S}}^{(t)} - \Pi) \right\|_F^2 \geq 0.
    \end{align*}
    
    For $t \in \mathcal{H}$, the disagreement error satisfies:
    \begin{align*}
        \mathbb{E} \left[ \left\| X^{(t+1)} - \bar{X}^{(t+1)} \right\|_F^2 \right]
        &=
        \mathbb{E} \left[ \left\| X^{(t+\frac{2}{3})} \left( W_{\text{S2S}}^{(t)} - \Pi \right) \right\|_F^2 \right] \\
        &=
        \frac{n-K}{n-1} \left\| X^{(t+\frac{2}{3})} (I - \Pi) \right\|_F^2 \\
        &=
        \frac{n-K}{n-1} \left\| X^{(t+\frac{2}{3})} - \bar{X}^{(t+\frac{2}{3})} \right\|_F^2.
        \label{eq:deviation_sts}
    \end{align*}
    \begin{itemize}
        \item For $K=1$, the factor $\frac{n-K}{n-1}$ equals $1$, indicating no contraction.
        \item For $K=n$, the factor $\frac{n-K}{n-1}$ equals $0$, corresponding to full contraction.
    \end{itemize}
\end{enumerate}
\end{lemma}
\begin{proof}[Proof of Lemma~\ref{lem:sampling_app} (iii)]
    The result is a consequence of the variance of sampling without replacement:
    \begin{align*}
        \sum_{i=1}^{n} \mathbb{E}_{\mathcal{S}^{(t)}} \left\| \bm{x}_{i}^{(t+1)} - \bar{\bm{x}}^{(t+1)} \right\|_2^2 
        &= \sum_{i=1}^{n} \mathbb{E}_{\mathcal{S}^{(t)}} \Biggl\| \mathds{1}_{i\in\mathcal{S}^{(t)}} \Biggl[ \frac{1}{K} \sum_{j\in\mathcal{S}^{(t)}} \left(\bm{x}_j^{(t+\frac{2}{3})} - \bar{\bm{x}}^{(t+\frac{2}{3})}\right) \Biggr] + \mathds{1}_{i\not\in\mathcal{S}^{(t)}} \left( \bm{x}_i^{(t+\frac{2}{3})} - \bar{\bm{x}}^{(t+\frac{2}{3})} \right)^2 \Biggr\|_2^2 \\
        &= K \underbrace{\mathbb{E}_{\mathcal{S}^{(t)}} \Biggl\| \frac{1}{K} \sum_{j\in\mathcal{S}^{(t)}} \left(\bm{x}_j^{(t+\frac{2}{3})} - \bar{\bm{x}}^{(t+\frac{2}{3})}\right) \Biggl\|_2^2}_{\text{bounded in Lemma~\ref{lem:sampling_error}}} + \frac{n-K}{n} \sum_{i=1}^{n} \left\| \bm{x}_i^{(t+\frac{2}{3})} - \bar{\bm{x}}^{(t+\frac{2}{3})} \right\|_2^2 \\
        &= \frac{n-K}{n(n-1)} \sum_{i=1}^{n} \left\| \bm{x}_i^{(t+\frac{2}{3})} - \bar{\bm{x}}^{(t+\frac{2}{3})} \right\|_2^2 + \frac{n-K}{n} \sum_{i=1}^{n} \left\| \bm{x}_i^{(t+\frac{2}{3})} - \bar{\bm{x}}^{(t+\frac{2}{3})} \right\|_2^2 \\
        &= \frac{n-K}{n-1} \sum_{i=1}^{n} \left\| \bm{x}_i^{(t+\frac{2}{3})} - \bar{\bm{x}}^{(t+\frac{2}{3})} \right\|_2^2.
    \end{align*}
\end{proof}

\subsection{Intermediate Lemmas}

\begin{lemma}[S2S: Bias Error]
\label{prop:sampling:averaging}
For every $t \geq 0$,
\begin{align*}
    \bar{X}^{(t+1)} = \bar{X}^{(t+\frac{2}{3})}.
\end{align*}
\end{lemma}
\begin{proof}[Proof of Lemma~\ref{prop:sampling:averaging}] 
Let $\mathcal{H} \coloneqq \{ t \leq T \mid t \equiv 0 \bmod H \}$. 
We have: 
\begin{itemize}
    \item For D2D rounds (\( t \notin \mathcal{H} \)) , by definition,  
    \( X^{(t+1)} = X^{(t+\frac{2}{3})} \) and $\bar{X}^{(t+1)} = \bar{X}^{(t+\frac{2}{3})}$.
    
    \item For D2S rounds ($t \in \mathcal{H}$), by Lemma~\ref{lem:sampling_app} (\emph{ii}), $\bar{X}^{(t+1)} = \bar{X}^{(t+\frac{2}{3})}$.
\end{itemize}
\end{proof}

\begin{lemma}[S2S: Disagreement Error]
\label{lem:sampling:consensus_intra_inter} 
For every $t \geq 0$,
\begin{align*}
    \mathbb{E} \left\| X^{(t+1)} - \bar{X}^{(t+1)} \right\|_F^2 = \begin{dcases}
        \mathbb{E} \left\| X^{(t+\frac{2}{3})} - \bar{X}^{(t+\frac{2}{3})} \right\|_F^2, & t \notin \mathcal{H}; \\
        \frac{n-K}{n-1} \mathbb{E} \left\| X^{(t+\frac{2}{3})} - \bar{X}^{(t+\frac{2}{3})} \right\|_F^2, &  t \in \mathcal{H},
    \end{dcases}
\end{align*}
where $\mathbb{E} \| X^{(t+\frac{2}{3})} - \bar{X}^{(t+\frac{2}{3})} \|_F^2$ is the D2D disagreement error already bounded in Lemma~\ref{consensus_intra_inter}.
\end{lemma}
\begin{proof}[Proof of Lemma~\ref{lem:sampling:consensus_intra_inter}] We have: 
\begin{itemize}
    \item For D2D rounds (\( t \notin \mathcal{H} \)) , by definition,  
    \( X^{(t+1)} = X^{(t+\frac{2}{3})} \).
    
    \item For D2S rounds ($t \in \mathcal{H}$), we apply Lemma~\ref{lem:sampling_app} (\emph{iii}). 
\end{itemize}
\end{proof}

\subsection{Proof of Theorem~1}
\label{app:subsec:theorem_S2S}

\begin{proof}[Proof of Theorem~1 (Convex Objectives)] \phantom{Ciao} \\
Combine Lemma~\ref{lem:descent} and Lemma~\ref{prop:sampling:averaging}. For every $t \geq 0$,
\[
\mathbb{E} \left\| \bar{\bm{x}}^{(t+1)} - \bm{x}^* \right\|_2^2 
\leq \mathbb{E} \left\| \bar{\bm{x}}^{(t)} - \bm{x}^* \right\|_2^2 - \eta \mathbb{E} \left( f(\bar{\bm{x}}^{(t)}) - f^* \right) 
+ \frac{\eta^{2} \bar{\sigma}^2}{n} + \frac{3\eta L}{2} \Xi^{(t)}.
\]

Apply Lemma~\ref{lem:convergence_recursion} with $r^{(t)} = \mathbb{E}\|\bar{\bm{x}}^{(t+1)}-\bm{x}^*\|_2^2$, $\Delta^{(t)} = \mathbb{E}(f(\bar{\bm{x}}^{(t)}) - f^*)$, $a=b=e=\frac{3\eta L}{2}$, and $c=d=0$:
\begin{align*}
\frac{1}{T+1} \sum_{t=0}^{T} \mathbb{E} \left( f(\bar{\bm{x}}^{(t)}) - f^\star \right) 
&\leq \frac{\left\| \bm{x}^{(0)} - \bm{x}^\star \right\|^2}{\eta (T+1)} + \frac{\eta \bar{\sigma}^2}{n} + \frac{3L}{2} \left( \bar{\Xi}^{(T+1)}_{\text{intra}} + \bar{\Xi}^{(T+1)}_{\text{inter}} \right).
\end{align*}

\emph{For the intra-component disagreement}, apply Lemmas~\ref{consensus_intra_inter} and~\ref{lem:sampling:consensus_intra_inter}:
\begin{align*}
    \Xi^{(t)}_{\text{intra}}
    \leq \begin{dcases}
        \left( 1 - \frac{p}{4} \right) \Xi^{(t-1)}_{\text{intra}}
    + \frac{6 \eta^{2} \bar{\zeta}_{\text{intra}}^2}{p}, & t \notin \mathcal{H}; \\
    \frac{n-K}{n-1} \left( 1 - \frac{p}{4} \right) \Xi^{(t-1)}_{\text{intra}}
    + \frac{n-K}{n-1} \frac{6 \eta^{2} \bar{\zeta}_{\text{intra}}^2}{p}, & t \in \mathcal{H}.
    \end{dcases}
\end{align*}

For the recursion, apply Lemma~\ref{prop:recursion} with $a_2 = 1 - \frac{p}{4} < 1$, $b_2 = \frac{6 \eta^{2} \bar{\zeta}_{\text{intra}}^2}{p}$, $a_1 = (1 - \frac{K-1}{n-1}) a_2$, $b_1 = \frac{n-K}{n-1} b_2$. \\
For simplicity, define $\gamma \coloneqq \frac{p}{4}$ and $\delta \coloneqq \frac{n-K}{n-1}$, such that $a_2 = 1 - \gamma$, $a_1 = \delta a_2$, and $b_1 = \delta b_2$:
\begin{align*}
        \bar{\Xi}^{(T+1)}_{\text{intra}} 
        &\leq \left[ \frac{a_1 b_2 (1 - a_2^{H-1})}{(1-a_1 a_2^{H-1})(1-a_2)^2} + \frac{b_1}{(1-a_1 a_2^{H-1})(1-a_2)} + \frac{b_2}{1 - a_2} (H - 1) \right] \left( \frac{1}{H} + \frac{1}{T+1} \right) \\
        &= \left[ \frac{\delta a_2 b_2 (1 - a_2^{H-1})}{(1- \delta a_2^{H})(1-a_2)^2} + \frac{\delta b_2}{(1-\delta a_2^{H})(1-a_2)} + \frac{b_2}{1 - a_2} (H - 1) \right] \left( \frac{1}{H} + \frac{1}{T+1} \right) \\
        &= \Biggl[ \underbrace{\frac{\delta (1-\gamma) b_2 [1 - (1-\gamma)^{H-1}]}{[1- \delta (1-\gamma)^{H}]\gamma^2}}_{\coloneqq T_1} + \underbrace{\frac{\delta b_2}{[1-\delta (1-\gamma)^{H}]\gamma}}_{\coloneqq T_2} + \frac{b_2}{\gamma} (H - 1) \Biggr] \left( \frac{1}{H} + \frac{1}{T+1} \right).
    \end{align*}
Using that $1 - (1-\gamma)^{H-1} \leq (H-1)\gamma$ and that $1- \delta (1-\gamma)^{H} \geq 1-\gamma$, we have $T_1 \leq \frac{\delta b_2 (H-1)}{(1-\delta)\gamma}$ and $T_2 = \frac{\delta b_2}{(1-\delta)\gamma}$:
\begin{align*}
    \bar{\Xi}^{(T+1)}_{\text{intra}} 
    &\leq \frac{b_2}{\gamma} \left[ (H - 1) + \frac{\delta H}{(1-\delta)} \right] \left( \frac{1}{H} + \frac{1}{T+1} \right) \\
    &\leq \frac{b_2}{\gamma} \frac{H}{1-\delta} \left( \frac{1}{H} + \frac{1}{T+1} \right)  \\
    &\leq \frac{n-1}{K-1} \frac{24 \eta^{2} \bar{\zeta}_{\text{intra}}^2}{p^2} \left( 1 + \frac{H}{T+1} \right) \\
    &\leq \frac{n-1}{K-1} \frac{48 \eta^{2} \bar{\zeta}_{\text{intra}}^2}{p^2},
\end{align*}
where, in the last inequality, we simplified the bound using that, for $T\geq H-1$, $\frac{H}{T+1}\leq1$.

\emph{For the inter-component disagreement}, combine Lemmas~\ref{consensus_intra_inter} and ~\ref{lem:sampling:consensus_intra_inter}, with $\rho>0$:
\begin{align*}
    \Xi^{(t)}_{\text{inter}}
    \leq \begin{dcases}
        (1+\rho) \Xi^{(t-1)}_{\text{inter}} + (1+\rho^{-1}) \eta^2 \bar{\zeta}_{\text{inter}}^2, & t \notin \mathcal{H}; \\
    \frac{n-K}{n-1} (1+\rho) \Xi^{(t-1)}_{\text{inter}}
    + \frac{n-K}{n-1} (1+\rho^{-1}) \eta^2 \bar{\zeta}_{\text{inter}}^2, & t \in \mathcal{H}.
    \end{dcases}
\end{align*}

For the recursion, apply Lemma~\ref{prop:recursion} with $a_2 = (1+\rho) > 1$, $b_2 = (1+\rho^{-1}) \eta^2 \bar{\zeta}_{\text{inter}}^2$, $a_1 = \delta a_2$, $b_1 = \delta b_2$:
\begin{align*}
    \bar{\Xi}^{(T+1)}_{\text{inter}} 
    &\leq \left[ \frac{a_1 b_2 (a_2^{H-1} - 1)(a_2^H - 1)}{(1-a_1 a_2^{H-1})(a_2 - 1)^2} + \frac{b_1(a_2^H - 1)}{(1-a_1 a_2^{H-1})(a_2 - 1)} + \frac{b_2(a_2^H - a_2H + H -1)}{(a_2 - 1)^2} \right] \left( \frac{1}{H} + \frac{1}{T+1} \right) \\
    &\leq \left[ \frac{\delta a_2 b_2 (a_2^{H-1} - 1)(a_2^H - 1)}{(1-\delta a_2^{H})(a_2 - 1)^2} + \frac{\delta b_2(a_2^H - 1)}{(1- \delta a_2^{H})(a_2 - 1)} + \frac{b_2(a_2^H - a_2H + H -1)}{(a_2 - 1)^2} \right] \left( \frac{1}{H} + \frac{1}{T+1} \right) \\
    &= \Biggl[ \underbrace{\frac{\delta (1+\rho) (1+\rho^{-1}) [(1+\rho)^{H-1} - 1][(1+\rho)^H - 1]}{[1-\delta (1+\rho)^{H}]\rho^2}}_{\coloneqq T_3} + \underbrace{\frac{\delta (1+\rho^{-1})[(1+\rho)^H - 1]}{[1-\delta (1+\rho)^{H}]\rho}}_{\coloneqq T_4}\\
    &\quad + \underbrace{\frac{(1+\rho^{-1})[(1+\rho)^H - (1+\rho)H + H -1]}{\rho^2}}_{\coloneqq T_5} \Biggr] \eta^2 \bar{\zeta}_{\text{inter}}^2 \left( \frac{1}{H} + \frac{1}{T+1} \right).
\end{align*}

We choose $\rho = \frac{1-\delta}{2H}$, such that $C = a_1 a_2^{H-1} = \frac{n-K}{n-1} (1+\rho)^{H} < 1$ in Lemma~\ref{prop:recursion}. \\
As a consequence, we have: $T_3 \leq \frac{54\delta H^2 (H-1)}{(1-\delta)^2}$, $T_4 \leq \frac{12 \delta H^2}{(1-\delta)^2}$, and $T_5 \leq \frac{4 H^2 (H-1)}{1-\delta}$:
\begin{align*}
    \bar{\Xi}^{(T+1)}_{\text{inter}} 
    &\leq 70 \eta^2 \bar{\zeta}_{\text{inter}}^2 
    \frac{\delta H^2 (H-1)}{(1-\delta)^2}  \left( \frac{1}{H} + \frac{1}{T+1} \right) \\
    &\leq
    \left( \frac{n-1}{K-1} \right)^2 70 \eta^2 \bar{\zeta}_{\text{inter}}^2 \left( H(H-1) + \frac{H^2(H-1)}{T+1} \right) \\
    &\leq
    \left( \frac{n-1}{K-1} \right)^2 140 \eta^2 \bar{\zeta}_{\text{inter}}^2 H(H-1),
\end{align*}
where, in the last inequality, we again simplified the bound using that, for $T\geq H-1$, $\frac{H}{T+1}\leq1$.

Replace $\bar{\Xi}^{(T+1)}_{\text{intra}}$ and $\bar{\Xi}^{(T+1)}_{\text{inter}}$ in the bound:
\begin{align*} 
&\frac{1}{T+1} \sum_{t=0}^{T} \mathbb{E} \left( f(\bar{\bm{x}}^{(t)}) - f^\star \right) 
\leq 
\frac{|| \bm{x}^{(0)} - \bm{x}^* ||^2_2}{\eta (T+1)} 
+ \frac{\eta \bar{\sigma}^2}{n} 
+ \frac{n-1}{K-1} \frac{72 \eta^2 L \bar{\zeta}_{\text{intra}}^2 }{p^2} 
+ \left( \frac{n-1}{K-1} \right)^2 210 \eta^2 L H(H-1) \bar{\zeta}_{\text{inter}}^2.
\end{align*}

Finally, apply~\citep[Lemmas~16 and 17]{koloskovaUnifiedTheoryDecentralized2020} with
$r_t=\mathbb{E}|| \bar{\bm{x}}^{(t)} - \bm{x}^* ||^2_2$,
$e_t=\mathbb{E} (f(\bar{\bm{x}}^{(t)}) - f^\star)$, \\
$a=0$,
$b=1$,
$c=\frac{\bar{\sigma}^2}{n}$,
$d=\frac{8L}{p}$,
and 
$e= \frac{n-1}{K-1} \frac{72 L \bar{\zeta}_{\text{intra}}^2 }{p^2} 
+ \left( \frac{n-1}{K-1} \right)^2 210 L H(H-1) \bar{\zeta}_{\text{inter}}^2$.
\end{proof}

\begin{proof}[Proof of Theorem~2 (Non-Convex Objectives)]
Combine Lemma~\ref{lem:descent_nc} and Lemma~\ref{prop:sampling:averaging} ($\| \bar{\bm{x}}^{(t+1)} - \bar{\bm{x}}^{(t+\frac{2}{3})} \|_2^2 = 0$). 

For every $t \geq 0$:
\begin{align*}
    \mathbb{E} [f(\bar{\bm{x}}^{(t+1)})]
    \leq
    \mathbb{E} [f(\bar{\bm{x}}^{(t)})] - \frac{\eta}{4} \mathbb{E}\left\| \nabla f(\bar{\bm{x}}^{(t)}) \right\|_2^2 + \frac{\eta^2 L \bar{\sigma}^2}{2n} + \eta L^2 \Xi^{(t)}.
\end{align*}

Apply Lemma~\ref{lem:convergence_recursion} with $r^{(t)} = \mathbb{E} [f(\bar{\bm{x}}^{(t)})]$, $\Delta^{(t)} = \frac{1}{4} \mathbb{E}\|\nabla f(\bar{\bm{x}}^{(t)})\|_2^2$, $a=b=e=\eta L^2$, and $c=d=0$:
\begin{align*}
    \frac{1}{T+1} \sum_{t=0}^T \mathbb{E}\left\| \nabla f(\bar{\bm{x}}^{(t)}) \right\|_2^2
    &\leq \frac{4 (f(\bar{\bm{x}}^{(0)}) - f^\star)}{\eta (T+1)} + \frac{2\eta L \bar{\sigma}^2}{n} + 4L^2 \left( \bar{\Xi}_{\text{intra}}^{(T+1)} + \bar{\Xi}_{\text{inter}}^{(T+1)} \right).
\end{align*}

Replace the values $\bar{\Xi}^{(T+1)}_{\text{intra}} 
    \leq \frac{n-1}{K-1} \frac{48 \eta^{2} \bar{\zeta}_{\text{intra}}^2}{p^2}$,
    $\bar{\Xi}^{(T+1)}_{\text{intra}} 
    \leq \big( \frac{n-1}{K-1} \big)^2 140 \eta^2 H(H-1) \bar{\zeta}_{\text{inter}}^2$ found previously:
\begin{align*}
    \frac{1}{T+1} \sum_{t=0}^T \mathbb{E}\left\| \nabla f(\bar{\bm{x}}^{(t)}) \right\|_2^2
    &\leq
    \frac{4 (f(\bar{\bm{x}}^{(0)}) - f^\star)}{\eta (T+1)} + \frac{2\eta L \bar{\sigma}^2}{n} + \frac{n-1}{K-1} \frac{192 \eta^2 L^2 \bar{\zeta}_{\text{intra}}^2 }{p^2} 
+ \left( \frac{n-1}{K-1} \right)^2 560 \eta^2 L^2 H(H-1) \bar{\zeta}_{\text{inter}}^2.
\end{align*}
Apply~\citep[Lemmas~16 and 17]{koloskovaUnifiedTheoryDecentralized2020} with 
$r_t=\mathbb{E}[f(\bar{\bm{x}}^{(t)})]$,
$e_t=\mathbb{E} \| \nabla f(\bar{\bm{x}}^{(t)})\|_2^2$, \\
$a=0$, 
$b=1$, 
$c=\frac{2 L \bar{\sigma}^2}{n}$, $d=\frac{8L}{p}$, 
and 
$e=\frac{n-1}{K-1} \frac{192 L^2 \bar{\zeta}_{\text{intra}}^2}{p^2} 
+ \left( \frac{n-1}{K-1} \right)^2 560 L^2 H(H-1) \bar{\zeta}_{\text{inter}}^2$.
\end{proof}

\section{Convergence Analysis of S2A}
\label{app:sec:convergenceS2A}

\subsection{Properties of S2A}
\label{app:subsec:S2A}

\begin{lemma}[Sampled-to-All] 
\label{lem:broadcast_app}
Define:
\begin{align*}
    (W_{\text{S2A}}^{(t)})_{ij} = \begin{dcases}
    \frac{1}{K}, & i \in \mathcal{S}^{(t)}; \\
    0, & \text{otherwise}.
    \end{dcases}
\end{align*}
The matrix $W_{\text{S2A}}^{(t)}$ satisfies the following properties:
\begin{enumerate}[label=(\roman*),leftmargin=20pt]
    \item \textbf{Column-stochastic, not row-stochastic.} $\bm{1}^\top W_{\text{S2A}}^{(t)} = \bm{1}^\top$, whereas $W_{\text{S2A}}^{(t)} \bm{1} = \frac{K}{n} \bm{1}_{\{ i \in \mathcal{S} \}}$.
    
    Consequently, $W_{\text{S2A}}^{(t)} \Pi = W_{\text{S2A}}^{(t)}$, and $\Pi W_{\text{S2A}}^{(t)} = \Pi$.
    
    \item \textbf{Bias error.} $W_{\text{S2A}}^{(t)}$ does not preserve the average of the iterates between D2D and D2S rounds:
    \begin{align*}
        \bar{X}^{(t+1)} = X^{(t+\frac{2}{3})} W_{\text{S2A}}^{(t)} \Pi = X^{(t+\frac{2}{3})} W_{\text{S2A}}^{(t)}
        \neq 
        X^{(t+\frac{2}{3})} \Pi \eqqcolon \bar{X}^{(t+\frac{2}{3})}.
    \end{align*}
    We define the bias error as: 
    \begin{align*}
        \mathbb{E} \left\| \bar{X}^{(t+1)} -  \bar{X}^{(t+\frac{2}{3})} \right\|_F^2 = \mathbb{E} \left\| X^{(t+\frac{2}{3})} (W_{\text{S2A}}^{(t)} - \Pi) \right\|_F^2 \geq 0.
    \end{align*}

    For $t \in \mathcal{H}$, the bias error satisfies:
    \begin{align*}
        \mathbb{E} \left[ \left\| \bar{X}^{(t+1)} - \bar{X}^{(t+\frac{2}{3})} \right\|_F^2 \right] 
        &=
        \mathbb{E} \left[ \left\| X^{(t+\frac{2}{3})} (W_{\text{S2A}}^{(t)} - \Pi) \right\|_F^2 \right] \\
        &=
        \frac{n-K}{K(n-1)} \mathbb{E} \left[ \left\| X^{(t+\frac{2}{3})} (I - \Pi) \right\|_F^2 \right] \\
        &=
        \frac{n-K}{K(n-1)} \mathbb{E} \left[ \left\| X^{(t+\frac{2}{3})} - \bar{X}^{(t+\frac{2}{3})} \right\|_F^2 \right].
    \end{align*}
    \begin{itemize}
        \item For $K=1$, the factor $\frac{n-K}{K(n-1)}$ equals $1$, indicating no contraction.
        \item For $K=n$, the factor $\frac{n-K}{K(n-1)}$ equals $0$, corresponding to full contraction.
    \end{itemize}

    \item \textbf{The global average is unbiased in expectation.} Since $\mathbb{E}_{\mathcal{S}^{(t)}}[W_{\text{S2A}}^{(t)}] = \Pi$,
    \begin{align*}
        \mathbb{E}_{\mathcal{S}^{(t)}}[\bar{X}^{(t+1)}] =  \mathbb{E}_{\mathcal{S}^{(t)}}[X^{(t+\frac{2}{3})} W_{\text{S2A}}^{(t)} \Pi] = X^{(t+\frac{2}{3})} \mathbb{E}_{\mathcal{S}^{(t)}}[W_{\text{S2A}}^{(t)}] \Pi = X^{(t+\frac{2}{3})} \Pi^2 = X^{(t+\frac{2}{3})} \Pi \eqqcolon \bar{X}^{(t+\frac{2}{3})}.
    \end{align*}

    In other words, $W_{\text{S2A}}$ propagates the average of the iterates in expectation.

    \item \textbf{The disagreement error is zero.} For $t \in \mathcal{H}$,
    \begin{align*}
        X^{(t+1)} - \bar{X}^{(t+1)} = X^{(t+\nicefrac{2}{3})} W_{\text{S2A}}^{(t)} - X^{(t+\nicefrac{2}{3})} W_{\text{S2A}}^{(t)} \Pi = X^{(t+\nicefrac{2}{3})} W_{\text{S2A}}^{(t)} - X^{(t+\nicefrac{2}{3})} W_{\text{S2A}}^{(t)} = 0.
    \end{align*}

\end{enumerate}
\end{lemma}
\begin{proof}[Proof of Lemma~\ref{lem:broadcast_app} (ii)] \hphantom{Ciao}\\
This result can be derived from the variance of sampling without replacement~\cite[Lemma 4]{jhunjhunwalaFedVARPTacklingVariance2022}:
\begin{align*}
    &\mathbb{E}_{\mathcal{S}^{(t)}} \left\| \bar{\bm{x}}^{(t+1)} - \bar{\bm{x}}^{(t+\frac{2}{3})} \right\|_2^2 \\
    &= \mathbb{E}_{\mathcal{S}^{(t)}} \Biggl\| \frac{1}{K} \sum_{i \in \mathcal{S}^{(t)}} \left( \bm{x}_i^{(t+\frac{2}{3})} - \bar{\bm{x}}^{(t+\frac{2}{3})} \right) \Biggr\|_2^2 \\
    &= \mathbb{E}_{\mathcal{S}^{(t)}} \left\| \frac{1}{K} \sum_{i=1}^{n} \mathds{1}_{i \in \mathcal{S}^{(t)}} \left(\bm{x}_i^{(t+\frac{2}{3})} - \bar{\bm{x}}^{(t+\frac{2}{3})}\right) \right\|_2^2 \\
    &= \frac{1}{K^2} \mathbb{E}_{\mathcal{S}^{(t)}} \Biggl[ \sum_{i=1}^{n} \left[ \mathds{1}_{i \in \mathcal{S}^{(t)}} \right]^2 \left\|\bm{x}_i^{(t+\frac{2}{3})} - \bar{\bm{x}}^{(t+\frac{2}{3})}\right\|_2^2 
    + \sum_{j=1}^{n} \sum_{\substack{i=1\\i \neq j}}^{n} \mathds{1}_{i \in \mathcal{S}^{(t)}} \mathds{1}_{j \in \mathcal{S}^{(t)}} \left\langle \bm{x}_i^{(t+\frac{2}{3})} - \bar{\bm{x}}^{(t+\frac{2}{3})}, \bm{x}_j^{(t+\frac{2}{3})} - \bar{\bm{x}}^{(t+\frac{2}{3})} \right\rangle \Biggr] \\
    &= \frac{1}{K^2} \sum_{i=1}^{n} \frac{K}{n} \left\|\bm{x}_i^{(t+\frac{2}{3})} - \bar{\bm{x}}^{(t+\frac{2}{3})}\right\|_2^2
    + \frac{1}{K^2} \sum_{i \neq j} \frac{K}{n} \frac{K-1}{n-1} \left\langle \bm{x}_i^{(t+\frac{2}{3})} - \bar{\bm{x}}^{(t+\frac{2}{3})}, \bm{x}_j^{(t+\frac{2}{3})} - \bar{\bm{x}}^{(t+\frac{2}{3})} \right\rangle \\
    &= \frac{1}{K^2} \sum_{i=1}^{n} \left\|\bm{x}_i^{(t+\frac{2}{3})} - \bar{\bm{x}}^{(t+\frac{2}{3})}\right\|_2^2 \left[ \frac{K}{n} - \frac{K (K-1)}{n (n-1)} \right] 
    + \frac{1}{K^2} \frac{K (K-1)}{n (n-1)} \underbrace{\Bigl\| \sum_{i=1}^{n} \left( \bm{x}_i^{(t+\frac{2}{3})} - \bar{\bm{x}}^{(t+\frac{2}{3})} \right) \Bigr\|_2^2}_{=0} \\
    &= \frac{(n-K)}{K(n-1)} \frac{1}{n} \sum_{i=1}^{n} \left\|\bm{x}_i^{(t+\frac{2}{3})} - \bar{\bm{x}}^{(t+\frac{2}{3})}\right\|_2^2 
    = \frac{(n-K)}{K(n-1)} \frac{1}{n} \left\| X^{(t)} - \bar{X}^{(t)} \right\|_F^2.
\end{align*}
\end{proof}

\subsection{Intermediate Lemmas}

\begin{lemma}[S2A: Bias Error]
\label{lem:sampling_error}
For every \( t \ge 0 \),
\begin{align*}
    &
    \mathbb{E} \left\| \bar{\bm{x}}^{(t+1)} - \bar{\bm{x}}^{(t+\frac{2}{3})} \right\|_2^2 
    = 
    \begin{dcases}
        \frac{n-K}{K(n-1)} \frac{1}{n} \mathbb{E} \left\| X^{(t+\frac{2}{3})} - \bar{X}^{(t+\frac{2}{3})} \right\|_F^2, & t \in \mathcal{H}, \\
        0, & t \notin \mathcal{H},
    \end{dcases}
\end{align*}
where $\mathbb{E} \| X^{(t+\frac{2}{3})} - \bar{X}^{(t+\frac{2}{3})} \|_F^2$ is the D2D disagreement error already bounded in Lemma~\ref{consensus_intra_inter}.
\end{lemma}

\begin{proof}[Proof of Lemma~\ref{lem:sampling_error}] We have:
\begin{itemize}
    \item For D2D rounds ($t \notin \mathcal{H}$), by definition, $\bar{\bm{x}}^{(t+1)} = \bar{\bm{x}}^{(t+\frac{2}{3})}$.
    \item For D2S rounds ($t \in \mathcal{H}$), by Lemma~\ref{lem:broadcast_app} (\emph{ii}):
    \begin{align*}
        \mathbb{E} \left\| \bar{\bm{x}}^{(t+1)} - \bar{\bm{x}}^{(t+\frac{2}{3})} \right\|_2^2 = \frac{1}{n} \mathbb{E} \left\| \bar{X}^{(t+1)} - \bar{X}^{(t+\frac{2}{3})} \right\|_F^2 
        =
        \frac{n-K}{K(n-1)} \frac{1}{n} \mathbb{E} \left\| X^{(t+\frac{2}{3})} - \bar{X}^{(t+\frac{2}{3})} \right\|_F^2.
    \end{align*}
\end{itemize}
\end{proof}

\begin{lemma}[S2A: Disagreement Error] 
\label{consensus_sta}
For every $t \geq 0$,
\begin{align*}
    \left\| X^{(t+1)} - \bar{X}^{(t+1)} \right\|_F^2 
    =
    \begin{dcases}
        \left\| X^{(t+\frac{2}{3})} - \bar{X}^{(t+\frac{2}{3})} \right\|_F^2, & t \not \in \mathcal{H}; \\
        0, & t \in \mathcal{H},
    \end{dcases}
\end{align*}
where $\mathbb{E} \| X^{(t+\frac{2}{3})} - \bar{X}^{(t+\frac{2}{3})} \|_F^2$ is the D2D disagreement error already bounded in Lemma~\ref{consensus_intra_inter}.
\end{lemma}
\begin{proof}[Proof of Lemma~\ref{consensus_sta}] We have:
\begin{itemize}
    \item For D2D rounds ($t \notin \mathcal{H}$), by definition, $X^{(t+1)} = X^{(t+\frac{2}{3})}$ and $\bar{X}^{(t+1)} = \bar{X}^{(t+\frac{2}{3})}$.
    
    \item For D2S rounds ($t \in \mathcal{H}$), by Lemma~\ref{lem:broadcast_app} (\emph{iv}), $X^{(t+1)} = \bar{X}^{(t+1)}$.
\end{itemize}
\end{proof}

\subsection{Proof of Theorem~2}
\label{app:subsec:theorem_S2A}

\begin{proof}[Proof of Theorem~2 (Convex Objectives)]
Combine Lemmas \ref{lem:descent}, \ref{consensus_intra_inter}, \ref{lem:sampling_error}, and~\ref{consensus_sta}:
\begin{itemize}
    \item If $t \in \mathcal{H}$:
    \begin{align*}
        \mathbb{E} \left\| \bar{\bm{x}}^{(t+1)} - \bm{x}^* \right\|_2^2 
        &\leq 
        \mathbb{E} \left\| \bar{\bm{x}}^{(t)} - \bm{x}^* \right\|_2^2 
        - \eta \mathbb{E} \left( f(\bar{\bm{x}}^{(t)}) - f^* \right) 
        + \frac{\eta^{2} \bar{\sigma}^2}{n} \\
        &\quad 
        + \left[ \frac{3\eta L}{2}  + \frac{(n-K)}{ K(n-1)} \left( 1 - \frac{p}{4} \right) \right] \Xi_{\text{intra}}^{(t)} \\
        &\quad 
        + \left[ \frac{3\eta L}{2} + \frac{(n-K)}{ K(n-1)} (1+\rho) \right] \Xi_{\text{inter}}^{(t)} \\
        &\quad 
        + \frac{(n-K)}{ K(n-1)} \frac{6 \eta^{2} \bar{\zeta}_{\text{intra}}^2}{p}
        + \frac{(n-K)}{ K(n-1)} (1+\rho^{-1}) \eta^2 \bar{\zeta}_{\text{inter}}^2.
    \end{align*}
    \item If $t \not \in \mathcal{H}$:
    \begin{align*}
        \mathbb{E} \left\| \bar{\bm{x}}^{(t+1)} - \bm{x}^* \right\|_2^2 
        &\leq 
        \mathbb{E} \left\| \bar{\bm{x}}^{(t)} - \bm{x}^* \right\|_2^2 - \eta \mathbb{E} \left( f(\bar{\bm{x}}^{(t)}) - f^* \right) 
        + \frac{\eta^{2} \bar{\sigma}^2}{n} \\
        &\quad
        + \frac{3 \eta L}{2}  \Xi_{\text{intra}}^{(t)}
        + \frac{3 \eta L}{2} \eta \Xi_{\text{inter}}^{(t)}.
    \end{align*}
\end{itemize}

Apply Lemma~\ref{lem:convergence_recursion} with 
$r^{(t)} \coloneqq \mathbb{E} \| \bar{\bm{x}}^{(t)} - \bm{x}^* \|_2^2$, 
$\Delta^{(t)} \coloneqq \mathbb{E} (f(\bar{\bm{x}}^{(t)}) - f^*)$, 
$a \coloneqq \frac{3\eta L}{2} + \frac{(n-K)}{ K(n-1)} \left( 1 - \frac{p}{4} \right)$, \\
$b \coloneqq \frac{3\eta L}{2}  + \frac{(n-K)}{ K(n-1)} (1+\rho)$, 
$c \coloneqq \frac{n - K}{K (n - 1)} \frac{6 \eta^2 \bar{\zeta}_{\text{intra}}^2}{p}$, 
$d \coloneqq \frac{n - K}{K (n - 1)} (1 + \rho^{-1}) \eta^2 \bar{\zeta}_{\text{inter}}^2$, and 
$e \coloneqq \frac{3\eta L}{2}$:
\begin{align*}
    \frac{1}{T+1} \sum_{t=0}^{T} \mathbb{E} \left( f(\bar{\bm{x}}^{(t)}) - f^\star \right) 
    &\leq 
    \frac{\| \bm{x}^{(0)} - \bm{x}^\star \|^2}{\eta (T+1)} + \frac{\eta \bar{\sigma}^2}{n} +  \frac{2 \eta L + \frac{n-K}{K(n-1)}\frac{1 - \frac{p}{4}}{H}}{\eta} \bar{\Xi}_{\text{intra}}^{(T+1)} +  \frac{2 \eta L + \frac{n-K}{K(n-1)} \frac{1+\rho}{H}}{\eta} \bar{\Xi}_{\text{inter}}^{(T+1)} \\
    &\quad 
    + \frac{n - K}{K (n - 1)} \frac{6 \eta \bar{\zeta}_{\text{intra}}^2}{H p} + \frac{n - K}{K (n - 1)} \frac{1 + \rho^{-1}}{H} \eta \bar{\zeta}_{\text{inter}}^2.
\end{align*}

\emph{For the intra-component disagreement}, combine Lemmas~\ref{consensus_intra_inter} and~\ref{consensus_sta}:
\begin{align*}
    \Xi_{\text{intra}}^{(t)} \leq 
    \begin{dcases}
    0, & t \in \mathcal{H}, \\
    \left( 1 - \frac{p}{4} \right) \Xi_{\text{intra}}^{(t-1)} + \frac{6 \eta^2 \bar{\zeta}_{\text{intra}}^2}{p}, & t \notin \mathcal{H}.
    \end{dcases}
\end{align*}
For the recursion, apply Lemma~\ref{prop:recursion} with $a_1 = b_1 = 0$, $a_2 = 1 - \frac{p}{4} < 1$ and $b_2 = \frac{6 \eta^2 \bar{\zeta}_{\text{intra}}^2}{p}$:
\begin{align*}
    \bar{\Xi}_{\text{intra}}^{(T+1)}
    \leq
    \frac{24 \eta^2 \bar{\zeta}_{\text{intra}}^2 }{p^2} \left( \frac{H-1}{H} + \frac{H-1}{T+1} \right)
    \leq \frac{48 \eta^2 \bar{\zeta}_{\text{intra}}^2 }{p^2},
\end{align*}
where, to simplify the bound, we used that for $T\geq H-1$, $\frac{1}{T+1}\leq\frac{1}{H}$, and that $\frac{H-1}{H} < 1$.

\emph{For the inter-component disagreement}, combine Lemmas~\ref{consensus_intra_inter} and~\ref{consensus_sta}, with $\rho>0$:
\begin{align*}
    \Xi_{\text{inter}}^{(t)} \leq 
    \begin{dcases}
    0, & t \in \mathcal{H}, \\
    (1 + \rho) \Xi_{\text{inter}}^{(t-1)} + (1 + \rho^{-1}) \eta^2 \bar{\zeta}_{\text{inter}}^2, & t \notin \mathcal{H}.
    \end{dcases}
\end{align*}
For the recursion, apply Lemma~\ref{prop:recursion} with $a_1 = b_1 = 0$, $a_2 = 1 + \rho > 1$ and $b_2 = (1 + \rho^{-1}) \eta^2 \bar{\zeta}_{\text{inter}}^2$:
\begin{align*}
    \bar{\Xi}_{\text{inter}}^{(T+1)}
    &\leq   
    \eta^2 \bar{\zeta}_{\text{inter}}^2 (1 + \rho^{-1}) \frac{(1 + \rho)^H - (1 + \rho) H + H - 1}{\rho^2} 
    \left(\frac{1}{H} + \frac{1}{T+1} \right) \\
    &\leq 4 \eta^2 \bar{\zeta}_{\text{inter}}^2 \left( H(H-1) + \frac{H^2(H-1)}{T+1} \right) \\
    &\leq 8 \eta^2 \bar{\zeta}_{\text{inter}}^2 H(H-1).
\end{align*}
In the second inequality, we chose $\rho = \frac{2}{H}$, such that $F_H(\rho) \coloneqq \frac{(1 + \rho)^H - (1 + \rho) H + H - 1}{\rho^2} = \sum_{k=2}^{H} \binom{H}{k} \rho^{k-2} \leq 2 H (H-1)$, \\ $(1 + \rho^{-1}) = 1 + \frac{H}{2} \leq 2H$, $(1 + \rho) = 1 + \frac{2}{H} \leq 3$, $(1 + \rho^{-1}) F_H(\rho) \leq 4H^2(H-1)$, and $(1 + \rho)(1 + \rho^{-1}) F_H(\rho) \leq 12H^2(H-1)$.

In the last inequality, to simplify the bound, we used again that for $T\geq H-1$, $\frac{1}{T+1}\leq\frac{1}{H}$.

Replace $\bar{\Xi}_{\text{intra}}^{(T+1)}$ and $\bar{\Xi}_{\text{inter}}^{(T+1)}$ in the bound:
\begin{align*}
    \frac{1}{T+1} \sum_{t=0}^{T} \mathbb{E} \left( f(\bar{\bm{x}}^{(t)}) - f^\star \right) 
    &\leq 
    \frac{\| \bm{x}^{(0)} - \bm{x}^\star \|^2}{\eta (T+1)} + \frac{\eta \bar{\sigma}^2}{n} \\
    &\quad
    + \frac{96 \eta^2 L \bar{\zeta}_{\text{intra}}^2 }{p^2} 
    + \frac{n-K}{K(n-1)}\left( 1 - \frac{p}{4} \right) \frac{48 \eta \bar{\zeta}_{\text{intra}}^2 }{p^2} \left( \frac{H-1}{H^2} \right) 
    + \frac{n - K}{K (n - 1)} \frac{6 \eta \bar{\zeta}_{\text{intra}}^2}{H p}
    \\
    &\quad
    + 16 \eta^2 L \bar{\zeta}_{\text{inter}}^2 H(H-1) 
    + \frac{n-K}{K(n-1)} 12 \eta \bar{\zeta}_{\text{inter}}^2 \left(H-1 + \frac{H-1}{H} \right) 
    + \frac{n - K}{K (n - 1)} 2 \eta \bar{\zeta}_{\text{inter}}^2.
\end{align*}

To simplify the bound, use that $p \leq 1$, therefore $(1 - \frac{p}{4})<1$, and that $\frac{H-1}{H} \leq 1$:
\begin{align*} 
\frac{1}{T+1} \sum_{t=0}^{T} \mathbb{E} \left( f(\bar{\bm{x}}^{(t)}) - f^\star \right) 
&\leq \frac{\| \bm{x}^{(0)} - \bm{x}^\star \|^2}{\eta (T+1)} 
+ \frac{\eta \bar{\sigma}^2}{n}\\
&\quad + \frac{96 \eta^2 L \bar{\zeta}_{\text{intra}}^2}{p^2}
+ \frac{n-K}{K(n-1)} \frac{54 \eta \bar{\zeta}_{\text{intra}}^2}{H p^2} 
+ 16 \eta^2 L H^2 \bar{\zeta}_{\text{inter}}^2 
+ \frac{n-K}{K(n-1)} 26 \eta H \bar{\zeta}_{\text{inter}}^2.
\end{align*}

Finally, apply~\citep[Lemmas~16 and 17]{koloskovaUnifiedTheoryDecentralized2020} with 
$r_t=\mathbb{E} || \bar{\bm{x}}^{(t)} - \bm{x}^* ||^2_2$,
$e_t=\mathbb{E} (f(\bar{\bm{x}}^{(t)}) - f^\star)$, \\
$a=0$,
$b=1$,
$c=\frac{\bar{\sigma}^2}{n} + \frac{n-K}{K(n-1)} 54 \frac{\bar{\zeta}_{\text{intra}}^2}{H p^2} + \frac{n-K}{K(n-1)} 26 H \bar{\zeta}_{\text{inter}}^2$,
$d=\frac{8L}{p}$, and 
$e= \frac{96 L \bar{\zeta}_{\text{intra}}^2}{p^2} + 16 L H^2 \bar{\zeta}_{\text{inter}}^2$.
\end{proof}

\begin{proof}[Proof of Theorem~1 (Non-Convex Objectives)]
Combine Lemmas \ref{lem:descent_nc}, \ref{consensus_intra_inter}, \ref{lem:sampling_error}, and~\ref{consensus_sta}:
\begin{itemize}
    \item If $t \in \mathcal{H}$:
    \begin{align*}
        &\mathbb{E} [f(\bar{\bm{x}}^{(t+1)})] \leq \\
        &\leq
        \mathbb{E} [f(\bar{\bm{x}}^{(t)})] - \frac{\eta}{4} \mathbb{E}\left\| \nabla f(\bar{\bm{x}}^{(t)}) \right\|_2^2 + \frac{\eta^2 L \bar{\sigma}^2}{2n} + \eta L^2 \frac{1}{n} \mathbb{E} \left\| X^{(t)} -  \bar{X}^{(t)} \right\|_F^2 + \frac{L}{2} \mathbb{E} \left\| \bar{\bm{x}}^{(t+1)} - \bar{\bm{x}}^{(t+\frac{2}{3})} \right\|_2^2  \\
        &\leq
        \mathbb{E} [f(\bar{\bm{x}}^{(t)})] - \frac{\eta}{4} \mathbb{E}\left\| \nabla f(\bar{\bm{x}}^{(t)}) \right\|_2^2 + \frac{\eta^2 L \bar{\sigma}^2}{2n} + \eta L^2 \frac{1}{n} \mathbb{E} \left\| X^{(t)} -  \bar{X}^{(t)} \right\|_F^2 + \frac{L}{2} \frac{(n-K)}{K(n-1)} \frac{1}{n} \mathbb{E} \left\|X^{(t+\frac{2}{3})} - \bar{X}^{(t+\frac{2}{3})}\right\|_F^2 \\
        &\leq
        \mathbb{E} [f(\bar{\bm{x}}^{(t)})] - \frac{\eta}{4} \mathbb{E}\left\| \nabla f(\bar{\bm{x}}^{(t)}) \right\|_2^2 + \frac{\eta^2 L \bar{\sigma}^2}{2n} 
        + \frac{L}{2} \left[ 2\eta L  +  \frac{(n-K)}{K(n-1)} \left( 1 - \frac{p}{4} \right) \right] \Xi_{\text{intra}}^{(t)} \\
        & \quad + \frac{L}{2} \left[ 2 \eta L 
        + \frac{(n-K)}{K(n-1)} (1+\rho) \right] \Xi_{\text{inter}}^{(t)} 
        + \frac{(n-K)}{ K(n-1)} \frac{3 \eta^{2} L \bar{\zeta}_{\text{intra}}^2}{p}
        + \frac{(n-K)}{ K(n-1)} (1+\rho^{-1}) \frac{L \eta^2 \bar{\zeta}_{\text{inter}}^2}{2} .
    \end{align*}
    \item If $t \not \in \mathcal{H}$:
    \begin{align*}
        \mathbb{E} [f(\bar{\bm{x}}^{(t+1)})]
        &\leq
        \mathbb{E} [f(\bar{\bm{x}}^{(t)})] - \frac{\eta}{4} \mathbb{E}\left\| \nabla f(\bar{\bm{x}}^{(t)}) \right\|_2^2 + \frac{\eta^2 L \bar{\sigma}^2}{2n} + \eta L^2 \Xi_{\text{intra}}^{(t)} + \eta L^2 \Xi_{\text{inter}}^{(t)}.
    \end{align*}
\end{itemize}

Apply Lemma~\ref{lem:convergence_recursion} with 
$r^{(t)} = \mathbb{E} [f(\bar{\bm{x}}^{(t)})]$, 
$\Delta^{(t)} = \frac{1}{4} \mathbb{E} \| \nabla f(\bar{\bm{x}}^{(t)}) \|_2^2$,  
$a = \frac{L}{2} \left[ 2\eta L  +  \frac{(n-K)}{K(n-1)} \left( 1 - \frac{p}{4} \right) \right]$, \\ 
$b = \frac{L}{2} \left[ 2 \eta L + \frac{(n-K)}{K(n-1)} (1+\rho) \right]$,  
$c = \frac{(n-K)}{ K(n-1)} \frac{3 \eta^{2} L \bar{\zeta}_{\text{intra}}^2}{p}$,  
$d = \frac{(n-K)}{ K(n-1)} (1+\rho^{-1}) \frac{L \eta^2 \bar{\zeta}_{\text{inter}}^2}{2}$, and
$e = \eta L^2$:
\begin{align*}
    \frac{1}{T+1} \sum_{t=0}^{T} \mathbb{E} \left\| \nabla f(\bar{\bm{x}}^{(t)}) \right\|_2^2
    &\leq \frac{4(f(\bar{\bm{x}}^{(0)}) - f^\star)}{\eta (T+1)} + \frac{2 \eta L \bar{\sigma}^2}{n} \\
    &\quad + \frac{2 L}{\eta} \left[ 2\eta L  +  \frac{(n-K)}{K(n-1)} \frac{1 - \frac{p}{4}}{H} \right]  \bar{\Xi}_{\text{intra}}^{(T+1)} + \frac{2 L}{\eta} \left[ 2 \eta L 
    + \frac{(n-K)}{K(n-1)} \frac{1+\rho}{H} \right] \bar{\Xi}_{\text{inter}}^{(T+1)} \\ 
    &\quad + \frac{(n-K)}{ K(n-1)} \frac{12 \eta L \bar{\zeta}_{\text{intra}}^2}{H p} + \frac{(n-K)}{ K(n-1)} \frac{1+\rho^{-1}}{H} 2 \eta L  \bar{\zeta}_{\text{inter}}^2.
\end{align*}

Replace the values $\bar{\Xi}_{\text{intra}}^{(T+1)} \leq \frac{48 \eta^2 \bar{\zeta}_{\text{intra}}^2 }{p^2}$, $\bar{\Xi}_{\text{inter}}^{(T+1)} \leq 8 \eta^2 \bar{\zeta}_{\text{inter}}^2 H(H-1)$, $1 + \rho \leq 3$, $1 + \rho^{-1} \leq 2H$ found previously, \\
and simplify again using $p \leq 1$, therefore $(1 - \frac{p}{4})<1$, and $\frac{H-1}{H} \leq 1$:
\begin{align*}
    \frac{1}{T+1} \sum_{t=0}^{T} \mathbb{E} \left\| \nabla f(\bar{\bm{x}}^{(t)}) \right\|_2^2
    &\leq \frac{4(f(\bar{\bm{x}}^{(0)}) - f^\star)}{\eta (T+1)} + \frac{2 \eta L \bar{\sigma}^2}{n} \\
    &\quad + \frac{192 \eta^2 L^2 \bar{\zeta}_{\text{intra}}^2}{p^2} 
    + \frac{n-K}{K(n-1)} \frac{108 \eta L \bar{\zeta}_{\text{intra}}^2}{H p^2}
    + 32 \eta^2 L^2 H^2 \bar{\zeta}_{\text{inter}}^2
    + \frac{n-K}{K(n-1)} 52 \eta L H \bar{\zeta}_{\text{inter}}^2.
\end{align*}

Finally, apply~\citep[Lemmas~16 and 17]{koloskovaUnifiedTheoryDecentralized2020} with 
$r^{(t)} = \mathbb{E} [f(\bar{\bm{x}}^{(t)})]$, 
$\Delta^{(t)} = \mathbb{E} \| \nabla f(\bar{\bm{x}}^{(t)}) \|_2^2$, \\
$a=0$,
$b=1$,
$c=\frac{2L \bar{\sigma}^2}{n} + \frac{n-K}{K(n-1)} \frac{108 L \bar{\zeta}_{\text{intra}}^2}{H p^2} + \frac{n-K}{K(n-1)} 52 L H \bar{\zeta}_{\text{inter}}^2$,
$d=\frac{8L}{p}$, and 
$e= \frac{192 L^2 \bar{\zeta}_{\text{intra}}^2}{p^2} + 32 L^2 H^2 \bar{\zeta}_{\text{inter}}^2$.
\end{proof}

\newpage
\section{Additional Theoretical Results}
\label{sec:additional_theory}

\subsection{Mixing Parameter for Multi-Component Communication Graphs}
\label{app:sec_mixing}

\begin{lemma}[Mixing parameter for a block-diagonal communication matrix]\label{lem:block-mixing}
Let $W \in \mathbb{R}^{n \times n}$
be a block-diagonal matrix with $C$ blocks, where each block $W_c \in \mathbb{R}^{n_c \times n_c}$ corresponds to the D2D mixing matrix of component $c$.

For each $c$, define the local averaging projector $\Pi_c \coloneqq \frac{1}{n_c} \bm{1} \bm{1}^\top \in \mathbb{R}^{n_c \times n_c}$ and the local mixing parameter $p_c \in (0, 1]$ such that:
\begin{align*}
    \norm{W_c-\Pi_c}_F^2 
    \leq (1-p_c) \norm{I_{n_{c}} -\Pi_c}_F^2.
\end{align*}
For fixed $W_c$, one can take $p_c = 1 - \lambda_2(W_c^\top W_c)$~\cite{boydRandomizedGossipAlgorithms2006}. 

Then, the matrix $W$ satisfies:
\begin{align*}
    \norm{W-\Pi_C}_F^2 
    \leq (1-p) \norm{I_{n} -\Pi_C}_F^2,
\end{align*}
with mixing parameter $ p 
    \geq 
    \frac{\sum_{c=1}^{C} p_c (n_c-1)}{\sum_{c=1}^{C} (n_c - 1)}$.
\end{lemma}

\begin{proof}
Because $W$ and $\Pi_C$ are block-diagonal, the Frobenius norm  decomposes over components:
\begin{align*}
    \norm{W - \Pi_C}_F^2 
    = \sum_{c=1}^{C} \norm{W_c-\Pi_c}_F^2 
    \leq \sum_{c=1}^{C} (1-p_c) \norm{I_{n_c} -\Pi_c}_F^2
    = \sum_{c=1}^{C} (1-p_c) (n_c-1),
\end{align*}
where the rightmost equality uses $\|I_{n_c}-\Pi_c\|_F^2=n_c-1$.

For the same argument,
\begin{align*}
    \norm{I_n -\Pi_C}_F^2 = \sum_{c=1}^{C} \norm{I_{n_c}-\Pi_c}_F^2 = \sum_{c=1}^{C} (n_c - 1).
\end{align*}

Combining:
\begin{align*}
    p 
    \geq 
    1 - \frac{\norm{W - \Pi_C}_F^2}{\norm{I_n -\Pi_C}_F^2}
    \geq
    1 - \frac{\sum_{c=1}^{C} (1-p_c) (n_c-1)}{\sum_{c=1}^{C} (n_c - 1)}
    =
    \frac{\sum_{c=1}^{C} p_c (n_c-1)}{\sum_{c=1}^{C} (n_c - 1)}.
\end{align*}
\end{proof}

We observe that $p \!=\! \frac{\sum_{c=1}^{C} p_c (n_c-1)}{\sum_{c=1}^{C} (n_c - 1)} \!\geq\! p_{\min} \!\coloneqq\! \min\limits_{1\leq c \leq C} p_c \!=\! 1 - \lambda_{C+1}(W^\top W)$, \\where $\lambda_{C+1}(W^{\top}W)$ denotes the largest eigenvalue of~$W^{\top}W$
strictly below~1 .  

\subsection{Extension to Random Mixing Matrices}
\label{app:sec_mixing_random}

As we mentioned in Section~\ref{sec:analysis}, our theoretical results extend to random mixing matrices, following the approach in~\cite{koloskovaUnifiedTheoryDecentralized2020, barsRefinedConvergenceTopology2023}. At each round~$t$ of Algorithm~1, the mixing matrix $W^{(t)}$ is sampled from a distribution $\mathcal{W}^{(t)}$, independent of the iterates $\bm{x}^{(t)}$, and possibly time-varying.

To analyze convergence in this setting, we modify Lemma~\ref{lem:spectral-mixing} and Assumption~\ref{asm:heterogeneity} by taking expectation with respect to~$W$.
Specifically, $\mathbb{E}_{W} \| X (W - \Pi_C) \|_F^2 \leq (1 - p) \| X (I - \Pi_C) \|_F^2$
and
$\frac{1}{n} \sum_{i=1}^{n} \mathbb{E}_{W, \xi} \| \sum_{j=1}^{n} (W - \Pi_C)_{ij} \nabla F_j(\bm{x}, \xi) \|_2^2 \leq \bar{\zeta}_{\text{intra}}^2$.
Theorems~1 and~2 then hold under the assumption that the distributions $\mathcal{W}^{(0)}, \dots, \mathcal{W}^{(T)}$ satisfy these conditions. 

The proof follows exactly the same steps as in the deterministic case, by appropriately conditioning on the realization of the random mixing matrices and on the iterates~\cite{koloskovaUnifiedTheoryDecentralized2020, barsRefinedConvergenceTopology2023}.

\clearpage
\subsection{Iteration vs. Communication Complexity}
\label{sec:comm_complex}

While Section~5.3 compares S2S and S2A in terms of iteration complexity, we now analyze their total communication cost, measured as the number of messages required to reach a target accuracy~$\epsilon \geq 0$.

\paragraph{Per-round communication cost.}  
At each D2S round $t \in \mathcal{H}$, both primitives involve uplink (device-to-server) and downlink (server-to-device) communication. The number of messages exchanged per round is:
\begin{center}
\begin{tabular}{@{}lcc@{}}
\toprule
\textbf{Primitive} & \textbf{uplinks} & \textbf{downlinks}\\
\midrule
\textsc{S2S} & $K$ & $K$\\
\textsc{S2A} & $K$ & $n$\\
\bottomrule
\end{tabular}
\end{center}

In typical federated learning settings, where uplink cost dominates, both primitives incur a per-round communication cost of $K$ messages, making iteration complexity a reasonable proxy for communication cost, and our theoretical comparison of S2S and S2A in Section 5.3 remains valid.
However, when downlink cost is not negligible, S2S incurs a lower communication cost, saving $n-K$ downlink messages per server round, and we could keep this into account for the comparison.

\paragraph{Total communication cost.}
Given the iteration complexities $T_{\text{S2S}}(\epsilon)$ and $T_{\text{S2A}}(\epsilon)$ from Theorems~1 and~2, and defining the number of server rounds as $R \coloneqq  \lceil \frac{T}{H}\rceil$, the total number of messages exchanged by S2S and S2A to reach the accuracy~$\epsilon$ are:
\begin{align*}
\Gamma_{\text{S2S}}(\epsilon) &= 2K R_{\text{S2S}}(\epsilon)
 = \frac{2K}{H} T_{\text{S2S}}(\epsilon), \\
\Gamma_{\text{S2A}}(\epsilon) &= (K+n) R_{\text{S2A}}(\epsilon)
 = \frac{K+n}{H} T_{\text{S2A}}(\epsilon).
\end{align*}

The communication cost ratio is:
\begin{align*}
\frac{\Gamma_{\text{S2A}}(\epsilon)}{\Gamma_{\text{S2S}}(\epsilon)}
=
\frac{K+n}{2K}
\frac{T_{\text{S2A}}(\epsilon)}{T_{\text{S2S}}(\epsilon)}.
\end{align*}

Interestingly, the qualitative regimes follow those in Section~\ref{sec:regimes_theory}:
\begin{enumerate}[label=\emph{\textbf{R\arabic*.}},leftmargin=20pt]
\item \emph{$\bm{\bar{\zeta}}_{\text{\textbf{intra}}},\, \bm{\bar{\zeta}}_{\text{\textbf{inter}}}$ \textbf{are low:}} 
S2A converges faster than S2S for high sampling rates ($K/n\!\geq\!0.1$, Fig.~\ref{fig:bound_app}(a)), most server periods (Fig.~\ref{fig:bound_app}(e)), and higher D2D network connectivities ($p\!\geq\!0.3$, Fig.~\ref{fig:bound_app}(i)).

\item \emph{$\bm{\bar\zeta}_{\text{\textbf{inter}}}\!\bm{\ll}\!\bm{\bar\zeta}_{\text{\textbf{intra}}}\textbf{:}$}
S2S converges faster
for most sampling rates (Fig.~\ref{fig:bound_app}(b)), low server periods ($H\!\leq\!10$, Fig~\ref{fig:bound_app}(f)), and for most mixing parameters (Fig~\ref{fig:bound_app}(j));
S2A converges faster
otherwise.

\item \emph{$\bm{\bar\zeta}_{\text{\textbf{inter}}}$ \textbf{is high:}}
S2S converges faster for most values of $K/n$, $H$, and $p$, irrespective of
$\bar\zeta_{\text{intra}}$ (Figs.~\ref{fig:bound_app}(c,d,g,h,k,l)).
\end{enumerate}

\begin{figure*}[h]
  \centering
  
  \includegraphics[width=0.51\linewidth]{Figs/bound_legend.pdf}
  
  %======================== 1st big panel ========================%
  \fbox{%
  \begin{subfigure}[t]{0.31\textwidth}
    \centering
    
    % ---- first row ----
    \begin{subfigure}[t]{0.48\linewidth}
      \caption{\tiny $\bar{\zeta}_{\text{intra}} {=} \bar{\zeta}_{\text{inter}} {=} \frac{1}{10}$}
      \includegraphics[width=\linewidth]{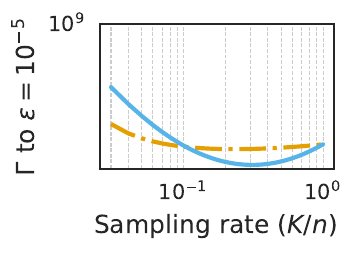}
    \end{subfigure}\hfill
    \begin{subfigure}[t]{0.48\linewidth}
      \caption{\tiny $\bar{\zeta}_{\text{intra}} {=} 1, \bar{\zeta}_{\text{inter}} {=} \frac{1}{10}$}
      \includegraphics[width=\linewidth]{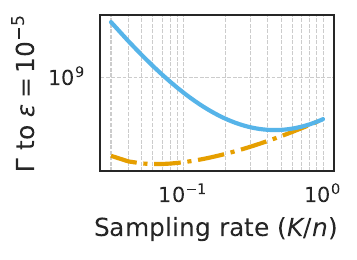}
    \end{subfigure}\\[4pt]
    % ---- second row ----
    \begin{subfigure}[t]{0.48\linewidth}
      \caption{\tiny $\bar{\zeta}_{\text{intra}} {=} \frac{1}{10}, \bar{\zeta}_{\text{inter}} {=} 1$}
      \includegraphics[width=\linewidth]{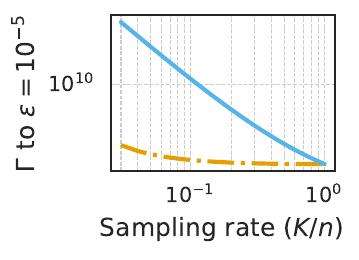}
    \end{subfigure}\hfill
    \begin{subfigure}[t]{0.48\linewidth}
      \caption{\tiny $\bar{\zeta}_{\text{intra}} {=} \bar{\zeta}_{\text{inter}} {=} 1$}
      \includegraphics[width=\linewidth]{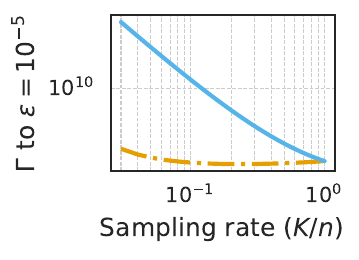}
    \end{subfigure}
  \end{subfigure}
  }
  \hfill
  %======================== 2nd big panel ========================%
  \fbox{%
  \begin{subfigure}[t]{0.31\textwidth}
    \centering
    \begin{subfigure}[t]{0.48\linewidth}
      \caption{\tiny $\bar{\zeta}_{\text{intra}} {=} \bar{\zeta}_{\text{inter}} {=} \frac{1}{10}$}
      \includegraphics[width=\linewidth]{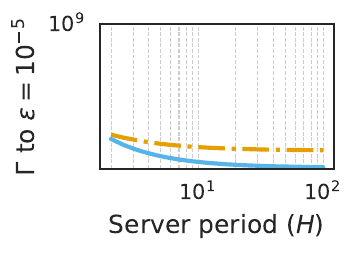}
    \end{subfigure}\hfill
    \begin{subfigure}[t]{0.48\linewidth}
      \caption{\tiny $\bar{\zeta}_{\text{intra}} {=} 1, \bar{\zeta}_{\text{inter}} {=} \frac{1}{10}$}
      \includegraphics[width=\linewidth]{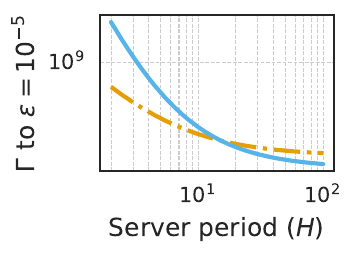}
    \end{subfigure}\\[4pt]
    \begin{subfigure}[t]{0.48\linewidth}
      \caption{\tiny $\bar{\zeta}_{\text{intra}} {=} \frac{1}{10}, \bar{\zeta}_{\text{inter}} {=} 1$}
      \includegraphics[width=\linewidth]{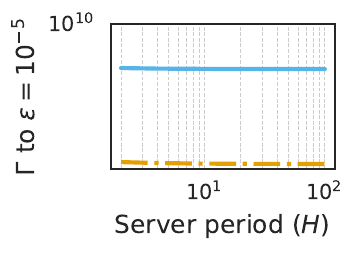}
    \end{subfigure}\hfill
    \begin{subfigure}[t]{0.48\linewidth}
      \caption{\tiny $\bar{\zeta}_{\text{intra}} {=} \bar{\zeta}_{\text{inter}} {=} 1$}
      \includegraphics[width=\linewidth]{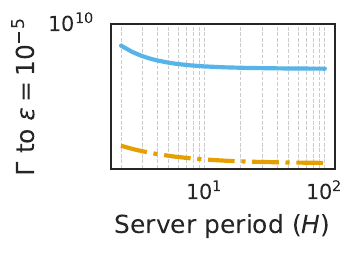}
    \end{subfigure}
  \end{subfigure}
  }
  \hfill
  %======================== 3rd big panel ========================%
  \fbox{%
  \begin{subfigure}[t]{0.31\textwidth}
    \centering
    \begin{subfigure}[t]{0.48\linewidth}
      \caption{\tiny $\bar{\zeta}_{\text{intra}} {=} \bar{\zeta}_{\text{inter}} {=} \frac{1}{10}$}
      \includegraphics[width=\linewidth]{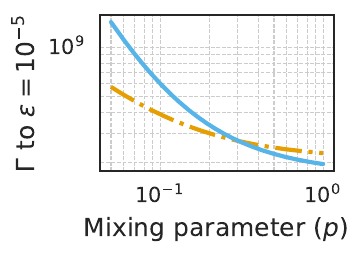}
    \end{subfigure}\hfill
    \begin{subfigure}[t]{0.48\linewidth}
      \caption{\tiny $\bar{\zeta}_{\text{intra}} {=} 1, \bar{\zeta}_{\text{inter}} {=} \frac{1}{10}$}
      \includegraphics[width=\linewidth]{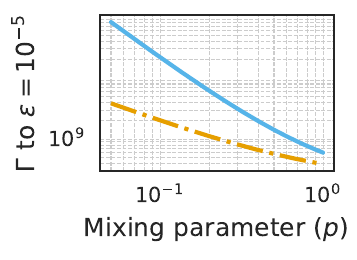}
    \end{subfigure}\\[4pt]
    \begin{subfigure}[t]{0.48\linewidth}
      \caption{\tiny $\bar{\zeta}_{\text{intra}} {=} \frac{1}{10}, \bar{\zeta}_{\text{inter}} {=} 1$}
      \includegraphics[width=\linewidth]{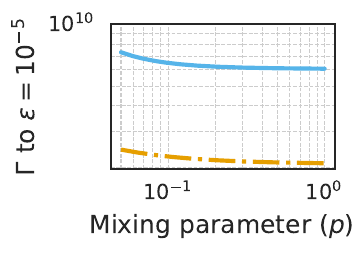}
    \end{subfigure}\hfill
    \begin{subfigure}[t]{0.48\linewidth}
      \caption{\tiny $\bar{\zeta}_{\text{intra}} {=} \bar{\zeta}_{\text{inter}} {=} 1$}
      \includegraphics[width=\linewidth]{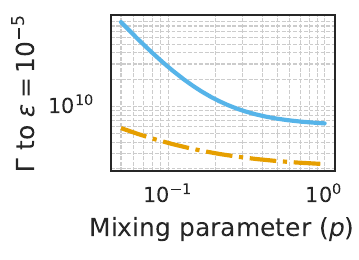}
    \end{subfigure}
  \end{subfigure}
  }

   \caption{Communication costs $\Gamma_{\text{S2S}}(\epsilon)$ and $\Gamma_{\text{S2A}}(\epsilon)$ for $n\!=\!100$, $\epsilon=10^{-5}$, $L\!=\!f_0\!=\!1$, $\bar{\sigma}\!=\!0$. 
           Left panel: Sampling rate ($K/n$) at $H\!=\!5$, $p\!=\!1$. Center panel: Server period ($H$) at $K/n\!=\!0.2$, $p\!=\!1$.
           Right panel: Mixing parameter~($p$) at $K/n\!=\!0.2$, $H\!=\!5$.}
  \label{fig:bound_app}
\end{figure*}

\clearpage
\section{Additional Experimental Results}
\label{app:sec:experiments}

This appendix complements Section~6 by providing further details on the experimental setup, supporting the main results (via Tables~\ref{tab:aggregate_K}--\ref{tab:cifar_H}), and presenting new experiments (Figs.~\ref{fig:app:mnist_heatmap}--\ref{fig:track}), including heatmap summaries on MNIST and CIFAR-10 (Figs.~\ref{fig:app:mnist_heatmap}--\ref{fig:app:cifar_heatmap}), a deeper analysis of outlier behaviors (Figs.~\ref{fig:app:cifar_rounds}--\ref{fig:app:cifar_rounds_H}), comparison of fixed and dynamic topologies (Fig.~\ref{fig:fixed_vs_dynamic}), different server optimizers (Fig.~\ref{fig:fedavg_vs_fedavgm}), additional CIFAR-100 runs (Fig.~\ref{fig:app:cifar100}), and empirical measurements of bias and disagreement errors (Fig.~\ref{fig:track}).

\paragraph{Detailed Experimental Setup.} 
In line with prior work on semi-decentralized federated learning~\cite{linSemiDecentralizedFederatedLearning2021, guoHybridLocalSGD2021, chenTamingSubnetDriftD2DEnabled2024}, we benchmark the S2S and S2A primitives on two standard image classification datasets: MNIST~\cite{lecun-mnisthandwrittendigit-2010} and CIFAR-10~\cite{krizhevsky2009learning}. These datasets offer a controlled and reproducible testbed for evaluating algorithmic performance in federated settings.

For MNIST, we use a single-layer linear classifier with $7,\!850$ trainable parameters. This model is computationally lightweight and supports efficient training for our $9,\!600$ runs on an Intel(R) Xeon(R) E-2224G CPU @ 3.50GHz.

For CIFAR-10, we adopt a standard convolutional neural network consisting of two $5\!\times\!5$ convolutional layers. The first layer maps the $3$-channel RGB input to $32$ channels, followed by ReLU activation and $2\!\times\!2$ max pooling. The second layer outputs $64$ channels and is again followed by ReLU and max pooling. The resulting features are flattened and passed through a fully connected layer with $2048$ ReLU units, followed by a final linear classifier with 10 logits. 

For CIFAR-100, we consider a deeper convolutional neural network. The feature extractor consists of three convolutional blocks with $3\!\times\!3$ convolutions and ELU activations. The first block maps the $3$-channel RGB input to $128$ channels through two convolutional layers, followed by $2\!\times\!2$ max pooling. The second block increases the width to $256$ channels via two further $3\!\times\!3$ convolutions, again followed by max pooling and a dropout layer with rate $0.25$. The third block maps to $512$ channels with two $3\!\times\!3$ convolutions, a final max-pooling layer, and dropout $0.25$, yielding a $512\!\times\!2\!\times\!2$ feature map. The classifier flattens these features and applies a fully connected layer with $1024$ ELU units and dropout $0.5$, followed by a final linear layer producing $100$ logits. 

These model architectures are included in benchmark libraries such as FedML~\cite{heFedMLResearchLibrary2020} and has been widely used in prior semi-decentralized FL work---e.g., for evaluating S2A~\cite{linSemiDecentralizedFederatedLearning2021, guoHybridLocalSGD2021} and S2S~\cite{chenTamingSubnetDriftD2DEnabled2024}. 

Our experiments on CIFAR-10 and CIFAR-100 comprise $9,\!600$ training runs on a Linux server with four NVIDIA GeForce GTX 1080 Ti GPUs.

\paragraph{Detailed Comparison of Figures~\ref{fig:MNIST}--\ref{fig:CIFAR}.} Tables~\ref{tab:aggregate_K}--\ref{tab:cifar_H} complement the visual comparison in Figures~\ref{fig:MNIST}--\ref{fig:CIFAR} by quantifying the test accuracy gap between S2S and S2A across the four heterogeneity regimes and key parameters: sampling rate ($K/n$), server period ($H$), and D2D topology. Tables~\ref{tab:aggregate_K}–\ref{tab:aggregate_H} aggregate the average performance gap over both datasets and all topologies, reporting mean and maximum accuracy gaps, and their statistical significance. Tables~\ref{tab:mnist_K}--\ref{tab:cifar_H} report detailed results by dataset, heterogeneity regime, and key parameter: for MNIST, Tables~\ref{tab:mnist_K} and~\ref{tab:mnist_H} vary sampling rate and server period, respectively; for CIFAR-10, the corresponding results are shown in Tables~\ref{tab:cifar_K} and~\ref{tab:cifar_H}. These tables serve to support the empirical trends discussed in Section~6 and to identify regimes where either primitive yields statistically significant improvements.

\paragraph{Heatmap Summaries.} Figures~\ref{fig:app:mnist_heatmap} and~\ref{fig:app:cifar_heatmap} extend Figure~\ref{fig:cifar_heatmaps} by reporting accuracy gaps on both MNIST and CIFAR-10 datasets for the ring topology, across all four heterogeneity regimes. 
Each heatmap reports the test accuracy difference (S2S minus S2A) after $T = 100$ rounds, with sampling rate ($K/n$) varying across columns and server period ($H$) varying across rows. Positive values favor S2S. 
The new panels in Figures~\ref{fig:app:mnist_heatmap}(a)--(d) and~\ref{fig:app:cifar_heatmap}(a),(d) reinforce the empirical regimes discussed in Section~6. While the trends largely follow theoretical predictions, a few configurations---e.g., $(K/n, H)\!=\!(0.2, 20)$ in Figure~\ref{fig:app:cifar_heatmap}(c), and $(0.2, 15)$ and $(0.2, 20)$ in Figure~\ref{fig:app:cifar_heatmap}(d)---depart from the expected behavior. These are examined in detail below.

\paragraph{Analysis of Outlier Cases.} 
Figures~\ref{fig:app:cifar_rounds} and~\ref{fig:app:cifar_rounds_H} report test accuracy over $T = 1000$ communication rounds on CIFAR-10 with $K/n\!=\!0.2$ and ring topology, focusing on the outlier configurations identified in Figure~\ref{fig:app:cifar_heatmap}(c)--(d).
Figure~\ref{fig:app:cifar_rounds} fixes $H\!=\!20$ and compares the two opposite heterogeneity regimes: (a) intra non-IID, inter IID; and (b) intra IID, inter non-IID. 
Figure~\ref{fig:app:cifar_rounds_H} fixes the regime to intra non-IID, inter non-IID and reports two configurations from Figure~\ref{fig:app:cifar_heatmap}(d) with $H\!\in\!\{15, 20\}$. Specifically:
\begin{itemize}
    \item In Figure~\ref{fig:app:cifar_rounds}(a), S2A performs comparably to or better than S2S when inter-component heterogeneity is IID, even under non-IID intra-component distributions (Regimes R1 and R2).
    \item In Figure~\ref{fig:app:cifar_rounds}(b), where inter-component heterogeneity is non-IID, S2A performs better in early rounds (explaining the accuracy gap at $T = 100$ for $(K/n, H)\!=\!(0.2, 20)$ in Figure~\ref{fig:app:cifar_heatmap}(c)), but its performance deteriorates over time, with periodic drops at each D2S round. S2S, in contrast, shows more stable convergence and outperforms S2A for $T > 100$ (Regime R3).
    \item Figure~\ref{fig:app:cifar_rounds_H}(a) reports the case $(K/n, H)\!=\!(0.2, 15)$ from Figure~\ref{fig:app:cifar_heatmap}(d), where S2A performs best at $T\!=\!100$ (+4.8 p.p.), but is eventually outperformed by S2S (+6 p.p. at $T\!=\!1000$), consistent with theoretical results under high inter heterogeneity.
    \item Similarly, Figure~\ref{fig:app:cifar_rounds_H}(b) reports the case $(K/n, H)\!=\!(0.2, 20)$ from Figure~\ref{fig:app:cifar_heatmap}(d), where S2A initially performs better than S2S (+6.9 p.p. at $T = 100$), but is eventually outperformed by S2S (+11 p.p. at $T = 1000$).
\end{itemize}

\paragraph{Dynamic Topologies.}
Figure~\ref{fig:fixed_vs_dynamic} compares fixed and time-varying D2D graphs on CIFAR-10 in Regime~R3 (intra IID, inter non-IID). In these experiments, we fix $K/n = 0.2$ and $H = 20$, and consider a fixed regular graph and a random regular graph with the same degree (4). Recall that in our analysis, D2D variability is fully captured by $p$ (Lemma~\ref{lem:spectral-mixing}); the random regular topology yields faster intra-component mixing (we measure $p_\text{random} \!\approx\! 0.8 \!\gg\! p_\text{fixed} \!\approx\! 0.2$). We observe that moving from a fixed to a dynamically switching topology improves the final test accuracy of both S2S and S2A (up to $+3.4$ and $+4.5$~p.p., respectively), while leaving the qualitative S2S/S2A regime unchanged. Specifically, the average S2S/S2A gap increases from $+8.58 \pm 0.32$~p.p.\ (fixed) to $+11.52 \pm 0.42$~p.p.\ (random), i.e., by $+2.94 \pm 0.20$~p.p. These experiments with dynamic topologies (random regular graphs) are consistent with our theory for time-varying graphs with randomly switching links (Appendix~D) and suggest that dynamic topologies benefit S2S more than S2A in Regime~R3.

\paragraph{Alternative Server Optimizers.}
Figure~\ref{fig:fedavg_vs_fedavgm} compares FedAvg and FedAvgM (with momentum $\beta = 0.9$) on CIFAR-10 with ring topology in Regime~R3, for $K/n = 0.2$ and $H = 20$. We observe that the mean S2S--S2A gap is essentially unchanged: $+8.10 \pm 0.37$~p.p.\ for FedAvg and $+8.07 \pm 0.35$~p.p.\ for FedAvgM, yielding a difference of $-0.03 \pm 0.10$~p.p. Thus, the empirical results suggest that introducing momentum does not significantly alter the relative performance between S2S and S2A. In the last $100$ rounds, however, FedAvgM+S2A exhibits smaller accuracy drops (about $20\%$ reduction) compared to FedAvg+S2A, indicating that momentum can help reduce the S2A accuracy drops under high inter-component heterogeneity, although without significantly changing its asymptotic accuracy gap to S2S.

\paragraph{Experiments on CIFAR-100.}
Figure~\ref{fig:app:cifar100} reports test accuracy on CIFAR-100 under Regimes~R2 and~R3 for $K/n = 0.2$, $H = 20$, and complete topology. In Regime~R2 (intra non-IID, inter IID), S2A outperforms S2S by approximately $1.9 \pm 0.02$~p.p.\ on average, consistent with the conclusion that the broadcast operator (S2A) is beneficial when inter-component heterogeneity is low. In Regime~R3 (intra IID, inter non-IID), the performance advantage switches in favor of S2S, which now outperforms S2A by approximately $13.6 \pm 1.0$~p.p.\ on average, confirming that the broadcast-induced bias of S2A becomes detrimental under high inter-component heterogeneity.

\paragraph{Bias and Disagreement Errors.}
To further validate our theoretical analysis and the S2S/S2A comparison in Section~\ref{subsec:comparison}, Figure~\ref{fig:track} tracks the bias and disagreement errors, as defined in Section~\ref{subsec:comparison}(i)--(ii), over $T = 1000$ rounds on CIFAR-10 (ring topology, Regime~R3, $K/n = 0.2$, $H = 20$). In detail:
\begin{itemize}
    \item Figure~\ref{fig:track}(a) reports the test loss over communication rounds;
    \item Figure~\ref{fig:track}(b) reports the disagreement error at D2D rounds;
    \item Figures~\ref{fig:track}(c)--(d) report the disagreement and bias errors at D2S rounds;
    \item Figures~\ref{fig:track}(e)--(f) report the empirical disagreement and bias ratios at D2S rounds.
\end{itemize}

Consistently with our analysis, for S2A we observe pronounced spikes in the bias error after each D2S round and (numerically) zero disagreement error, whereas for S2S the bias error is zero but a non-zero disagreement remains after the D2S step. The non-zero bias of S2A correlates with its performance degradation relative to S2S in Regime~R3, while the residual disagreement of S2S does not prevent it from achieving higher final accuracy in this regime. Finally, the empirical disagreement and bias ratios in Figs.~\ref{fig:track}(e)--(f) oscillate around the theoretical values predicted by Eqs.~\eqref{eq:deviation_sts}--\eqref{eq:deviation_sta}: for $n = 100$ and $K = 20$, we have a disagreement ratio of $\frac{n-K}{n-1} \approx 0.81$ and a bias ratio of $\frac{n-K}{K(n-1)} \approx 0.04$, reinforcing the consistency between our analysis and experimental results.

\newpage

\begin{table}[h]
\vspace{3cm}
\centering
\caption{Test accuracy gap (percentage-point difference) between S2S and S2A, aggregated from Figures~\ref{fig:MNIST} and~\ref{fig:CIFAR} over 12 configurations ($K/n\!\in\!\{0.2,0.4,0.6,0.8\}$ $\times$ 3 topologies) with fixed server period $H\!=\!5$. S2S/S2A/-- counts the number of configurations where each method outperforms the other (gaps below the standard error are denoted by ``--''). 
Positive gaps favor~S2S. We report the mean, standard error, maximum gap values, and $p$-values from a $t$-test over the 12 comparisons.}

\vspace{0.5em}
\begin{tabular}{llcccccc}
\toprule
\textbf{Dataset} & \textbf{Intra/Inter Regime}
                 & \textbf{S2S/S2A/--}
                 & \textbf{Gap (mean $\pm$ se)}
                 & \textbf{$p$-value}
                 & \textbf{Gap (max)}
                 & \textbf{($K/n$, D2D topology)} \\ 
\midrule
MNIST   & IID / IID          & 0/9/3 & $-0.01 \pm 0.00$ & 0.007 & $-0.05$ & (0.2, ring) \\
MNIST   & non-IID / IID     & 6/6/0  & $+0.00 \pm 0.03$ & 0.859 & $+0.25$ & (0.2, complete) \\
MNIST   & IID / non-IID      & 12/0/0 & $+0.91 \pm 0.22$ & 0.001 & $+2.37$ & (0.2, complete) \\
MNIST   & non-IID / non-IID  & 12/0/0 & $+0.86 \pm 0.18$ & $<\!0.001$ & $+1.96$ & (0.2, ring) \\
\addlinespace[0.25em]
CIFAR-10 & IID / IID         & 1/10/1 & $-0.31 \pm 0.08$ & 0.002 & $-0.97$ & (0.2, ring) \\
CIFAR-10 & non-IID / IID     & 9/3/0 & $+1.80 \pm 0.73$ & 0.028 & $+8.43$ & (0.2, ring) \\
CIFAR-10 & IID / non-IID     & 12/0/0 & $+1.11 \pm 0.28$ & 0.001 & $+3.41$ & (0.2, ring) \\
CIFAR-10 & non-IID / non-IID & 12/0/0 & $+2.87 \pm 0.55$ & $<\!0.001$ & $+7.01$ & (0.2, complete) \\
\bottomrule
\label{tab:aggregate_K}
\end{tabular}
\vspace{0.5em}
\end{table}
\hfill
\vspace{2cm}
\begin{table}[h]
\centering
\caption{Test accuracy gap (percentage‐point difference) between S2S and S2A, aggregated from Figures~\ref{fig:MNIST} and~\ref{fig:CIFAR} over 12 configurations ($H\!\in\!\{5,10,15,20\}$ $\times$ 3 topologies) with fixed sampling rate $K/n\!=\!0.2$.  
S2S/S2A/-- counts the number of configurations where each method outperforms the other.
Positive gaps favor S2S.
We report the mean, standard error, maximum gap values, and $p$-values from a $t$-test over the 12 comparisons.
}
\vspace{0.5em}
\begin{tabular}{llcccccc}
\toprule
\textbf{Dataset} & \textbf{Intra/Inter Regime}
                 & \textbf{S2S/S2A/--}
                 & \textbf{Gap (mean $\pm$ se)}
                 & \textbf{$p$-value}
                 & \textbf{Gap (max)}
                 & \textbf{($H$, D2D topology)} \\ 
\midrule
MNIST   & IID / IID          & 1/10/1 & $-0.02 \pm 0.01$ & 0.006 & $-0.06$ & (10, ring) \\
MNIST   & non-IID / IID      & 8/2/2  & $+0.17 \pm 0.09$ & 0.076 & $+0.84$ & (20, complete) \\
MNIST   & IID / non-IID      & 11/1/0 & $+1.62 \pm 0.24$ & $<\!0.001$ & $+2.35$ & (5, complete) \\
MNIST   & non-IID / non-IID  & 12/0/0 & $+2.01 \pm 0.18$ & $<\!0.001$ & $+3.00$ & (10, complete) \\
\addlinespace[0.25em]
CIFAR-10 & IID / IID         & 0/10/2 & $-0.40 \pm 0.09$ & $<\!0.001$ & $-0.97$ & (5, ring) \\
CIFAR-10 & non-IID / IID     & 7/5/0  & $+0.99 \pm 1.15$ & 0.411 & $+8.52$ & (5, ring) \\
CIFAR-10 & IID / non-IID     & 11/1/0 & $+2.44 \pm 0.42$ & $<\!0.001$ & $+4.42$ & (10, complete) \\
CIFAR-10 & non-IID / non-IID & 7/5/0  & $-0.26 \pm 1.15$ & 0.827 & $+7.01$ & (5, complete) \\
\bottomrule
\end{tabular}
\label{tab:aggregate_H}
\vspace{0.5em}
\end{table}

\begin{table}
\centering
\caption{Test accuracy gap (percentage‐point difference) between S2S and S2A on MNIST, reported from Figure~\ref{fig:MNIST} for varying sampling rates $K/n\!\in\!\{0.2,0.4,0.6,0.8\}$ with fixed server period $H\!=\!5$. Positive gaps favor S2S. Each row corresponds to a heterogeneity regime and sampling rate ($K/n$), and each column to a D2D topology.  
Each entry is annotated with the best strategy: S2S, S2A, or -- (gap below standard error).}
\begin{tabular}{lcccc}
\toprule
\textbf{Intra/Inter Regime} & \textbf{Sampling rate ($K/n$)} & \textbf{Complete} & \textbf{Grid} & \textbf{Ring} \\
\midrule
IID / IID         & 0.2 & $-0.01$ (S2A) & $-0.02$ (S2A) & $-0.05$ (S2A) \\
                  & 0.4 & -- & $-0.02$ (S2A) & $-0.04$ (S2A) \\
                  & 0.6 & -- & $-0.01$ (S2A) & $-0.02$ (S2A) \\
                  & 0.8 & -- & $-0.01$ (S2A) & $-0.01$ (S2A) \\
\addlinespace
non-IID / IID     & 0.2 & $+0.25$ (S2S) & $+0.17$ (S2S) & $+0.04$ (S2S) \\
                  & 0.4 & $+0.04$ (S2S) & $-0.06$ (S2A) & $-0.16$ (S2A) \\
                  & 0.6 & $+0.06$ (S2S) & $-0.04$ (S2A) & $-0.10$ (S2A) \\
                  & 0.8 & $+0.01$ (S2S) & $-0.03$ (S2A) & $-0.05$ (S2A) \\
\addlinespace
IID / non-IID     & 0.2 & $+2.37$ (S2S) & $+2.36$ (S2S) & $+2.14$ (S2S) \\
                  & 0.4 & $+1.38$ (S2S) & $+1.38$ (S2S) & $+1.26$ (S2S) \\
                  & 0.6 & $+0.65$ (S2S) & $+0.71$ (S2S) & $+0.76$ (S2S) \\
                  & 0.8 & $+0.31$ (S2S) & $+0.29$ (S2S) & $+0.06$ (S2S) \\
\addlinespace
non-IID / non-IID & 0.2 & $+1.65$ (S2S) & $+1.87$ (S2S) & $+1.96$ (S2S) \\
                  & 0.4 & $+1.14$ (S2S) & $+1.31$ (S2S) & $+1.41$ (S2S) \\
                  & 0.6 & $+0.73$ (S2S) & $+0.84$ (S2S) & $+0.86$ (S2S) \\
                  & 0.8 & $+0.35$ (S2S) & $+0.39$ (S2S) & $+0.35$ (S2S) \\
\bottomrule
\end{tabular}
\label{tab:mnist_K}
\end{table}

\begin{table}
\centering
\caption{Test accuracy gap (percentage-point difference) between S2S and S2A on CIFAR-10, reported from Figure~\ref{fig:CIFAR} for varying sampling rates $K/n\!\in\!\{0.2,0.4,0.6,0.8\}$ with fixed server period $H\!=\!5$.  
Positive gaps favor S2S.  
Each row corresponds to a heterogeneity regime and sampling rate ($K/n$); each column to a D2D topology.  
Every entry is annotated with the best strategy: S2S, S2A, or -- (gap below standard error).}
\begin{tabular}{lcccc}
\toprule
\textbf{Intra/Inter Regime} & \textbf{Sampling rate ($K/n$)} & \textbf{Complete} & \textbf{Grid} & \textbf{Ring} \\
\midrule
IID / IID         & 0.2 & --  & $-0.30$ (S2A) & $-0.97$ (S2A) \\
                  & 0.4 & $+0.18$ (S2S) & $-0.44$ (S2A) & $-0.71$ (S2A) \\
                  & 0.6 & $-0.45$ (S2A) & $-0.38$ (S2A) & $-0.54$ (S2A) \\
                  & 0.8 & $-0.41$ (S2A) & $-0.24$ (S2A) & $-0.34$ (S2A) \\
\addlinespace
non-IID / IID     & 0.2 & $+0.01$ (S2S) & $+7.74$ (S2S) & $+8.43$ (S2S) \\
                  & 0.4 & $-0.50$ (S2A) & $+2.69$ (S2S) & $+3.99$ (S2S) \\
                  & 0.6 & $-0.25$ (S2A) & $+1.81$ (S2S) & $+1.42$ (S2S) \\
                  & 0.8 & $-0.08$ (S2A) & $+0.92$ (S2S) & $+1.03$ (S2S) \\
\addlinespace
IID / non-IID     & 0.2 & $+1.50$ (S2S) & $+2.96$ (S2S) & $+3.42$ (S2S) \\
                  & 0.4 & $+1.12$ (S2S) & $+0.24$ (S2S) & $+1.21$ (S2S) \\
                  & 0.6 & $+0.76$ (S2S) & $+1.21$ (S2S) & $+1.92$ (S2S) \\
                  & 0.8 & $+0.19$ (S2S) & $+0.34$ (S2S) & $+2.03$ (S2S) \\
\addlinespace
non-IID / non-IID & 0.2 & $+7.01$ (S2S) & $+2.50$ (S2S) & $+3.60$ (S2S) \\
                  & 0.4 & $+5.99$ (S2S) & $+4.38$ (S2S) & $+4.54$ (S2S) \\
                  & 0.6 & $+3.97$ (S2S) & $+3.33$ (S2S) & $+2.88$ (S2S) \\
                  & 0.8 & $+2.33$ (S2S) & $+1.44$ (S2S) & $+1.15$ (S2S) \\
\bottomrule
\end{tabular}
\label{tab:cifar_K}
\end{table}

\begin{table}
\centering
\caption{Test accuracy gap (percentage-point difference) between S2S and S2A on MNIST, reported from Figure~\ref{fig:MNIST} for varying server periods $H\in\{5,10,15,20\}$ with fixed sampling rate $K/n=0.2$.  
Positive gaps favor S2S.  
Each row corresponds to a heterogeneity regime and a server period; each column to a D2D topology.  
Every entry is annotated with the best strategy: S2S, S2A, or -- (gap below standard error).}
\begin{tabular}{lcccc}
\toprule
\textbf{Intra/Inter Regime} & \textbf{Server period ($H$)} & \textbf{Complete} & \textbf{Grid} & \textbf{Ring} \\
\midrule
IID / IID         & 5  & $-0.01$ (S2A) & $-0.02$ (S2A) & $-0.05$ (S2A) \\
                  & 10 & $-0.01$ (S2A) & $-0.02$ (S2A) & $-0.06$ (S2A) \\
                  & 15 & --            & $-0.03$ (S2A) & $-0.05$ (S2A) \\
                  & 20 & $+0.03$ (S2S) & $-0.03$ (S2A) & $-0.04$ (S2A) \\
\addlinespace
non-IID / IID     & 5  & $+0.25$ (S2S) & $+0.17$ (S2S) & $+0.04$ (S2S) \\
                  & 10 & $+0.29$ (S2S) & $+0.02$ (S2S) & $-0.10$ (S2A) \\
                  & 15 & $+0.71$ (S2S) & $+0.29$ (S2S) & $-0.14$ (S2A) \\
                  & 20 & $+0.84$ (S2S) & --   & --   \\
\addlinespace
IID / non-IID     & 5  & $+2.35$ (S2S) & $+2.35$ (S2S) & $+2.18$ (S2S) \\
                  & 10 & $+2.34$ (S2S) & $+2.33$ (S2S) & $+2.02$ (S2S) \\
                  & 15 & $+1.64$ (S2S) & $+1.66$ (S2S) & $+0.79$ (S2S) \\
                  & 20 & $+1.08$ (S2S) & $+1.05$ (S2S) & $-0.30$ (S2A) \\
\addlinespace
non-IID / non-IID & 5  & $+1.65$ (S2S) & $+1.87$ (S2S) & $+1.96$ (S2S) \\
                  & 10 & $+3.00$ (S2S) & $+2.96$ (S2S) & $+2.62$ (S2S) \\
                  & 15 & $+2.31$ (S2S) & $+2.06$ (S2S) & $+1.69$ (S2S) \\
                  & 20 & $+1.51$ (S2S) & $+1.36$ (S2S) & $+1.09$ (S2S) \\
\bottomrule
\end{tabular}
\label{tab:mnist_H}
\end{table}

\begin{table}
\centering
\caption{Test accuracy gap (percentage-point difference) between S2S and S2A on CIFAR-10, reported from Figure~\ref{fig:CIFAR} for varying server periods $H\!\in\!\{5,10,15,20\}$ with fixed sampling rate $K/n=0.2$.  
Positive gaps favor S2S.  
Each row corresponds to a heterogeneity regime and a server period; each column to a D2D topology.  
Every entry is annotated with the best strategy: S2S, S2A, or -- (gap below standard error).}
\begin{tabular}{lcccc}
\toprule
\textbf{Intra/Inter Regime} & \textbf{Server period ($H$)} & \textbf{Complete} & \textbf{Grid} & \textbf{Ring} \\
\midrule
IID / IID         & 5  & --   & $-0.30$ (S2A) & $-0.97$ (S2A) \\
                  & 10 & $-0.52$ (S2A) & $-0.40$ (S2A) & $-0.41$ (S2A) \\
                  & 15 & $-0.16$ (S2A) & $-0.74$ (S2A) & $-0.30$ (S2A) \\
                  & 20 & $-0.13$ (S2A) & $-0.78$ (S2A) & --   \\
\addlinespace
non-IID / IID     & 5  & $+0.46$ (S2S) & $+7.28$ (S2S) & $+8.52$ (S2S) \\
                  & 10 & $-3.99$ (S2A) & $+2.50$ (S2S) & $+2.46$ (S2S) \\
                  & 15 & $-3.33$ (S2A) & $+1.14$ (S2S) & $-0.45$ (S2A) \\
                  & 20 & $-2.50$ (S2A) & $+1.25$ (S2S) & $-2.37$ (S2A) \\
\addlinespace
IID / non-IID     & 5  & $+1.58$ (S2S) & $+2.94$ (S2S) & $+3.49$ (S2S) \\
                  & 10 & $+4.43$ (S2S) & $+4.14$ (S2S) & $+3.42$ (S2S) \\
                  & 15 & $+2.98$ (S2S) & $+1.74$ (S2S) & $+1.07$ (S2S) \\
                  & 20 & $+1.23$ (S2S) & $+2.58$ (S2S) & $-0.84$ (S2A) \\
\addlinespace
non-IID / non-IID & 5  & $+7.01$ (S2S) & $+2.42$ (S2S) & $+3.60$ (S2S) \\
                  & 10 & $+0.21$ (S2S) & $+2.18$ (S2S) & $+1.87$ (S2S) \\
                  & 15 & $-3.27$ (S2A) & --   & $-4.79$ (S2A) \\
                  & 20 & $-4.04$ (S2A) & $-1.38$ (S2A) & $-6.90$ (S2A) \\
\bottomrule
\end{tabular}
\label{tab:cifar_H}
\end{table}

\begin{figure}
    \centering
    % Row 1
    \begin{subfigure}[t]{0.3\textwidth}
        \includegraphics[width=\linewidth]{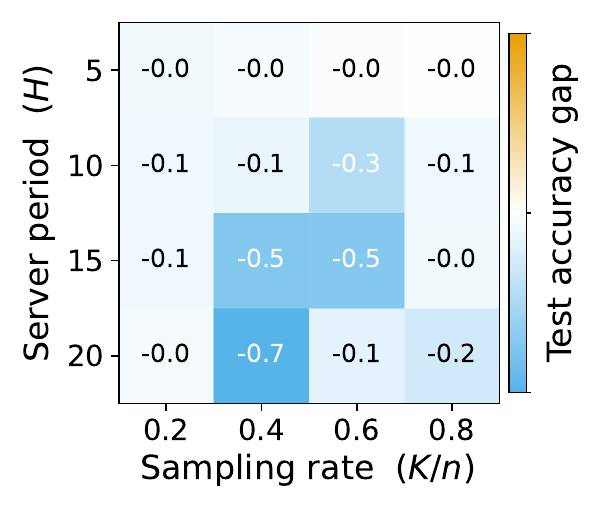}
        \caption{Intra IID, Inter IID}
    \end{subfigure}
    \hspace{2cm}
    \begin{subfigure}[t]{0.3\textwidth}
        \includegraphics[width=\linewidth]{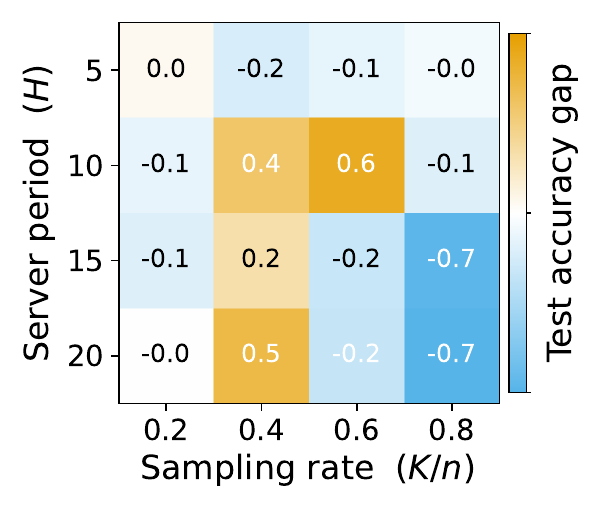}
        \caption{Intra non-IID, Inter IID}
    \end{subfigure}
    % Row 2
    \begin{subfigure}[t]{0.3\textwidth}
        \includegraphics[width=\linewidth]{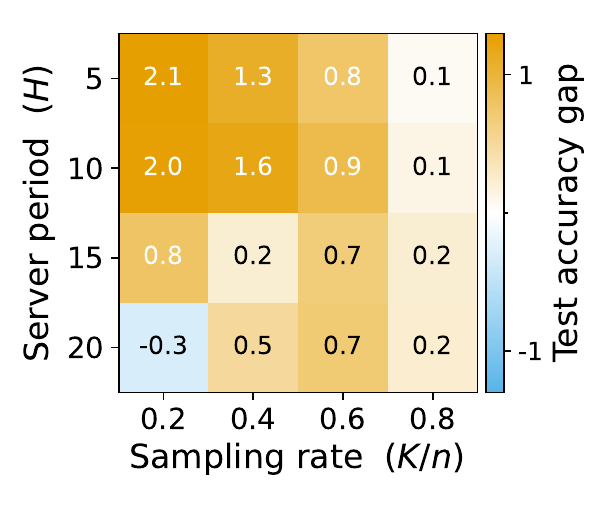}
        \caption{Intra IID, Inter non-IID}
    \end{subfigure}
    \hspace{2cm}
    \begin{subfigure}[t]{0.3\textwidth}
        \includegraphics[width=\linewidth]{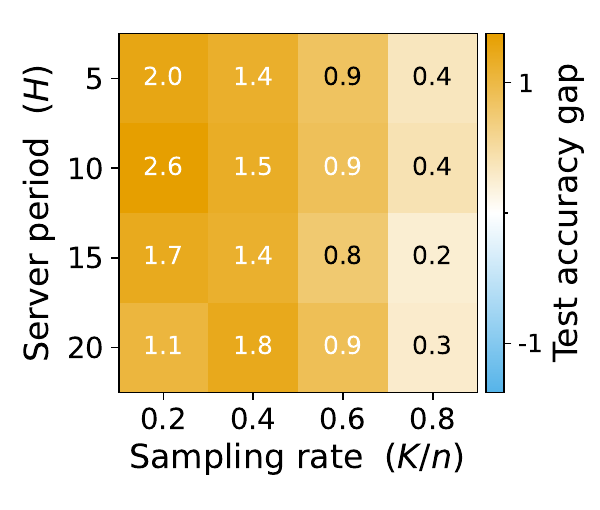}
        \caption{Intra non-IID, Inter non-IID}
    \end{subfigure}
    \caption{\normalsize Accuracy gap on MNIST with ring topology.}
    \label{fig:app:mnist_heatmap}
\end{figure}
\begin{figure}
    \centering
    % Row 3
    \begin{subfigure}[t]{0.3\textwidth}
        \includegraphics[width=\linewidth]{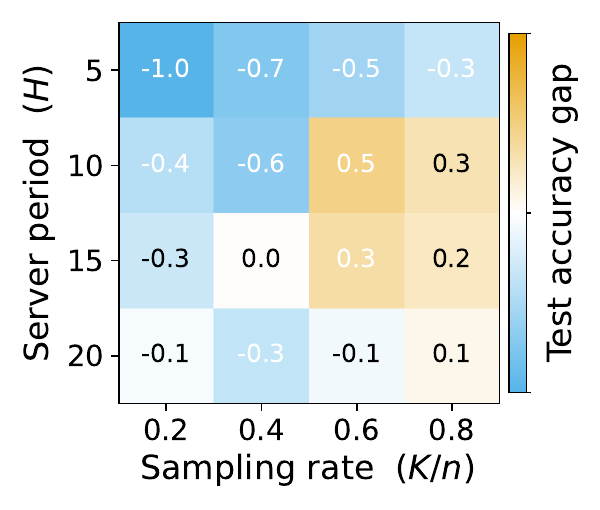}
        \caption{Intra IID, Inter IID}
    \end{subfigure}
    \hspace{2cm}
    \begin{subfigure}[t]{0.3\textwidth}
        \includegraphics[width=\linewidth]{Figs/cifar_heatmap_iid_non_iid.pdf}
        \caption{Intra non-IID, Inter IID}
    \end{subfigure}
    % Row 4
    \begin{subfigure}[t]{0.3\textwidth}
        \includegraphics[width=\linewidth]{Figs/cifar_heatmap_non_iid_iid.pdf}
        \caption{Intra IID, Inter non-IID}
    \end{subfigure}
    \hspace{2cm}
    \begin{subfigure}[t]{0.3\textwidth}
        \includegraphics[width=\linewidth]{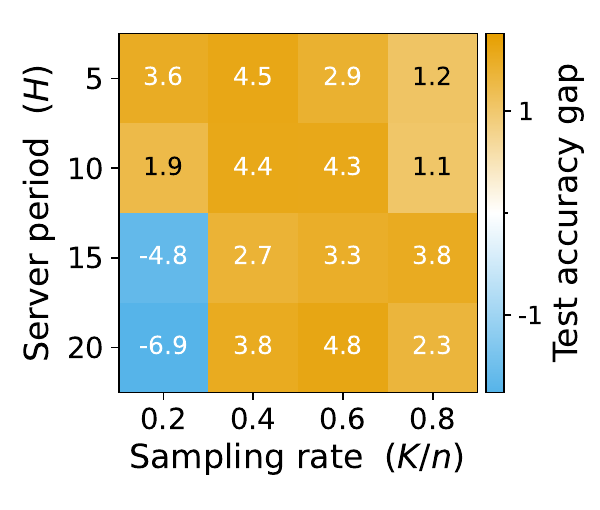}
        \caption{Intra non-IID, Inter non-IID}
    \end{subfigure}
    \caption{\normalsize Accuracy gap on CIFAR-10 with ring topology.}
    \label{fig:app:cifar_heatmap}
\end{figure}

\begin{figure}
    \centering

    \begin{minipage}[t]{\textwidth}
        \centering
        \begin{subfigure}[b]{0.3\linewidth}
            \centering
            \includegraphics[width=\linewidth]{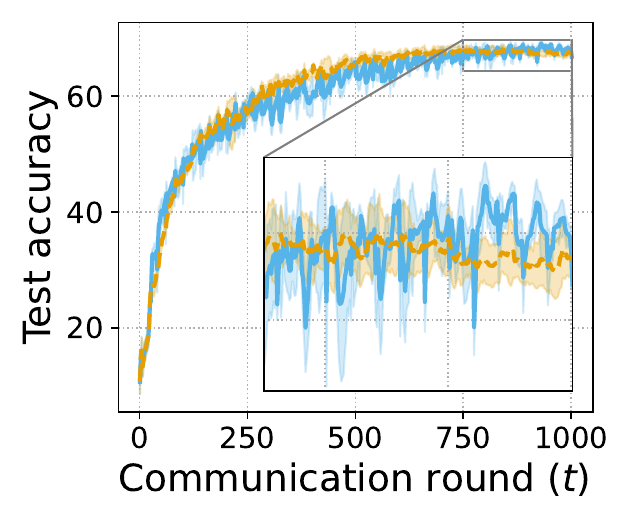}
            \caption{Regime R2 (intra non-IID, inter IID)}
            \label{fig:app:rounds_non_iid_iid}
        \end{subfigure}
        \hspace{1cm}
        \begin{subfigure}[b]{0.3\linewidth}
            \centering
            \includegraphics[width=\linewidth]{Figs/cifar_rounds_non_iid_iid.pdf}
            \caption{Regime R3 (intra IID, inter non-IID)}
            \label{fig:app:rounds_iid_non_iid}
        \end{subfigure}
        \caption{\normalsize Test accuracy on CIFAR-10; $K/n=0.2$, $H=20$, ring topology.}
        \label{fig:app:cifar_rounds}
    \end{minipage}
    
    \vspace{6em}
    
    \begin{minipage}[t]{\textwidth}
        \centering
        \begin{subfigure}[b]{0.3\linewidth}
            \centering
            \includegraphics[width=\linewidth]{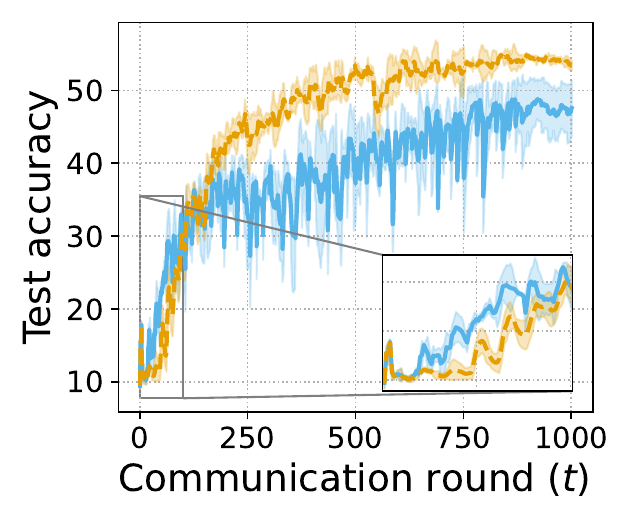}
            \caption{$H=15$}
            \label{fig:app:rounds_non_iid_non_iid_H_15}
        \end{subfigure}
        \hspace{1cm}
        \begin{subfigure}[b]{0.3\linewidth}
            \centering
            \includegraphics[width=\linewidth]{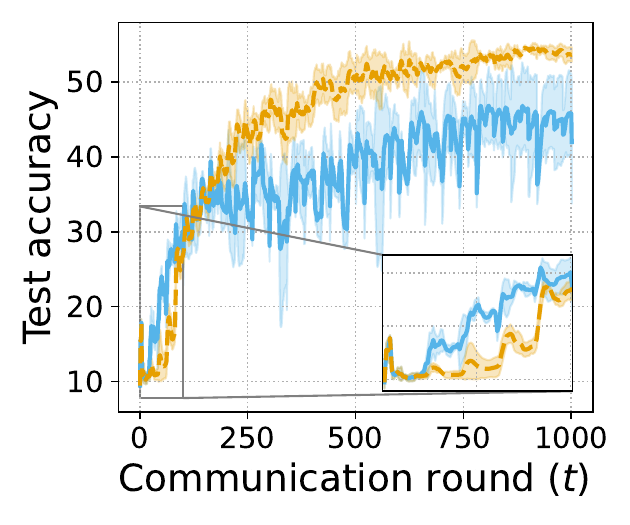}
            \caption{$H=20$}
            \label{fig:app:rounds_non_iid_non_iid_H_20}
        \end{subfigure}
        \caption{\normalsize Test accuracy on CIFAR-10; Regime R3 (intra non-IID, inter non-IID), $K/n=0.2$, ring topology.}
        \label{fig:app:cifar_rounds_H}
    \end{minipage}

\end{figure}

\clearpage
\begin{figure}
    \centering
    \includegraphics[width=0.9\textwidth]{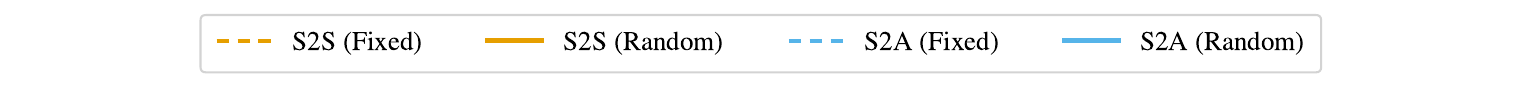}
    
    \includegraphics[width=0.55\linewidth]{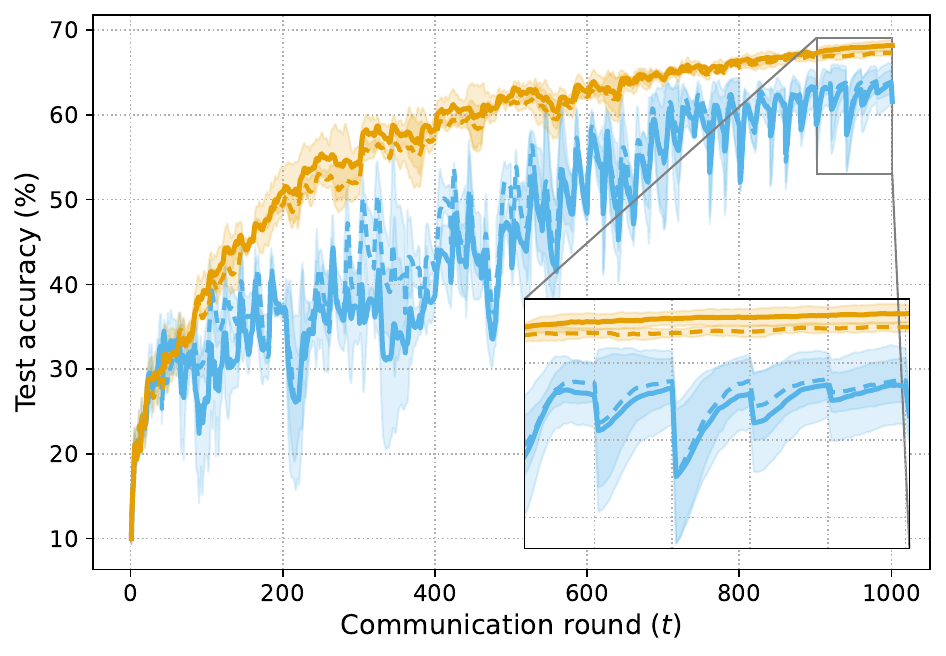}
    \caption{Test accuracy on CIFAR-10, Regime R3 (intra IID, inter non-IID), $K/n\!=\!0.2$, $H\!=\!20$, fixed/random regular graphs.} 
    \label{fig:fixed_vs_dynamic}
\end{figure}

\begin{figure}
    \centering
    \includegraphics[width=0.9\textwidth]{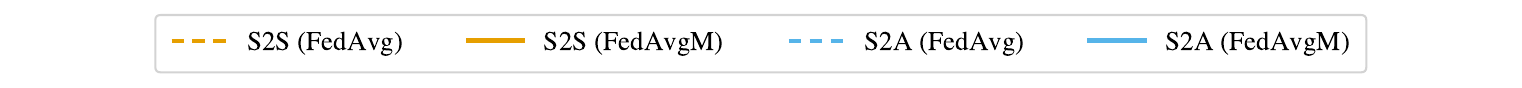}
    
    \includegraphics[width=0.55\linewidth]{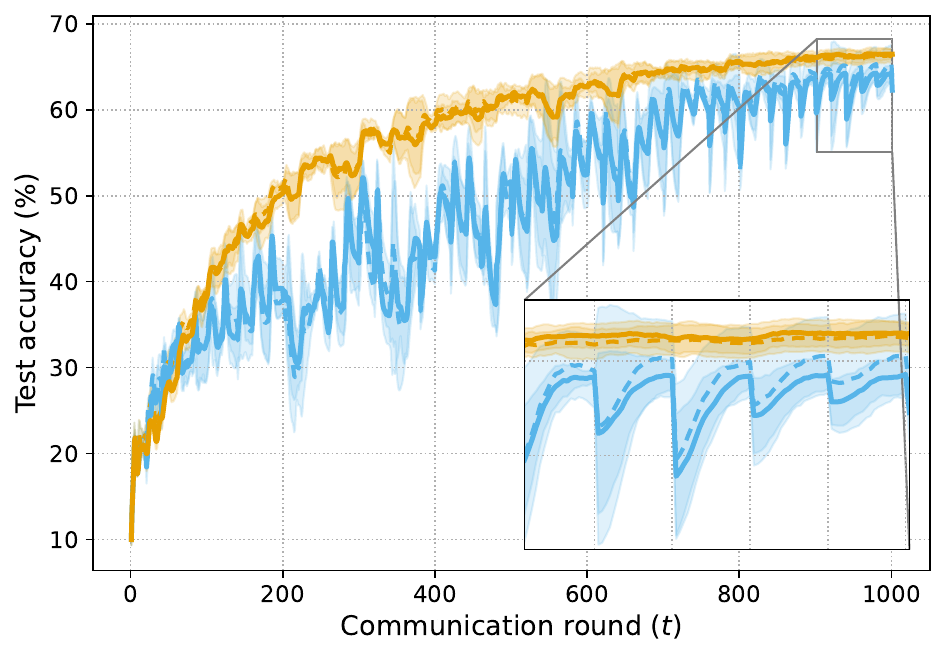}
    \caption{Test accuracy on CIFAR-10, Regime R3, $K/n\!=\!0.2$, $H\!=\!20$, ring topology, momentum $\beta=0.9$.}
    \label{fig:fedavg_vs_fedavgm}
\end{figure}

\begin{figure}
    \centering
        \begin{subfigure}[b]{0.5\linewidth}
            \centering
            \includegraphics[width=\linewidth]{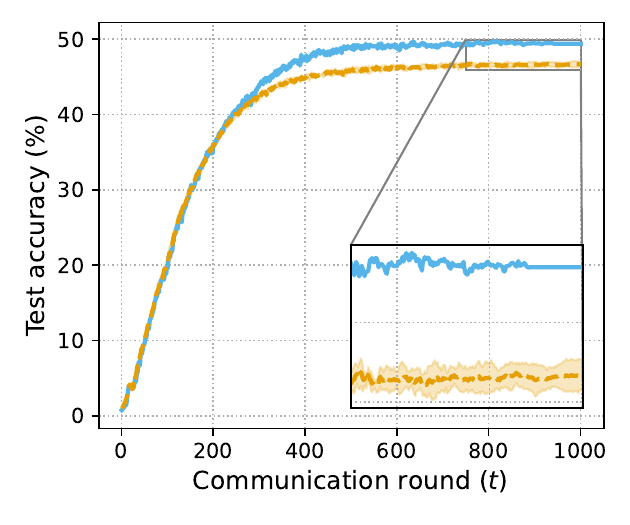}
            \caption{Regime R2 (intra non-IID, inter IID)}
            \label{fig:app:CIFAR100_R2}
        \end{subfigure}
        \hspace{1cm}
        \begin{subfigure}[b]{0.5\linewidth}
            \centering
            \includegraphics[width=\linewidth]{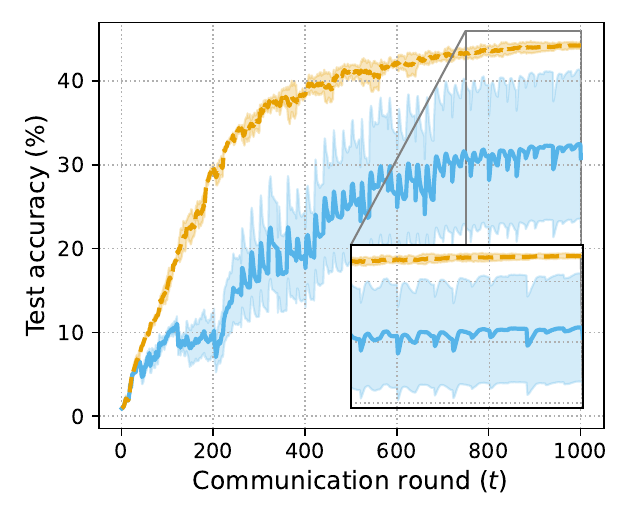}
            \caption{Regime R3 (intra IID, inter non-IID)}
            \label{fig:app:CIFAR100_R3}
        \end{subfigure}
        \caption{\normalsize Test accuracy on CIFAR-100; $K/n=0.2$, $H=20$, complete topology.}
        \label{fig:app:cifar100}
\end{figure}

\begin{figure}
    \centering
    % Row 1
    \begin{subfigure}[b]{0.46\linewidth}
        \centering
        \includegraphics[width=\linewidth]{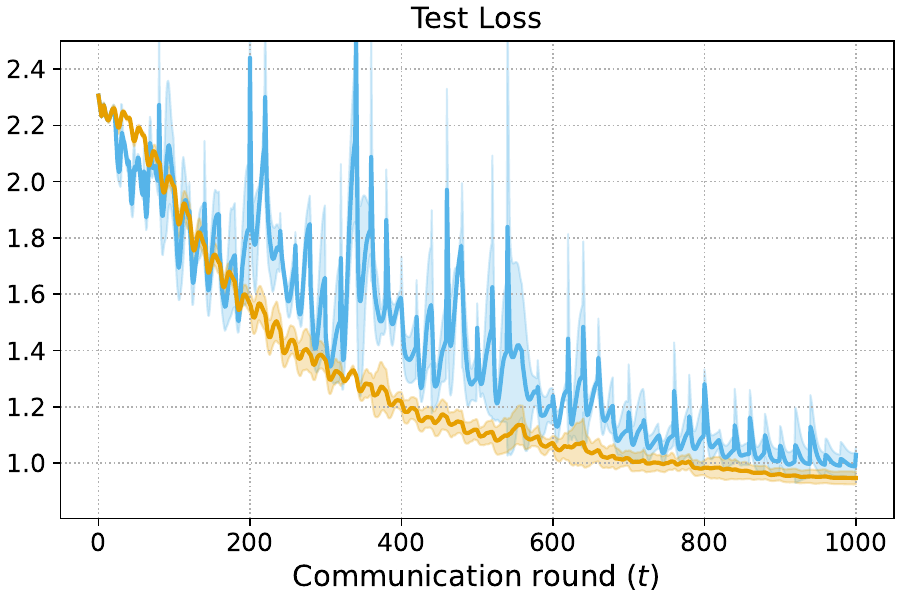}
        \caption{Test loss over communication rounds.}
        \label{fig:app:track_loss}
    \end{subfigure}
    \hspace{1cm}
    \begin{subfigure}[b]{0.46\linewidth}
        \centering
        \includegraphics[width=\linewidth]{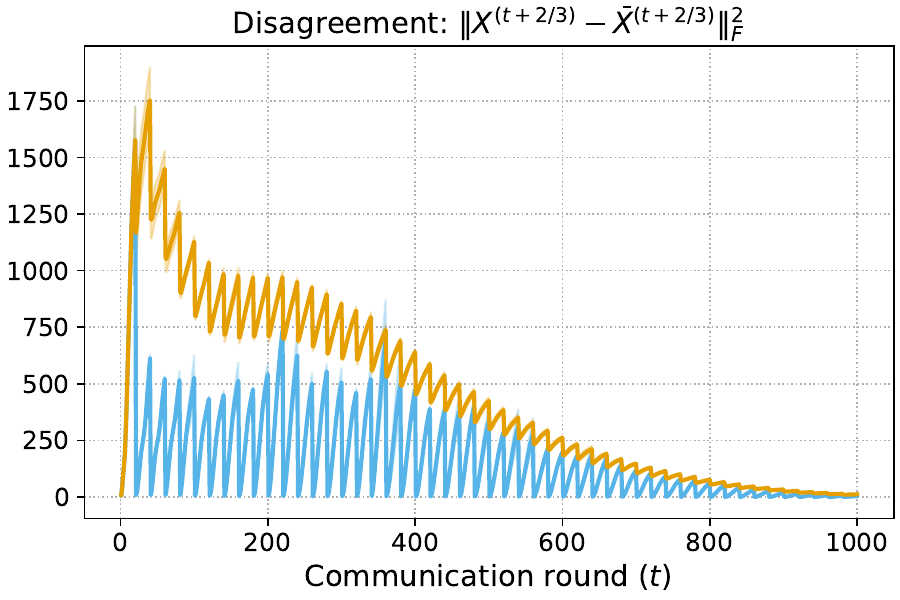}
        \caption{Disagreement error at D2D rounds.}
        \label{fig:app:track_disagreement1}
    \end{subfigure}

    \vspace{0.6cm}

    % Row 2
    \begin{subfigure}[b]{0.46\linewidth}
        \centering
        \includegraphics[width=\linewidth]{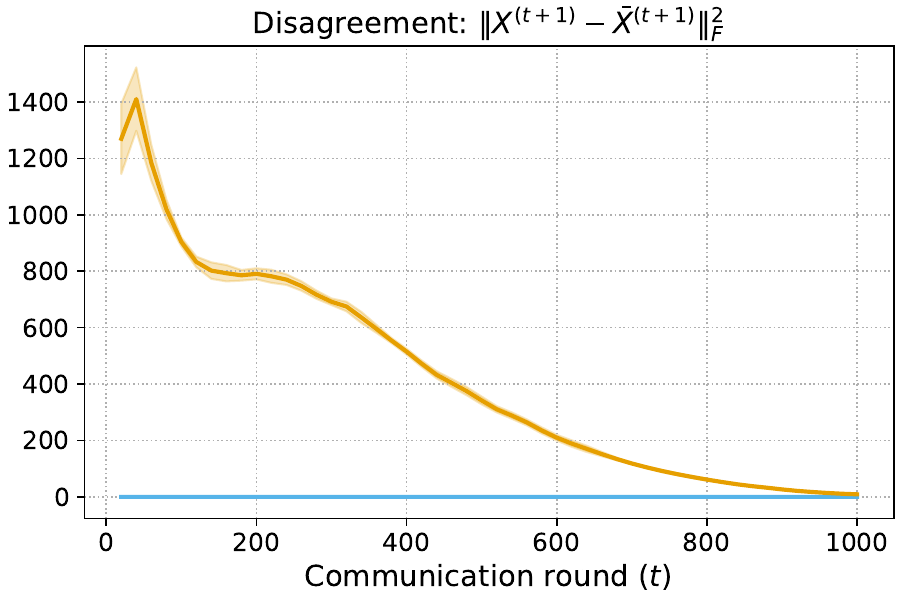}
        \caption{Disagreement error at D2S rounds.}
        \label{fig:app:track_disagreement2}
    \end{subfigure}
    \hspace{1cm}
    \begin{subfigure}[b]{0.46\linewidth}
        \centering
        \includegraphics[width=\linewidth]{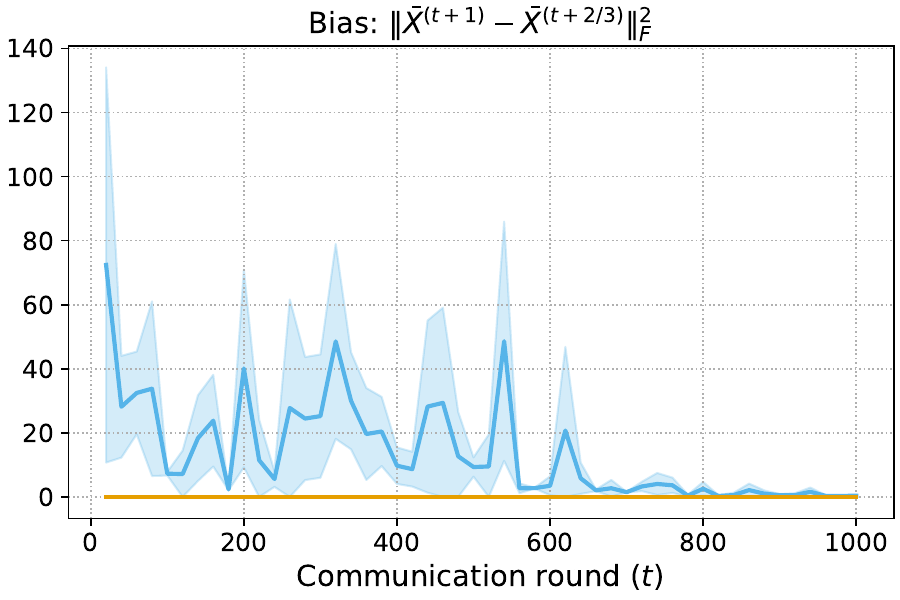}
        \caption{Bias error at D2S rounds.}
        \label{fig:app:track_bias}
    \end{subfigure}

    \vspace{0.6cm}

    % Row 3
    \begin{subfigure}[b]{0.46\linewidth}
        \centering
        \includegraphics[width=\linewidth]{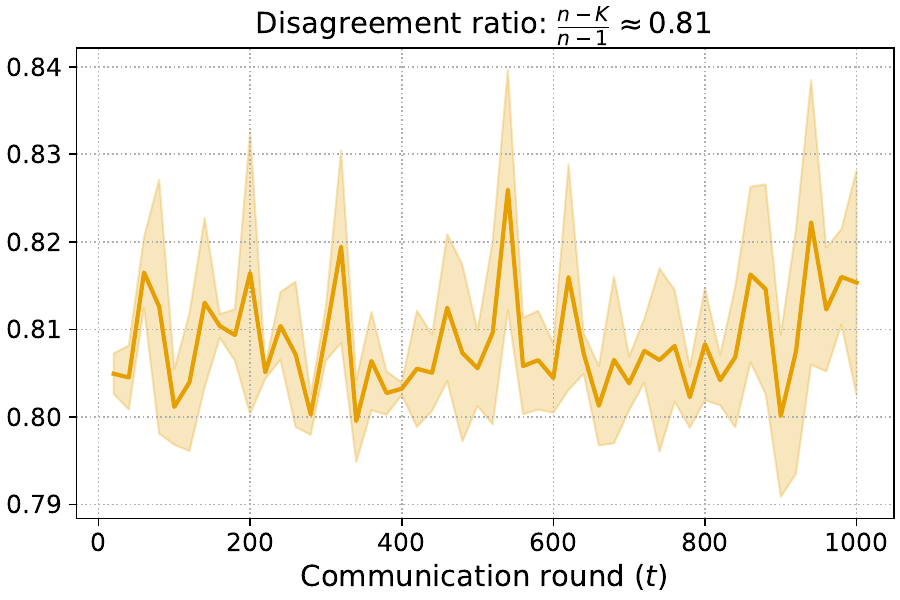}
        \caption{Disagreement ratio at D2S rounds.}
        \label{fig:app:track_disagreement_ratio}
    \end{subfigure}
    \hspace{1cm}
    \begin{subfigure}[b]{0.46\linewidth}
        \centering
        \includegraphics[width=\linewidth]{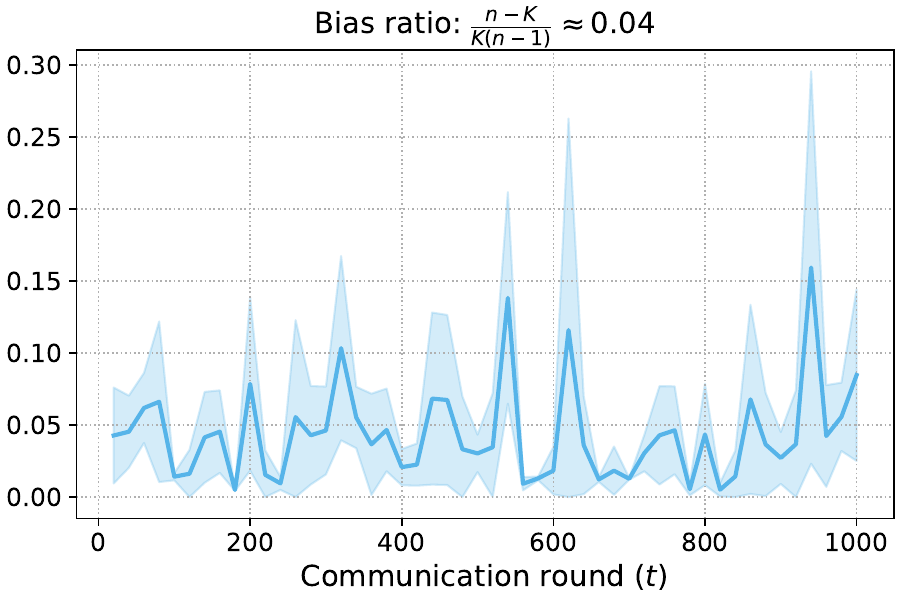}
        \caption{Bias ratio at D2S rounds.}
        \label{fig:app:track_bias_ratio}
    \end{subfigure}

    \caption{Bias and disagreement errors on CIFAR-10, Regime R3, $K/n=0.2$, $H=20$, ring topology.}
    \label{fig:track}
\end{figure}

\fi
\end{document}